\documentclass{article}

     \usepackage[final]{neurips_2022}

\usepackage[utf8]{inputenc} %
\usepackage[T1]{fontenc}    %
\usepackage{hyperref}       %
\usepackage{url}            %
\usepackage{booktabs}       %
\usepackage{amsfonts}       %
\usepackage{nicefrac}       %
\usepackage{microtype}      %
\usepackage[dvipsnames]{xcolor}         %

\usepackage{amsthm,amsmath,amssymb,bbm,bm}
\usepackage{cancel}
\usepackage{natbib}
\usepackage{longtable}
\usepackage{multirow}
\usepackage{setspace}
\usepackage{centernot}
\usepackage{array}
\usepackage{algorithmic}
\usepackage[linesnumbered, ruled, vlined]{algorithm2e}
\usepackage{mathrsfs}
\usepackage{dsfont}
\usepackage{relsize}
\usepackage{rotating}
\usepackage{enumitem}
\usepackage{float}
\floatstyle{plaintop}
\restylefloat{table}
\usepackage{subcaption}
\usepackage{multirow}
\usepackage{wrapfig}
\usepackage{graphicx}
\usepackage{tikzit}

\tikzstyle{Circle}=[fill=none, draw=black,ultra thick, shape=circle, align=center, text width=0.8cm]
\tikzstyle{Small Circle}=[fill=none, draw=black,ultra thick,dashed, shape=circle, align=center, text width=0.6cm]
\tikzstyle{Circle Gray}=[fill={rgb,255: red,191; green,191; blue,191}, draw=black,ultra thick, shape=circle, align=center, tikzit fill={rgb,255: red,191; green,191; blue,191}, tikzit draw=black, tikzit shape=circle]
\tikzstyle{Ellipse}=[fill=none, draw=black,ultra thick, shape=ellipse, align=center, text width=0.7cm]
\tikzstyle{Ellipse Upper Gray}=[fill=none, draw=black,ultra thick, path picture={\fill[gray] (path picture bounding box.north west) rectangle (path picture bounding box.east);}, shape=ellipse, align=center, text width=0.7cm]
\tikzstyle{Ellipse Lower Gray}=[fill=none, draw=black,ultra thick, path picture={\fill[gray] (path picture bounding box.south west) rectangle (path picture bounding box.east);}, shape=ellipse, align=center, text width=0.7cm]
\tikzstyle{Rectangle}=[fill=none, draw=black,ultra thick, shape=rectangle, tikzit draw=black, tikzit fill=none, tikzit shape=rectangle]
\tikzstyle{Rectangle Upper Gray}=[fill=none, draw=black,ultra thick, path picture={\fill[gray] (path picture bounding box.north west) rectangle (path picture bounding box.east);}, shape=rectangle, align=center, text width=0.7cm]
\tikzstyle{Rectangle Lower Gray}=[fill=none, draw=black,ultra thick, path picture={\fill[gray] (path picture bounding box.south west) rectangle (path picture bounding box.east);}, shape=rectangle, align=center, text width=0.7cm]
\tikzstyle{Rectangle Gray}=[fill={rgb,255: red,191; green,191; blue,191}, draw=black, shape=rectangle, tikzit fill={rgb,255: red,191; green,191; blue,191}, tikzit draw=black, tikzit shape=rectangle]

\tikzstyle{ArrowTo}=[draw=black,ultra thick, tikzit fill=none, tikzit draw=black, ->]
\tikzstyle{ArrowToDash}=[draw=black,ultra thick,dashed, tikzit fill=none, tikzit draw=black, ->]
\tikzstyle{ArrowFrom}=[draw=black,ultra thick, tikzit fill=none, tikzit draw=black, <-]
\tikzstyle{ArrowFromDash}=[draw=black,ultra thick,dashed, tikzit fill=none, tikzit draw=black, <-]
\tikzstyle{ArrowTwo}=[draw=black,ultra thick, tikzit fill=none, tikzit draw=black, <->]
\tikzstyle{ArrowTwoDash}=[draw=black,ultra thick,dashed, tikzit fill=none, tikzit draw=black, <->]

\hypersetup{
colorlinks=true,
linkcolor=red,
urlcolor=blue,
citecolor=blue
}

\usepackage{comment}

\newenvironment{smashedalign*}
{\par$\!\aligned}
{\endaligned$\par}

\newcommand{\numit}{\stepcounter{equation}\tag{\theequation}}

\theoremstyle{definition}
\newtheorem{ex}{Example}%
\newtheorem{assumption}{Assumption}
\newtheorem{definition}{Definition}[section]

\newtheorem{theorem}{Theorem}[section]
\newtheorem{lemma}{Lemma}[section] %
\newtheorem{proposition}{Proposition}[section]

\newtheorem{corollary}{Corollary}[section]

\newcommand{\indep}{\rotatebox[origin=c]{90}{$\models$}}
\newcommand{\dep}{\cancel{\rotatebox[origin=c]{90}{$\models$}}}

\newcommand{\proj}{\mbox{{\rm proj}}}

\def\expit{\mathrm{expit}}

\DeclareMathOperator*{\argmin}{arg\,min}
\DeclareMathOperator*{\arginf}{arg\,inf}

\newcommand{\BB}{{\mathbb{B}}}

\newcommand{\EE}{{\mathbb{E}}}

\newcommand{\NN}{{\mathbb{N}}}

\newcommand{\PP}{{\mathbb{P}}}

\newcommand{\RR}{{\mathbb{R}}}

\newcommand{\VV}{{\mathbb{V}}}

\newcommand{\bA}{{\mathbf{A}}}

\newcommand{\bI}{{\mathbf{I}}}

\newcommand{\bK}{{\mathbf{K}}}

\newcommand{\bM}{{\mathbf{M}}}

\newcommand{\bR}{{\mathbf{R}}}

\newcommand{\bU}{{\mathbf{U}}}

\newcommand{\bW}{{\mathbf{W}}}

\newcommand{\bY}{{\mathbf{Y}}}
\newcommand{\bZ}{{\mathbf{Z}}}

\newcommand{\bOmega}{{\boldsymbol{\Omega}}}
\newcommand{\bXi}{{\boldsymbol{\Xi}}}

\newcommand{\mA}{{\mathcal{A}}}
\newcommand{\mB}{{\mathcal{B}}}

\newcommand{\mD}{{\mathcal{D}}}

\newcommand{\mF}{{\mathcal{F}}}
\newcommand{\mG}{{\mathcal{G}}}
\newcommand{\mH}{{\mathcal{H}}}
\newcommand{\mI}{{\mathcal{I}}}

\newcommand{\mL}{{\mathcal{L}}}
\newcommand{\mM}{{\mathcal{M}}}
\newcommand{\mN}{{\mathcal{N}}}

\newcommand{\mP}{{\mathcal{P}}}

\newcommand{\mR}{{\mathcal{R}}}
\newcommand{\mS}{{\mathcal{S}}}

\newcommand{\mU}{{\mathcal{U}}}
\newcommand{\mV}{{\mathcal{V}}}
\newcommand{\mW}{{\mathcal{W}}}
\newcommand{\mX}{{\mathcal{X}}}

\newcommand{\mZ}{{\mathcal{Z}}}

\newcommand{\hhat}{\hat{h}}

\newcommand{\hdelta}{\widehat{\delta}}
\newcommand{\hDelta}{\widehat{\Delta}}

\newcommand{\tmG}{\tilde{\mG}}
\newcommand{\tmH}{\tilde{\mH}}

\renewcommand{\star}[1]{\mbox{{\rm star}}\{ #1 \}}

\newcommand{\norm}[1]{\| #1 \|}

\newcommand{\nmF}[1]{\| #1 \|_{\mF}}    %
\newcommand{\nmH}[1]{\| #1 \|_{\mH}}    %
\newcommand{\nmG}[1]{\| #1 \|_{\mG}}    %
\newcommand{\nmEmp}[1]{\| #1 \|_{n}}  %

\newcommand{\Vpi}{V^{\pi}}

\newcommand{\qpi}{q^{\pi}}
\newcommand{\hqpi}{\hat{q}^{\pi}}

\newcommand{\vpi}{v^{\pi}}
\newcommand{\hvpi}{\hat{v}^{\pi}}

\newcommand{\Ppi}{\mP^{\pi}}
\newcommand{\hPpi}{\widehat{\mP}^{\pi}}
\newcommand{\hmP}{\widehat{\mP}}

\newcommand{\bigO}{\ensuremath{\mathop{}\mathopen{}\mathcal{O}\mathopen{}}}

\def\qpi{q^\pi}

\newcommand{\pushright}[1]{\ifmeasuring@#1\else\omit\hfill$\displaystyle#1$\fi\ignorespaces}

\newcommand{\abs}[1]{|#1|}

\def\ones{\mathbf{1}}

\newcommand{\samfixed}[1]{}

\def\given{\mid}
\def\Given{\, \Big| \,}

\def\ds1{{\mathrm{1 \hspace{-2.6pt} I}}}

\def\calA{{\cal A}}

\def\calD{{\cal D}}

\def\calM{{\cal M}}

\def\calP{{\cal P}}

\def\calS{{\cal S}}

\def\calU{{\cal U}}
\def\calV{{\cal V}}
\def\calW{{\cal W}}

\def\calZ{{\cal Z}}

\title{Off-Policy Evaluation for Episodic Partially Observable Markov Decision Processes under Non-Parametric Models}

\author{%
  Rui Miao\\
  University of California, Irvine\\
  \texttt{rmiao2@uci.edu} \\
  \And
  Zhengling Qi\thanks{Corresponding author.} \\
  The George Washington University\\
  \texttt{qizhengling@gwu.edu}\\
  \And
  Xiaoke Zhang \\
  The George Washington University\\
  \texttt{xkzhang@gwu.edu}\\
}

\begin{document}

\allowdisplaybreaks

\maketitle
\begin{abstract}
We study the problem of off-policy evaluation (OPE) for episodic Partially Observable Markov Decision Processes (POMDPs) with continuous states. Motivated by the recently proposed proximal causal inference framework, we develop a non-parametric identification result for estimating the policy value via a sequence of so-called \textit{V-bridge} functions with the help of time-dependent proxy variables. We then develop a fitted-Q-evaluation-type algorithm to estimate V-bridge functions recursively, 
where a non-parametric instrumental variable (NPIV) problem is solved at each step. By analyzing this challenging sequential NPIV problem,
we establish the finite-sample error bounds for estimating the V-bridge functions and accordingly that for evaluating the policy value, in terms of
the sample size, length of horizon and so-called \textit{(local) measure of ill-posedness} at each step. To the best of our knowledge, this is the first finite-sample error bound for OPE in POMDPs under non-parametric models.
\end{abstract}

\section{Introduction}

In practical reinforcement learning (RL), 
a representation of the full state which makes the system Markovian and therefore amenable to most existing RL algorithms is not known {\em a priori.} Decision makers are often facing so-called \textit{partial observability} of the state information, which significantly hinders the task of RL. In general, agents have to maintain all historical information and establish a belief system on the hidden state for optimal decision making. A partially observable Markov decision process (POMDP) is often used to model the data generating process. See examples in robotics \citep{rafferty2011faster}, precision medicine \citep{tsoukalas2015data}, stochastic game \citep{hansen2004dynamic} and many others. However, it is well known that learning optimal policies in POMDP is computationally intractable \citep{papadimitriou1987complexity}. The issue of partial observability becomes more serious in the batch setting, where agents are not able to actively collect additional data and further explore the environment. For example, standard off-policy evaluation (OPE) methods, which aim to learn a policy value from the batch data generated from some behavior policy, would fail to give a consistent estimate because of unobserved state variables.

Due to this practical concern, there is a recent line of research studying the OPE under the framework of a confounded POMDP, where the behavior policy to generate the batch data is allowed to depend on some unobserved state variables \citep[e.g.,][]{tennenholtz2020off,nair2021spectral,bennett2021proximal,shi2021minimax}. Their identification results on the policy value are inspired by the negative controls or so-called proxy variables in the literature of causal inference \citep[e.g.,][]{miao2018identifying,tchetgen2020introduction}. A building block of these results is the existence of some bridge functions, namely $Q$/$V$-bridge or weight-bridge functions, which are projections of the $Q$/$V$-functions or importance weights defined over the original state space onto the observation space. %
The corresponding statistical estimation of these bridge functions mainly relies on solving linear integral equations \citep[e.g.,][]{kress1989linear}.  Different from the tabular case studied by \cite{tennenholtz2020off} and \cite{nair2021spectral} and linear models studied by \cite{shi2021minimax} theoretically, solving linear integral equations with non-parametric models in the continuous state/observation space are known to be challenging due to the potential ill-posedness \citep{chen2011rate}, leading to slow statistical convergence rates. %
However, existing theoretical results developed by \cite{bennett2021proximal} and \cite{shi2021minimax} 
require fast enough convergence rates for these bridge function estimators in order to establish  the asymptotic normality of their estimators for OPE, which could be illusive when the problem is seriously ill-posed under non-parametric models. This is different from the supervised learning where a fast enough convergence rate can be easily achieved under non-parametric models. Therefore, to fill this important theoretical gap, it is necessary to study the finite-sample performance of OPE of which bridge functions are estimated non-parametrically.

Motivated by these, in this paper, we study the OPE for confounded and episodic POMDPs with continuous states, where we non-parametrically estimate $V$-bridge functions.
Our main contribution to the literature is three-fold.
First, relying on some  time-dependent proxy variables, we establish a non-parametric identification result for OPE using $V$-bridge functions for time-inhomogeneous confounded POMDPs. Based on the identification result, we develop a new  fitted-Q-evaluation(FQE)-type approach to estimating $V$-bridge functions recursively and obtain an estimator for OPE based on the bridge function estimators. At each step of our algorithm,  we propose to fit a non-parametric instrumental variable (NPIV) regression using a min-max estimation method, i.e., solving a linear integral equation with a non-parametric model. Our algorithm can be viewed as a sequential NPIV estimation, which is not well studied in the literature. Second and most importantly, we establish the finite-sample error bound for estimating $V$-bridge functions and accordingly that for evaluating the policy value, 
in terms of the \textit{sample size}, \textit{length of horizon} and \textit{(local) measure of ill-posedness} at each step.  
Unlike the well studied standard NPIV model in the econometrics literature \citep[e.g.,][]{ai2003efficient,newey2003instrumental} where the response variable is directly observed, 
the response variable in our NPIV model at each step of the algorithm relies on the model estimate at its previous step. 
This difference makes our theoretical analysis
substantially difficult.
By carefully characterizing the statistical error due to the NPIV estimation at each step and more importantly, its propagation effect on future estimates, we are able to establish the first finite-sample result of OPE for confounded POMDPs under non-parametric models, which achieves a polynomial order over the length of horizon and sample size. %
Finally, our theoretical results on the sequential NPIV estimation are generally applicable to other sequential-type conditional moment restriction problems. The development of the
uniform finite-sample error bounds of the NPIV estimation, extending the pointwise result in the previous literature such as \cite{dikkala2020minimax}, may be of independent interest.

\section{Related Work}
Recently there is a surge of interest in studying OPE with unobserved variables in the sequential decision making problem. Specifically, \cite{zhang2016markov} are among the first who proposed the framework of confounded MDPs, which essentially considers i.i.d. confounders in the dynamic system and therefore preserves the Markovian property. Along this direction, OPE methods are developed under various identification conditions such as partial identification using sensitivity analysis \citep{namkoong2020off,kallus2020confounding,bruns2021model}, instrumental variable or mediator assisted OPE \citep{liao2021instrumental,li2021causal,shi2022off} and many others. Another line of research focuses on more general confounded POMDP models , where the Markovian assumption is violated, under which several
point estimation results were developed such as the aforementioned proxy variables related methods \citep{tennenholtz2020off,deaner2018proxy,ying2021proximal,bennett2021proximal,nair2021spectral,shi2021minimax}, spectral methods in undercomplete POMDPs \citep{hsu2012spectral,anandkumar2014tensor,jin2020sample} and predictive state representation related methods \citep{littman2001predictive,singh2012predictive,cai2022sample}. 

Our proposed method, which uses proxy variables for OPE, is closely related to those recently developed by  \cite{bennett2021proximal}, \cite{shi2021minimax}, and \cite{ying2021proximal}.
\cite{bennett2021proximal} and \cite{ying2021proximal} studied episodic POMDPs (or complex longitudinal studies) and mainly focused on developing asymptotic normality results of their policy value estimators. Their results rely on some high level rate conditions on the bridge function estimation, which are \textit{unknown} if they would be satisfied when using non-parametric models due to the aforementioned measure of ill-posedness. \cite{shi2021minimax} mainly focused on time-homogeneous infinite-horizon POMDPs and developed asymptotic normality for their estimators under similar high-level conditions, which therefore has the same issue. Besides, while \cite{shi2021minimax} also established finite-sample bounds for their bridge function estimation and corresponding OPE, they only study the tabular case or linear/parametric models, where the issue of ill-posedness \textit{does not exist}. In this paper, we provide a systematic investigation on the estimation of $V$-bridge functions and establish finite-sample guarantees for them and the corresponding OPE under non-parametric models. Specifically, we tackle the challenging episodic setting, where $V$-bridge functions are estimated sequentially. Without carefully controlling the effect of ill-posedness at each step and its propagation effect on future steps, the estimation error for these $V$-bridge functions and also that for OPE could be exponentially large in terms of the length of horizon. Motivated by the chaining argument in the empirical process theory, we successfully disentangle the effects of ill-posedness on the current step and future steps separately and thus establish finite-sample bounds for $V$-bridge functions and OPE both with a polynomial dependence on the length of horizon, which are new theoretical results we contribute to the literature.

Since our $V$-bridge function estimation can be formulated as a sequential NPIV problem, it is natually related to classical NPIV estimations, which have been extensively studied in the econometrics literature \citep[see, e.g.,][for earlier reference]{newey2003instrumental,ai2003efficient,ai2012semiparametric,hall2005nonparametric,chen2011rate,chen2018optimal,darolles2011nonparametric,blundell2007semi}. 
Recently there is also a growing interest in the min-max estimation for NPIV models
\citep[see, e.g.,][for some recent developments]{muandet2020dual,dikkala2020minimax,hartford2017deep}. As commented before, existing theoretical results for standard NPIV models cannot be directly applied to our setting due to the sequential structure of our FQE algorithm, so we need to develop new theory to address our setting. 
Technically, in order to establish a  polynomial-order finite-sample error bound over the length of horizon for OPE, which is particularly important in RL, we decompose the measure of ill-posedness at each step of our sequential NPIV estimation into two components: the so-called (local) measure of one-step transition ill-posedness and the standard (local) measure of ill-posedness \citep[e.g.,][]{chen2012estimation}. Thanks to this novel decomposition, the effect of the first component on the estimation error of $V$-bridge functions and OPE is multiplicative but can be properly controlled while that of the second component could be large but is only cumulative. See Theorem \ref{thm:decomposition}. %
Finally, we remark that \cite{ai2012semiparametric} also studied the sequential NPIV estimation problem, where the non-parametric components are estimated jointly. However, this method could be computationally inefficient in RL with a long horizon. More importantly, their results are built on the nested structure among conditional moment restriction models, which are not satisfied in our setting.

\section{Preliminaries and Notations}\label{sec:prelim}
In this section, we introduce the framework of discrete-time confounded POMDPs and its related OPE problem. Consider an episodic and confounded POMDP denoted by $\calM = (\calS, \calU, \calA, T, \calP, r)$, with $\calS$ and $\calU$ as the observed and unobserved continuous state spaces respectively, 
$\calA$ as the discrete action space, $T$ as the length of horizon, $\calP = \{\mathbb{P}_t\}_{t=1}^T$  as the transition kernel over $\calS \times \calU \times \calA$ to $\calS \times \calU$, and $r = \{r_{t}\}_{t=1}^T$ as the reward function over $\calS \times \calU \times \calA$. $\calS$ can also be treated as the observation space in the classical POMDP. Then the process of $\calM$ can be summarized as $\{S_t, U_t, A_t, R_t\}_{t=1}^T$ with $S_t$ and $U_t$ as observed and unobserved state variables, $A_t$ as the action, and $R_t$ as the reward, where $r_t(s, u, a) = \EE[R_t \given S_t = s, U_t = u, A_t = a]$ for any $(s, u, a) \in \calS \times \calU \times \calA$. For simplicity, we assume that $\abs{R_t} \leq 1$ 
uniformly in  $1 \leq t \leq T$. %

The goal of OPE in a confounded POMDP is to evaluate the performance of a target policy using the batch data collected by some behavior policy. In this paper, the target policy we focus on is a sequence of functions mapping from the state space $\calS$ to a probability mass function over the action space $\calA$, denoted by $\pi = \{\pi_t\}_{t=1}^T$, where $\pi_t(a \given s)$ is the probability of choosing an action $A_t=a$ given the state value $S_t = s$. We remark that our proposed identification results stated in Section \ref{sec:identify} can be generalized to other policies such as history-dependent ones. Given a target policy $\pi$, define its state value function as
\begin{align}\label{eqn: Q-function}
\textstyle	\Vpi_t(s, u) = \EE^\pi[\sum_{t' = t}^{T}R_{t'} \given S_t = s, U_t = u], \quad \text{for every $(s, u) \in \calS \times \calU$,}
\end{align}
where $\EE^\pi$ denotes the expectation with respect to the distribution whose action at decision time $t$ follows $\pi_t$ for any $t \geq 1$. We consider the batch setting, where the observed action $A_t$ is generated by some behavior policy $\tilde\pi^b_t$ depending on both $S_t$ and $U_t$ for $1 \leq t \leq T$. We aim to use the batch data to estimate the \textit{policy value} of a target policy $\pi$, which is defined as 
\begin{align}\label{def: integrated value fun}
	\calV(\pi) = \EE[\Vpi_1(S_1, U_1)],
\end{align}
where $\EE$ denotes the expectation with respect to the behavior policy. Due to the unobserved $U_t$, standard OPE methods that rely on the Markovianity will give bias estimations. In the following, we introduce an identification result for estimating the policy value using some proxy variables.

\emph{Notations}: For two sequences $\{\varpi(n)\}_{n\geq1}$ and $\{\theta(n)\}_{n\geq1}$, the notation $\varpi(n) \gtrsim  \theta(n)$ (resp. $\varpi(n) \lesssim \theta(n)$) means that there exists a sufficiently large constant (resp. small) constant $c_1>0$ (resp. $c_2>0$) such that $\varpi(N) \geq c_1 \theta(N)$ (resp. $\varpi(n) \leq c_2 \theta(n)$). We use $\varpi(n) \asymp \theta(n)$ when $\varpi(n) \gtrsim \theta(n)$ and $\varpi(n) \lesssim \theta(n)$. For any random variable $X$, we use $\mL^q\{X\}$ to denote the class of all measurable functions with finite $q$-th moments for $1 \leq q \leq \infty$. Then the $\mL^q$-norm is denoted by $\norm{\bullet}_{\mL^q\{X\}}$. When there is no confusion in the underlying distribution, we also write it as $\norm{\bullet}_{\mL^q}$ or $\norm{\bullet}_{q}$.  In particular, $\norm{\bullet}_\infty$ denotes the sup-norm. 
In addition, we use Big $O$ and small $o$ as the convention.

\section{Identification Results}\label{sec:identify}

Inspired by the proximal causal inference recently proposed by \cite{tchetgen2020introduction},
we develop a non-parametric identification result for estimating $\calV(\pi)$, which is similar to those by \cite{bennett2021proximal} and \cite{shi2021minimax}. Assume that we can additionally observe the so-called reward-inducing proxy variables $W_t$ that are only
related to the action $A_t$ through $(S_t, U_t)$ and action-inducing proxy variables $Z_t$ that are only related to the reward $R_t$
through $(S_t,U_t)$ at each decision time $t$. See Figure \ref{fig:POMDP} for a directed acyclic graph (DAG) to illustrate their relationships and a time series data example in \cite{miao2018confounding}. For another example, the action-inducing proxy variables $Z_t$ can be defined as the observed history before time $t$, then $Z_t$ and related arrows in Figure \ref{fig:POMDP} can be removed. Detailed assumptions and discussion are given in Appendix \ref{sec:additional assumption}. %
Denote the spaces of $\{Z_t\}_{t=1}^T$ and $\{W_t\}_{t=1}^T$ 
by $\calW$ and $\calZ$ respectively.

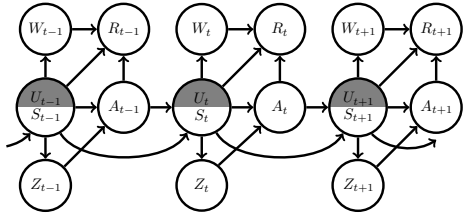
\begin{figure}[H]
\centering
\resizebox{0.45\textwidth}{!}{
  \begin{tikzpicture}
	\begin{pgfonlayer}{nodelayer}
		\node [style=Ellipse Upper Gray] (USp) at (-15, 0) {$U_{t-1}$ $S_{t-1}$};
		\node [style=Circle] (Zp) at (-15, -6) {$Z_{t-1}$};
		\node [style=Circle] (Wp) at (-15, 6) {$W_{t-1}$};
		\node [style=Circle] (Ap) at (-9, 0) {$A_{t-1}$};
		\node [style=Circle] (Rp) at (-9, 6) {$R_{t-1}$};
		\node [style=Ellipse Upper Gray] (US) at (-3, 0) {$U_t$ $S_t$};
		\node [style=Circle] (Z) at (-3, -6) {$Z_t$};
		\node [style=Circle] (W) at (-3, 6) {$W_t$};
		\node [style=Circle] (A) at (3, 0) {$A_t$};
		\node [style=Circle] (R) at (3, 6) {$R_t$};
		\node [style=Ellipse Upper Gray] (USn) at (9, 0) {$U_{t+1}$ $S_{t+1}$};
		\node [style=Circle] (Zn) at (9, -6) {$Z_{t+1}$};
		\node [style=Circle] (Wn) at (9, 6) {$W_{t+1}$};
		\node [style=Circle] (An) at (15, 0) {$A_{t+1}$};
		\node [style=Circle] (Rn) at (15, 6) {$R_{t+1}$};

	\end{pgfonlayer}
	\begin{pgfonlayer}{edgelayer}
		\draw [style=ArrowTo] (USp) to (Wp);
		\draw [style=ArrowTo] (USp) to (Zp);
		\draw [style=ArrowTo] (Zp) to (Ap);
		\draw [style=ArrowTo] (Wp) to (Rp);
		\draw [style=ArrowTo] (USp) to (Rp);
		\draw [style=ArrowTo] (Ap) to (Rp);
		\draw [style=ArrowTo] (USp) to (Ap);
		\draw [style=ArrowTo] (US) to (W);
		\draw [style=ArrowTo] (US) to (Z);
		\draw [style=ArrowTo] (Z) to (A);
		\draw [style=ArrowTo] (W) to (R);
		\draw [style=ArrowTo] (US) to (R);
		\draw [style=ArrowTo] (A) to (R);
		\draw [style=ArrowTo] (US) to (A);
		\draw [style=ArrowTo] (USn) to (Wn);
		\draw [style=ArrowTo] (USn) to (Zn);
		\draw [style=ArrowTo] (Zn) to (An);
		\draw [style=ArrowTo] (Wn) to (Rn);
		\draw [style=ArrowTo] (USn) to (Rn);
		\draw [style=ArrowTo] (An) to (Rn);
		\draw [style=ArrowTo] (USn) to (An);
		\draw [style=ArrowTo] (Ap) to (US);
		\draw [style=ArrowTo] (A) to (USn);
		\draw [style=ArrowTo, bend right=15, looseness=1] (-18,-3) to (USp);
		\draw [style=ArrowTo, bend right=60, looseness=0.75] (USp) to (US);
		\draw [style=ArrowTo, bend right=60, looseness=0.75] (US) to (USn);
		\draw [style=ArrowTo, bend right=40, looseness=0.9] (USn) to (15, -2.5);

	\end{pgfonlayer}
\end{tikzpicture}
  }
  \caption{A representative DAG to illustrate the variables involved in the confounded POMDP.} %
  \label{fig:POMDP}
\end{figure}
Since the states $\left\{ U_t \right\}_{t=1}^T$ are unmeasured, we
cannot estimate the value function by the celebrated Bellman equation. However, with the help of confounding proxies $\left\{ W_t,Z_t \right\}_{t=1}^T$, the value of a target policy $\pi$ can be non-parametrically identified using observed variables under proper assumptions. 

To proceed, we define a class of $V$-bridge functions (or $V$-bridges for short) $\{\vpi_t\}_{t=1}^T$ defined over $\calW \times \calS$ such that for every $(s, u) \in \calS \times \calU$ and $t \geq 1$,
\begin{equation}
\label{eq:V-bridge}
\textstyle\EE\left[\vpi_t(W_t, S_t) \given U_t=u, S_t=s\right] = \EE^{\pi}\left[\sum_{t'=t}^T R_{t'} \Given U_t=u, S_t=s\right].
\end{equation}
If such $V$-bridges exist, then we obtain the following identification result for the policy value
in \eqref{def: integrated value fun}. 
\begin{proposition}[Identification] \label{prop:identification}
If there exist $\{\vpi_t\}_{t=1}^T$ that satisfy \eqref{eq:V-bridge}, then the value of target policy $\pi$ can be
identified by
$
  \mV(\pi) = \EE[\vpi_1(W_1, S_1)].
$ 

\end{proposition}
Note that $V$-bridges $\{\vpi_t\}_{t=1}^T$ that satisfy \eqref{eq:V-bridge} are not necessarily unique, but 
we can uniquely identify
$\mV(\pi)$ based on any of them. Next, we provide a theoretical guarantee for the existence of $V$-bridges
$\{\vpi_t\}_{t=1}^T$ in terms of a sequence of linear integral equations.
\begin{theorem}\label{thm:existence}
  For a POMDP model of which variables satisfy the relationships illustrated in Figure \ref{fig:POMDP} and some regularity conditions given in Appendix \ref{sec:additional assumption}, there always exist
  $V$-bridges $\{\vpi_t\}_{t=1}^T$ %
  satisfying \eqref{eq:V-bridge}.
  With $\vpi_{T+1}=0$, a particular sequence of $V$-bridges $\{\vpi_t\}_{t=1}^T$ can be obtained by solving the following linear integral equations:
\begin{equation}
  \label{eq:nonpara-identification}
  \EE \left\{\qpi_t(W_t,S_t,A_t) - R_t-\vpi_{t+1}(W_{t+1},S_{t+1})\given Z_t,S_t,A_t \right\} =0,
\end{equation}
where $\{\qpi_t\}_{t=1}^T$ are $Q$-bridges defined over $\calW \times \calS \times \calA$ such that
\begin{equation}
\label{eq:Q-bridge}
\textstyle\EE\left[\qpi_t(W_t,S_t, A_t) \given U_t=u, S_t=s, A_t = a \right] = \EE^{\pi}\left[\sum_{t'=t}^T R_{t'} \Given U_t=u, S_t=s, A_t = a\right],
\end{equation}
for every $(s, u, a) \in \calS \times \calU \times \calA$ and $t \geq 1$, and $\vpi_t(w,s) = \sum_{a\in\mA}\pi_t(a\given s) \qpi_t(w,s,a)$.
Clearly $Q$-bridges $\{\qpi_t\}_{t=1}^T$ also exist.
\end{theorem}

Theorem \ref{thm:existence} guarantees the existence of both $V$-bridges and $Q$-bridges, and also provides a natural procedure  \eqref{eq:nonpara-identification} to find $\{\vpi_t\}_{t=1}^T$
and eventually estimate the policy value $\calV(\pi)$. 
Then based on Proposition \ref{prop:identification} and Theorem \ref{thm:existence}, we can perform OPE via Algorithm \ref{alg:identification} in the population level. Specifically at each  step we will solve \eqref{eq:nonpara-identification} via a non-parametric model, which is a NPIV problem.

\noindent
\begin{minipage}{\linewidth}
\begin{algorithm}[H] \label{alg:identification}
\SetAlgoLined
\textbf{Input:} $\{(S_t,W_t,Z_t,A_t,R_t)\}_{t=1}^T$, a target policy $\pi = \{\pi_t\}_{t=1}^T$. \\
Let $\vpi_{T+1}=0$.\\
Repeat for $t=T,\dots,1$:\\
\Indp
Solve $\vpi_t$ and $\qpi_t$ by
  $\EE \left\{\qpi_t(W_t,S_t,A_t) - R_t-\vpi_{t+1}(W_{t+1},S_{t+1})\given Z_t,S_t,A_t \right\} =0$ with $\vpi_t(W_t,S_t) \triangleq \sum_{a\in\mA}\pi_t(a\given S_t)
\qpi_t(W_t,S_t,a)$.\\

\Indm
\textbf{Output:} $\mV(\pi) = \EE[\vpi_1(W_1, S_1)]$.
\caption{Identification of $\mV(\pi)$}
\end{algorithm}
\end{minipage}

\section{Estimation} \label{sec:estimation}

In this section, we discuss how to estimate $\calV(\pi)$ using batch data based on results given in Theorem \ref{thm:existence} and Algorithm \ref{alg:identification}. Let a pre-collected training dataset be $
\calD_n = \{ \left(S_{t,i}, W_{t,i}, Z_{t,i}, A_{t,i}, R_{t,i}\right)_{t=1}^T: i=1,\ldots, n\}$, which
consists of $n$ i.i.d. copies
of the observable trajectory $\left(S_{t}, W_{t}, Z_{t}, A_{t}, R_{t}\right)_{t=1}^T$ of a confounded POMDP.
Following Algorithm \ref{alg:identification}, we develop a FQE-type approach where 
we propose to solve a min-max problem for estimating $\vpi_t$  at the $t$-th step using the idea of \cite{dikkala2020minimax}, and then apply Proposition \ref{prop:identification} for OPE.

For convenience, we first rewrite the linear integral equations \eqref{eq:nonpara-identification} for solving $V$-bridges in terms of operators.
Define an operator $\widetilde\mP_t: \mL^2\{\mR \times\mW \times\mS\}\rightarrow \mL^2\{\mZ\times\mS\times\mA\}$ such that 
$[\widetilde\mP_tg](Z_t,S_t,A_t) = \EE [ g(R_t,W_{t+1},S_{t+1}) \given Z_t,S_t,A_t]$
for any $g\in\mL^2\{\mR\times\mW\times\mS\}$.
Define another operator $\overline{\mP}_t:\mL^2\{\mW\times\mS\times\mA\} \rightarrow\mL^2\{\mZ\times\mS\times\mA\}$ 
such that for any $h\in\mL^2\{\mW\times\mS\times\mA\}$,
$
  [\overline{\mP_t}h](Z_t,S_t,A_t) = \EE \left[ h(W_t,S_t,A_t) \given Z_t,S_t,A_t \right].
  $
Motivated by \eqref{eq:nonpara-identification}, %
we define the \emph{V-bridge transition operator}
$\Ppi_t: \mL^2\{\mR\times\mW\times\mS\}\rightarrow \mL^2\{\mW\times\mS\}$
such that
\begin{align*}
    \Ppi_tg = \left\langle
  \pi_t,\mP_tg \right\rangle, \  \text{ where $\mP_tg = \overline{\mP_t}^{-1}\widetilde{\mP}_tg$ for all
$g\in\mL^2\{\mR\times\mW\times\mS\}$.}
\end{align*}
 In particular, $\langle\pi_t(\cdot\given S_t), [\mP_t g](W_t,S_t,\cdot)\rangle \triangleq \sum_{a\in\mA} \pi_t(a\given S_t)[\mP_t g](W_t,S_t,a)$, 
and $\widetilde\mP_tg$ is invertible by $\overline{\mP_t}$. The invertibility is ensured by Assumption \ref{ass:completeness2} in Appendix \ref{sec:additional assumption}.

Then by the definition of $V$-bridges and \eqref{eq:nonpara-identification},  we can identify $\{\vpi_t\}_{t=1}^T$ via solving
\begin{equation}\label{eqn: linear integral equation}
\vpi_t = \Ppi_t(\vpi_{t+1}+R_t), \quad \text{for $t \geq 1$.}
\end{equation}
To find the estimated V-bridges $\{\hvpi_t\}_{t=1}^T$,
it suffices to estimate $\Ppi_t$. Note that one can regard \eqref{eqn: linear integral equation} as a series of conditional moment model restrictions
and we propose to solve them via a sequential NPIV estimation. In particular, at the $t$-th step, we adopt the min-max estimation method proposed by \citet{dikkala2020minimax} to estimate $\Ppi_t$
non-parametrically as follows: $\hPpi_tg = \left\langle \pi_t, \widehat{\mP}_tg \right\rangle$, where 
\begin{equation}\label{eqn: minmax estimation}
  \widehat{\mP}_t g / (T-t+1) = \argmin_{h\in\mH^{(t)}} \Big[\sup_{f\in\mF^{(t)}} \Big\{\Psi_{t,n}(h, f,g) - \lambda(\|f\|^2_{\mF^{(t)}} + \frac{M}{\delta^2}\norm{f}_n^2) \Big\} + \lambda\mu \|h\|^2_{\mH^{(t)}} \Big],
\end{equation}
where $\norm{f}_n^2 = n^{-1}\sum_{i=1}^n f^2(Z_{t,i},S_{t,i},A_{t,i})$ for $f \in \mF^{(t)}$.
$\mH^{(t)}$ on $\mW\times\mS\times\mA$ and $\mF^{(t)}$ on $\mZ\times\mS\times\mA$ are two user-defined function spaces  endowed with norms $\norm{\bullet}_{\mH^{(t)}}$ and $\norm{\bullet}_{\mF^{(t)}}$ respectively, $\lambda,\mu, M,\delta >0$ are tuning parameters, and 
\begin{equation*}
  \Psi_{t,n}(h,f,g) = n^{-1}\textstyle\sum_{i=1}^n [h(W_{t,i},S_{t,i},A_{t,i}) - (T-t+1)^{-1} g(R_{t,i},W_{t+1,i},S_{t+1,i})]f(Z_{t,i},S_{t,i},A_{t,i}),
\end{equation*}
where $g(R_t,W_{t+1},S_{t+1}) = R_t + \bar g(W_{t+1},S_{t+1})$ for some $\bar g\in\mG^{(t+1)}$ on $\mW\times\mS$, endowed with norm $\norm{\bullet}_{\mG^{(t+1)}}$. 

The rational behind \eqref{eqn: minmax estimation} is that when $\lambda, \lambda\mu \rightarrow 0$ and $\lambda M/\delta^2 \asymp 1$, the following two population-version min-max optimization problems 
\begin{align*}
    \min_{h\in\mH^{(t)}} \sup_{f\in\mF^{(t)}}&\EE [h(W_t,S_t,A_t) - (T-t+1)^{-1}g(R_t,W_{t+1},S_{t+1})]f(Z_t,S_t,A_t) - \textstyle\frac{1}{2}f^2(Z_t,S_t,A_t), \\
    \min_{h\in\mH^{(t)}} &\EE\{\EE[h(W_t,S_t,A_t) - (T-t+1)^{-1}g(R_t,W_{t+1},S_{t+1})\given Z_t,S_t,A_t]\}^2,
\end{align*}
have the same solution $h$ when the space $\mF^{(t)}$ of testing functions is rich enough. Note that $(T-t+1)^{-1}$ used above and in \eqref{eqn: minmax estimation} are for scaling purpose.

After $T$ steps, we output our estimator for the policy value based on the empirical counterpart of Proposition \ref{prop:identification}. Our FQE-type algorithm is summarized in Algorithms \ref{alg: one-step NPIV} and \ref{alg:DetailedFQE} in Appendix \ref{sec: simulation appendix}. %

\section{Theoretical Results}
\label{sec:theory}
In this section, we establish the finite-sample bounds for the $\mL^2$ error of estimating $V$-bridge $\vpi_1$  and the error of OPE, in terms of the sample size, length of horizon and two (local) measures of ill-posedness. Our bounds also rely on the critical radii of certain
spaces related to the user-defined function spaces $\mH^{(t)}$ and $\mF^{(t)}$ in \eqref{eqn: minmax estimation}, and also $\mG^{(t)}$ 
of $V$-bridge functions.

\textbf{1. Technical preliminaries.} Before presenting our main results, we first
introduce some concepts from the empirical process theory \citep{wainwright2019high}.%
\begin{definition}[Local Rademacher Complexity]
  Given any real-valued function class $\mF$ defined over a random vector $X$ and any radius $\delta > 0$, the
  local Rademacher complexity is given by
\begin{equation}
\label{eq:Local Rademacher Complexity}
\mR_n(\mF,\delta) = \EE_{\epsilon,X} [ \textstyle\sup_{f\in\mF:\norm{f}_n\leq\delta} | n^{-1}\textstyle\sum_{i=1}^n \epsilon_i f(X_i) | ],
\end{equation}
where $\{ X_i \}_{i=1}^n$ are i.i.d. copies of $X$ and  $\left\{ \epsilon_i \right\}_{i=1}^n$ are i.i.d. Rademacher random
variables. 
\end{definition}
By bounding the local Rademacher complexity, which measures the complexity of
the functional class $\mF$ locally in a neighborhood of the ground truth, we
can control the error rate of the proposed $V$-bridge estimator in each step.
A crucial parameter for local Rademacher complexity of a function class $\mF$
is called \emph{critical radius}.
\begin{definition}[Critical Radius]
Assume that $\mF$ is a star-shaped function class, i.e. $\alpha f\in\mF$ for any $f\in\mF$ and
scalar $\alpha\in[0,1]$, and also that $\mF$ is
$b$-uniformly bounded, i.e., $\norm{f}_{\infty}\leq b <\infty$, $\forall f\in\mF$.
The critical radius of $\mF$, denoted by $\delta_n$, is the solution to the inequality $\mR_n(\mF,\delta) \leq \delta^2/b.$
\end{definition}

\emph{Additional Notations}: 
We assume that the test functions $f$ belong to a star shaped, symmetric space
$\mF^{(t)}\subseteq\mL^2(\mZ\times\mS\times\mA)$
endowed with norm $\norm{\cdot}_{\mF^{(t)}}$.
For brevity of notation, hereafter we suppress the time-step indicator $(t)$ in the context
unless necessary.
 For a function space $\mF$, we define $\alpha\mF=\{\alpha f: f\in\mF\}$, for some $\alpha\in\RR$. Define $\mF_B = \{f\in\mF:\norm{f}_{\mF}^2\leq B\}$, for any $B>0$.
Define the projected root mean squared error $\norm{\proj_t f}_2 = \sqrt{\EE \{\EE [f(X)\given Z_t,S_t,A_t]\}^2}$, for any squared integrable $f$ with respect to the conditional distribution of $X$ given $(Z_t,S_t,A_t)$.

\emph{Standard  (Local) Measures of ill-posedness}: Let $\bar{\tau}_{1} =
\sup_{g\in\mG^{(1)}}\norm{g(W_1,S_1)}_2/\norm{\EE[g(W_1,S_1)\given Z_1,S_1]}_2$ be the measure of ill-posedness for $\mG^{(1)}(\mW_1\times\mS_1)$ projected on $\mZ_1\times\mS_1$. 
Let $\tau_t =
\sup_{h\in\mH^{(t)}}\norm{h(W_t,S_t,A_t)}_2/\norm{\proj_th(W_t,S_t,A_t)}_2$ be the
standard measure of ill-posedness for $\mH^{(t)}(\mW\times\mS\times\mA)$
projected on $\mZ\times\mS\times\mA$. It can be seen that $\bar{\tau}_{1}, \tau_t \geq 1$ for $t \geq 1$. Indeed we only require measuring  $\bar{\tau}_{1}$ and $\tau_t$ \textit{locally}. See more details in Appendix \ref{sec:proofs}.

\textbf{2. Results.}
We first give Assumption \ref{ass:technical} used to develop our theoretical results below.
\begin{assumption}
  \label{ass:technical}
  For each $t=1,\dots,T$,
\begin{enumerate}[leftmargin=.2in]
\item[(1)] \label{ass:closeness} Closeness. For any $g\in\mG^{(t+1)}$, $\mP_t(g+R_t)\in\mH^{(t)}$; For
  any $h\in\mH^{(t)}$, $\left\langle \pi_t,h \right\rangle\in\mG^{(t)}$.
\item[(2)] For any $h\in(T-t)\mH^{(t+1)}$, %
we have $\norm{\mP_t \left( \frac{R_t+\left\langle \pi_{t+1}, h \right\rangle}{T-t+1} \right)}_{\mH^{(t)}}^2\leq \norm{\frac{h}{T-t}}_{\mH^{(t+1)}}^2$.
\item[(3)] There exists a constant $C_{\mG}>0$ such that 
$\norm{\left\langle \pi_t,h \right\rangle}_{\mG^{(t)}}^2 \leq C_{\mG} \norm{h}_{\mH^{(t)}}^2$, 
for $h\in\mH^{(t)}$.

\item[(4)] $\qpi_t\in(T-t+1)\mH^{(t)}(\mW,\mS,\mA)$ and $\norm{\qpi_T}_{\mH^{(T)}}^2\leq M_{\mH}$, where $M_{\mH} > 0$ is a constant.
\item[(5)] Testing function class $\mF^{(t)}$ is sufficiently rich such that
  there exists $L>0$, $\norm{f^{*}-\proj_th_t}_2 \leq \eta_n^{(t)}$, where
  $f^{*}\in\argmin_{f\in\mF_{L^2\nmH{h_t}^2}^{(t)} } \norm{f-\proj_th_t}_2$, 
  for all $h_t\in\mH^{(t)}$.
\item[(6)] Behavior policies: there exists a constant $b_\pi$ such that $\pi_t^b(a \given s) \triangleq \EE[\tilde\pi_t^b (a\given U_t,S_t)\given S_t=s]\geq b_\pi>0$ for all $(s, a) \in \calS \times \calA$.
\end{enumerate}
\end{assumption}
Assumption \ref{ass:technical}~(1) is similar to Bellman completeness, which has been widely used in RL without unobserved states \citep[e.g.,][]{antos_learning_2008}. Note that both $\mG^{(t)}$ and $\mH^{(t)}$ can be chosen as infinite-dimensional spaces, e.g., RKHSs. Hence this assumption is relatively mild. Assumption \ref{ass:technical}~(2) requires the operator $\mP_t$ to be bounded, which can be ensured under some continuity conditions on  transition kernels \citep{kress1989linear}.  Assumption \ref{ass:technical}~(3) is a technical condition for controlling the complexity of $\mG$ by $\mH$. Assumption \ref{ass:technical}~(4) essentially assumes that we can model $\qpi$ (and $\vpi$) correctly at each $t$-step, which is again mild as $\mH^{(t)}$
for $t\geq 1$ can all be chosen as infinite-dimensional spaces. This assumption is also called realizability of value functions, which is commonly seen in the literature of RL \citep[e.g.,][]{antos_learning_2008}. Assumption \ref{ass:technical}~(5) is imposed to ensure that the space of testing functions $\mF$ is large enough so that we are able to capture the conditional expectation operator in each min-max estimation \eqref{eqn: minmax estimation}. Assumption \ref{ass:technical}~(6) basically requires a full coverage of our batch data generating process induced by the behavior policy, which is widely used in OPE \citep{precup2000eligibility,antos_learning_2008}. Next, we provide a key decomposition of the $\mL^2$ error for $V$-bridge estimation.

\begin{theorem}[Error decomposition]
\label{thm:decomposition}
  Under Assumption \ref{ass:technical} (1) and (6),  we can decompose
  the $\mL^2$ error of the estimated $V$-bridge by
\begin{align*}
  \norm{\vpi_1-\hvpi_1}_2 \leq \bar{\tau}_{1}\textstyle\sum_{t=1}^T \{\Pi_{t'=1}^{t}C_{t',t'-1}^{(t)}\}
  \tau_t\norm{\pi_t/\pi_t^b}_{\infty}\norm{\proj_t(\hmP_t-\mP_t)(\hvpi_{t+1}+R_t)}_2,
\end{align*}
where %
the measures of one-step transition ill-posedness $C_{1,0}^{(t)}\triangleq 1$ and
$C_{t',t'-1}^{(t)}, 2\leq t' \leq t \leq T$ are defined after Corollary \ref{cor:RKHS radii}. 
\end{theorem}
Theorem \ref{thm:decomposition} shows that there are four key components for upper bounding the $\mL^2$ error of $\hvpi_1$. The first component is the probability ratio  $\norm{\pi_t/\pi_t^b}_{\infty}$, which is used to measure the distributional mismatch between the target and behavior policies. The second component is 
$\norm{\proj_t(\hmP_t-\mP_t)(\hvpi_{t+1}+R_t)}_2$, the one-step projected error of $\hmP_t$ to $\mP_t$, where $\hvpi_{t+1}$ is the estimate for $\vpi_{t+1}$ depending on the observed data after $t$-step. We remark that this is different from the analysis in the standard NPIV estimation with a directly measured outcome. Hence the
results, e.g., from \cite{dikkala2020minimax}, cannot be directly applied to bound this component. The last two components are related to the (local) measure of ill-posedness. The third component $\tau_t$ is the measure of ill-posedness for characterizing the difficulty of estimating $\qpi_t$ by \eqref{eq:nonpara-identification} using $\mH^{(t)}$ at the $t$-th step. $\{\tau_t \}_{t=1}^T$ are similar to those used in the standard NPIV estimation such as \cite{chen2011rate}, and
the effect of each $\tau_t$ on the upper bound is cumulative.  The last component {\small $\{\Pi_{t'=1}^{t}C_{t',t'-1}^{(t)}\}_{t=1}^T$} quantify the propagation effect of estimation errors in previous steps
on the last step of estimating $\vpi_1$, which is multiplicative in terms of {\small$C_{t',t'-1}^{(t)}$}. 
We call {\small$C_{t',t'-1}^{(t)}$} the measure of one-step transition ill-posedness from $t'$ to $t'-1$ related to $t$-step NPIV estimations. %
Next we provide detailed bounds for the second and last components. The discussion of the third component can be found in Appendix \ref{sec:proofs}.4. %

\emph{Component 2: one-step projected error}.
In the following, we show
that $\norm{\proj_t(\hmP_t-\mP_t)(\hvpi_{t+1}+R_t)}_2$ is bounded by the
critical radii of some spaces defined as balls $\mH_{B}^{(t)}$, $\mG_{C_G(T-t+1)M_{\mH}}^{(t+1)}$ %
in hypothesis spaces $\mH^{(t)}$,
$\mG^{(t+1)}$ respectively and a ball $\mF_{3M}^{(t)}$ %
in testing space $\mF^{(t)}$, for some fixed constants
$M, B>0$ such that functions in $\mH_{B}^{(t)}$ and $\mF_{3M}^{(t)}$ have uniformly bounded ranges in $[-1,1]$ for all $1\leq t\leq T$.
Let
\begin{align*}
  \boldsymbol{\Omega}^{(t)} &=\{ (s_t,w_t,z_t,a_t,s_{t+1},w_{t+1}) \mapsto r(h_g^{*}(w_t,s_t,a_t) - g(w_{t+1},s_{t+1}))f(z_t,s_t,a_t):\\
    &\qquad g\in\mG_{C_G(T-t+1)M_{\mH}}^{(t+1)}, f\in\mF_{3M}^{(t)}, r\in[0,1] \},\text{ and}\\
  \boldsymbol{\Xi}^{(t)} &= \{(s_t,w_t,z_t,a_t)\mapsto r[h-h_g^{*}](w_t,s_t,a_t)f^{L^2 B}(z_t,s_t,a_t):\\
                     & \qquad h\in\mH^{(t)}, h-h_g^{*}\in\mH_B^{(t)},g\in\mG_{C_G(T-t+1)M_{\mH}}^{(t+1)}, r\in[0,1]\},
\end{align*}
where $h_g^{*}\in\mH^{(t)}$ is the solution to $\EE \left[ h(W_t,S_t,A_t) -
  g(W_{t+1},S_{t+1})\given Z_t,S_t,A_t \right]=0$, and $f^{L^2 B} =
\argmin_{f\in\mF_{L^2 B}^{(t)}}\norm{f-\proj_t(h-h_g^{*})}_2$ for a given $L>0$. An upper bound for $\norm{\proj_t(\hmP_t-\mP_t)(\hvpi_{t+1}+R_t)}_2$ is given in Theorem \ref{thm:one-step}.
\begin{theorem}
  \label{thm:one-step}
Suppose that Assumption \ref{ass:technical} holds.
Let $\delta_n^{(t)}=\bar\delta_n^{(t)}+c_0\sqrt{\frac{\log(c_1T/\zeta)}{n}}$ for some universal constants $c_0,c_1>0$ where $\bar\delta_n^{(t)}$ is the upper
bound of the critical radii of $\mF_{3M}^{(t)}$, $\boldsymbol{\Omega}^{(t)} $ and $\boldsymbol{\Xi}^{(t)}$.
Assume that the approximation error in Assumption \ref{ass:technical} (5) can be bounded by $\eta_n^{(t)}\leq \delta_n^{(t)}$. 
Furthermore, letting
tuning parameters satisfy $M\lambda \asymp (\delta_n^{(t)})^2$ and 
$\mu\geq \bigO(L^2+M/B)$, 
with probability at least $1-\zeta$, we have 
\begin{align*}
  \norm{\proj_t(\hmP_t - \mP_t)(\hvpi_{t+1}+R_t)}_2 \lesssim M_{\mH}(T-t+1)^2\delta_n^{(t)} \quad \text{for all $1\leq t\leq T$}.
\end{align*}
\end{theorem}
Depending on the choices of $\mH^{(t)}$, $\mG^{(t+1)}$, and $\mF^{(t)}$, 
we can obtain different finite-sample error bounds of the one-step projected error for each $t$. Below we provide two examples.

\begin{corollary}
\label{cor:VC radii}
  Let $\mF^{(t)}$, $\mH^{(t)}$ and $\mG^{(t+1)}$ be VC-subgraph classes  with VC dimensions $\VV(\mF^{(t)})$, $\VV(\mH^{(t)})$ and $\VV(\mG^{(t+1)})$ respectively. Then  
  with probability at least $1-\zeta$, for all $1\leq t\leq T$,
\begin{align*}
  \norm{\proj_t(\hmP_t - \mP_t)(\hvpi_{t+1}+R_t)}_2 \lesssim (T-t+1)^{2.5}\textstyle\sqrt{\frac{\log(c_1T/\zeta)\max\{\VV(\mF^{(t)}),\VV(\mH^{(t)}),\VV(\mG^{(t+1)})\}}{n}}. %
\end{align*}
\end{corollary}

The definition of the VC-subgraph class can be found in, e.g., \cite{wainwright2019high}. This is a broad class. For example, if one lets each of $\mF^{(t)}$, $\mH^{(t)}$ and $\mG^{(t+1)}$ be a linear space
$\mF = \{\theta^{\top}\phi(\cdot):\theta\in\RR^d\}$ with basis functions
$\phi(\cdot)$, then $\VV(\mF)=d+1$. Then the upper bound for the one-step projected error becomes $\bigO((T-t+1)^{2.5} d/\sqrt{n})$. %

\begin{corollary}
  \label{cor:RKHS radii}
  Let $\mH^{(t)}$, $\mG^{(t+1)}$ and $\mF^{(t)}$ be reproducing kernel
  Hilbert spaces (RKHSs) equipped with kernels $K_{\mH^{(t)}}$, $K_{\mG^{(t+1)}}$ and $K_{\mF^{(t)}}$
  respectively. For a given positive definite kernel $K$, we denote its
  nonincreasing eigenvalue sequence by $\{\lambda_j^{\downarrow}(K)\}_{j=1}^{\infty}$. We consider two scenarios for $\{\lambda_j^{\downarrow}(K)\}_{j=1}^{\infty}$. 
  
  (1) \textbf{Polynomial eigen-decay}: If
  $\lambda_j^{\downarrow}(K_{\mH^{(t)}})\leq aj^{-2\alpha_{\mH}}$, $\lambda_j^{\downarrow}(K_{\mG^{(t+1)}})\leq aj^{-2\alpha_{\mG}}$ and $\lambda_j^{\downarrow}(K_{\mF^{(t)}})\leq aj^{-2\alpha_{\mF}}$ for constants $\alpha_{\mH},\alpha_{\mG},\alpha_{\mF}>1/2$ 
    and $a>0$, then under the assumptions in Theorem \ref{thm:one-step}, %
     with probability at least $1-\zeta$, for all $1\leq t \leq T$, we have 
    \begin{align*}
        \norm{\proj_t(\hmP_t - \mP_t)(\hvpi_{t+1}+R_t)}_2 \lesssim (T-t+1)^{2.5}\sqrt{\log(c_1T/\zeta)} n^{-\frac{1}{2+\max\{1/\alpha_{\mH},1/\alpha_{\mG},1/\alpha_{\mF}\}}}\log(n).
    \end{align*}
    (2) \textbf{Exponential eigen-decay}: If
    $\lambda_j^{\downarrow}(K_{\mH^{(t)}})\leq a_1 e^{-a_2 j^{\beta_{\mH}}}$,
    $\lambda_j^{\downarrow}(K_{\mG^{(t+1)}})\leq a_1 e^{-a_2 j^{\beta_{\mG}}}$ and
    $\lambda_j^{\downarrow}(K_{\mF^{(t)}})\leq a_1 e^{-a_2 j^{\beta_{\mF}}}$, for
    constants $a_1,a_2,\beta_{\mH},\beta_{\mG},\beta_{\mF}>0$, then under the assumptions in Theorem \ref{thm:one-step}, 
   with probability at least $1-\zeta$, for all $1\leq t \leq T$, we have
    \begin{align*}
        \norm{\proj_t(\hmP_t - \mP_t)(\hvpi_{t+1}+R_t)}_2 
      \lesssim (T-t+1)^{2.5}\textstyle\left\{\sqrt{\frac{(\log n)^{1/\min\{\beta_{\mH},\beta_{\mG},\beta_{\mF}\}}}{n}} +  \sqrt{\frac{\log(c_1T/\zeta)}{n}}\right\}.
    \end{align*}
\end{corollary}
Kernels of the two types of eigen-decay considered above are very common. For example, the kernel of the $\alpha$-order Soblev space with $\alpha>1/2$, has a polynomial eigen-decay while the Gaussian kernel has an exponential eigen-decay, with $\beta=2$ for Lebesgue measure on real line and $\beta=1$ on a compact domain \citep{wei2017early}.%

\emph{Components 4: measure of one-step transition ill-posedness}.
We first provide more insights on  $\{\Pi_{t'=1}^{t}C_{t',t'-1}^{(t)}\}_{t=1}^T$ before providing an upper bound.
We formally define the local measure of one-step transition ill-posedness
$C_{t'+1,t'}^{(t)}$ 
recursively based on $C_{t,t-1}^{(t)}$ to $C_{2,1}^{(t)}$ as 
\begin{eqnarray*}
&& \hspace{-0.3in} C_{t'+1,t'}^{(t)}\triangleq\sup_{g\in\mG(W_{t'+1}\times S_{t'+1})} \frac{\norm{\EE^{\pi_{t'}}[g(W_{t'+1},S_{t'+1})\given Z_{t'},S_{t'}]}_2}{\norm{\EE[g(W_{t'+1},S_{t'+1})\given Z_{t'+1},S_{t'+1}]}_2}, \quad    \text{subject to}\\
&& \hspace{-0.4in} \norm{\EE[g(W_{t'+1},S_{t'+1})\given Z_{t'+1},S_{t'+1}]}_2
                   \lesssim  \tau_t (T-t+1)^2\delta_n^{(t)} \norm{\pi_{t'}/\pi_{t'}^b}_{\infty}\textstyle\prod_{s=t'+1}^{t-1}C_{s+1,s}^{(t)},
                   \end{eqnarray*}
with {\small $C_{t+1,t}^{(t)}\triangleq 1$} for each $t=1,\dots,T$. 
For {\small $C_{t,t-1}^{(t)}$}, we can upper bound $\norm{\EE[\hvpi(W_{t},S_{t})\given Z_{t},S_{t}]}_2$ by the projected error $\norm{\proj_t(\hmP_t-\mP_t)(\hvpi_{t+1}+R_t)}_2$ multiplied by the ill-posedness $\tau_t$.
By Theorem \ref{thm:one-step}, the projected error can be well controlled by $\delta^{(t)}_{n}$ with high probability, so {\small $C_{t,t-1}^{(t)}$} can be defined locally. Therefore, we can provide an upper bound for the denominator sequentially and define all {\small $C_{t'+1,t'}^{(t)}$} locally, which indicates that all {\small $C_{t'+1,t'}^{(t)}$} could be small.

For example, if we use the observed history as the action-inducing proxy, %
then
$\sigma(\mZ_1\times\mS_1)\subset\sigma(\mZ_2\times\mS_2)\subset\dots\subset\sigma(\mZ_T\times\mS_T)$
is a filtration. 
In this case, {\small $\prod_{t'=2}^t C_{t',t'-1}^{(t)}$} are expected to be small for $t\geq 2$ if the target policy is stationary.
While this can enlarge the critical radii $\delta_n^{(t)}$ due to the dimension of the action-inducing proxy, this only affects one-step errors. %
See detailed discussion in Appendix \ref{sec:additional results}. Motivated by this, it is reasonable to impose Assumption \ref{ass: one step transition ill-posedness} below on $C_{t'+1,t'}^{(t)}$.
\begin{assumption}\label{ass: one step transition ill-posedness}
 For every $t \geq 2$ and $2 \leq m \leq t$, $C_{m,m-1}^{(t)}\leq 1 + \frac{a_{t}}{m^{\alpha_t}}$ with time-dependent constants $a_{t}>0, \alpha_t\geq\alpha>1$.
\end{assumption}
\begin{corollary}
\label{cor: transition ill-posedness bound}
If Assumption \ref{ass: one step transition ill-posedness} holds, then $\prod_{t'=1}^t C_{t',t'-1}^{(t)} \leq \exp\{a_t\zeta(\alpha_t)\}$, where $\zeta(\alpha_t) = \sum_{n=1}^{\infty}(1/n)^{\alpha_t}$ is uniformly bounded for $t \geq 1$.
\end{corollary}
\textbf{\emph{Main result: error bounds for $V$-bridge estimation and OPE}}. 
Define $\texttt{trans-ill}=\max_{1 \leq t \leq T}\exp\{a_t\zeta(\alpha_t)\}$
and let $\texttt{ill}_{\max}=\bar{\tau}_{1}\max_{1\leq t\leq T}\tau_t\norm{\pi_t/\pi_t^b}_\infty$. Summarizing all aforementioned results, we have the following main theorem based on the polynomial eigen-decay case in Corollary \ref{cor:RKHS radii}. Other cases can be found in Appendix \ref{sec:additional results}.
\begin{theorem}[Finite-sample error bounds for $V$-bridges and policy value]
\label{thm:main}
Under Assumptions \ref{ass:technical} and \ref{ass: one step transition ill-posedness}, and assumptions in Theorem
\ref{thm:one-step} and Corollary \ref{cor:RKHS radii}~(1), with probability at least $1-\zeta$, we have
\begin{align*}
 \norm{\vpi_1-\hvpi_1}_2  & \lesssim \texttt{ill}_{\max} \times \texttt{trans-ill}\times T^{7/2}\sqrt{\log(c_1T/\zeta)} n^{-\frac{1}{2+\max\{1/\alpha_{\mH},1/\alpha_{\mG},1/\alpha_{\mF}\}}}\log(n), \text{and}\\
  |\mV(\pi) - \hat\mV(\pi)| & \lesssim \texttt{ill}_{\max} \times \texttt{trans-ill}\times T^{7/2}\sqrt{\log(c_1T/\zeta)}n^{-\frac{1}{2+\max\{1/\alpha_{\mH},1/\alpha_{\mG},1/\alpha_{\mF}\}}}\log(n).
\end{align*}
\end{theorem}
Theorem \ref{thm:main} provides the first finite-sample error bound for OPE
under confounded and episodic POMDPs in terms of the sample size, length of
horizon and two (\textit{local}) measures of ill-posedness. Without considering the measures of ill-posedness, the
derived error bound for $V$-bridge function nearly achieves the optimal
$\mL^2$-convergence rate in the classical non-parametric regression
\citep{stone1982optimal}. Moreover, our OPE error bound depends on a polynomial
order of $T$, i.e., $T^{7/2}$, which is larger than the standard $\bigO(T^3)$ in
the OPE without unobserved variables. However, when the function class consider in \eqref{eqn: minmax estimation} grows with the sample size $n$, $\texttt{ill}_{\max}$ will also
increase and therefore the convergence rates in Theorem \ref{thm:main} could be
much slower. 
Next we study a case when we can control the local measures of ill-posedness
$\{\tau_t\}_{t=1}^T$, by assuming that
$\lambda_{\min}(\Gamma_m^{(t)})\geq \nu_m$ for all $1\leq t\leq T$ almost surely and other regularity conditions in Lemma \ref{lem:
  ill-posedness bound}, where
$
\Gamma_m^{(t)} \triangleq \EE \left\{ \EE[e_I^{(t)}(W_t,S_t,A_t)\given Z_t,S_t,A_t] \EE[e_I^{(t)}(W_t,S_t,A_t)\given Z_t,S_t,A_t]^{\top} \right\}$ with
$e_I^{(t)}=(e_1^{(t)},\dots,e_m^{(t)})$ as the first $m$ eigenfunctions of
kernel $K_{\mH^{(t)}}$. Similar conditions can be imposed to control $\bar{\tau}_1$, which is
omitted here for simplicity.
Let
$
\eta(n, T, \zeta, \alpha_\mathcal{H}, \alpha_\mathcal{F}, \alpha_\mathcal{G}, b) \triangleq  T^{\frac{7(\alpha_{\mH}-1/2)+10b}{2(\alpha_{\mH}-1/2)+4b}} \big(\sqrt{\log(c_1T/\zeta)}  n^{-\frac{1}{2+\max\{1/\alpha_{\mH},1/\alpha_{\mG},1/\alpha_{\mF}\}}}\log(n)\big)^{\frac{\alpha_{\mH}-1/2}{\alpha_{\mH}-1/2+2b}},
$
with $b$ defined below.

\begin{corollary} \label{cor: ill-posedness main}
If assumptions in Theorem \ref{thm:main} holds and $\nu_m \geq m^{-2b}$ for some $b \geq 0$,  then
\begin{align*}
 \norm{\vpi_1-\hvpi_1}_2  & \lesssim \bar{\tau}_1\textstyle\max_{1\leq t\leq T}\norm{\pi_t/\pi_t^b}_{\infty} \times \texttt{trans-ill}\times %
 \eta(n, T, \zeta, \alpha_\mathcal{H}, \alpha_\mathcal{F}, \alpha_\mathcal{G}, b),\\
|\mV(\pi) - \hat\mV(\pi)| & \lesssim \bar{\tau}_1\textstyle\max_{1\leq t\leq T}\norm{\pi_t/\pi_t^b}_{\infty} \times \texttt{trans-ill}\times %
  \eta(n, T, \zeta, \alpha_\mathcal{H}, \alpha_\mathcal{F}, \alpha_\mathcal{G}, b).
\end{align*}
\end{corollary}
Corollary \ref{cor: ill-posedness main} considers the mildly ill-posed case, i.e., $\nu_m \geq m^{-2b}$, and shows that the local measure of ill-posedness can
deteriorate convergence rate of $\hvpi_1$ significantly. If $b$ is
large relative to $\alpha_{\mH}$ or further severely ill-posed case is considered (i.e., $\nu_m$ decays exponentially
fast, see Appendix \ref{sec: applying decompisition of OPE error}), then the convergence rate of $V$-bridge estimation could be much slower and the typical requirement on the nuisance parameter for achieving asymptotic normality for the policy value will fail. On the other hand,  it can be seen that when $b=0$, the finite sample error bounds match the results
in Theorem \ref{thm:main}.

\section{Simulation}
In this section, we perform a simulation study to evaluate the performance of our proposed OPE estimation and to verify the finite-sample error bound of our OPE estimator in Theorem \ref{thm:main}.

Let $\mS=\RR^2$, $\mU=\RR,\mW=\RR,\mZ=\RR$, and $\mA=\left\{ 1,-1 \right\}$.
At time $t$, the hidden state $U_t$, two proximal variables $Z_t$, $W_t$ satisfy the following multivariate normal distribution given $(S_t,A_t)$:
\begin{equation}
\label{eq:ZWU|SA}
(Z_t,W_t,U_t)\given (S_t,A_t) \sim \mN \left(
  \begin{bmatrix}
    \alpha_0 +\alpha_aA_t +\alpha_sS_t\\
    \mu_0    +\mu_aA_t    +\mu_sS_t\\
    \kappa_0 +\kappa_aA_t +\kappa_sS_t\\
  \end{bmatrix},
  \Sigma = \begin{bmatrix}
    \sigma_z^2 & \sigma_{zw} & \sigma_{zu}\\
    \sigma_{zw} & \sigma_w^2 & \sigma_{wu}\\
    \sigma_{zu} & \sigma_{wu} & \sigma_u^2
  \end{bmatrix}
\right), 
\end{equation}
where parameters are given in the Appendix.

The behavior policy is given by
$\tilde{\pi}_t^b (A_t\given U_t,S_t) = \expit \left\{ -A_t \left( t_0 + t_uU_t+t_s^{\top}S_t \right) \right\}$,
where $t_0=0$, $t_u=1$, and $t_s^{\top}=[-0.5, -0.5]$. 
Then by Assumption \ref{ass:technical} (6), $\pi_t^{b}(A_t\given S_t) = \expit\{-A_t\left(t_0+t_u\kappa_0 + (t_s+t_u\kappa_s)^{\top}S_t\right)\}$. %
The initial $S_1$ is uniformly sampled from $\RR^2$.
At time $t$, given $(S_t,U_t,A_t)$, we generate
$S_{t+1} = S_t + A_tU_t\ones_2 + e_{S_{t+1}}$,
where $\ones_2 = [1,1]^{\top}$ and the random error $e_{S_{t+1}}\sim \mN([0,0]^{\top}, \bI_2)$ with $\bI_2$ denoting the $2$-by-$2$ identity matrix. The reward is given by
$R_t = \expit\left\{\frac{1}{2}A_t(U_t+[1,-2]S_t)\right\} + e_t,$
where $e_t\sim \text{Uniform}[-0.1,0.1]$. One can verify that our simulation setting satisfies the conditions in Section \ref{sec: basic assumptions} so that our method can be applied.
\begin{figure}[H]
    \centering
    \includegraphics[width=0.4\textwidth]{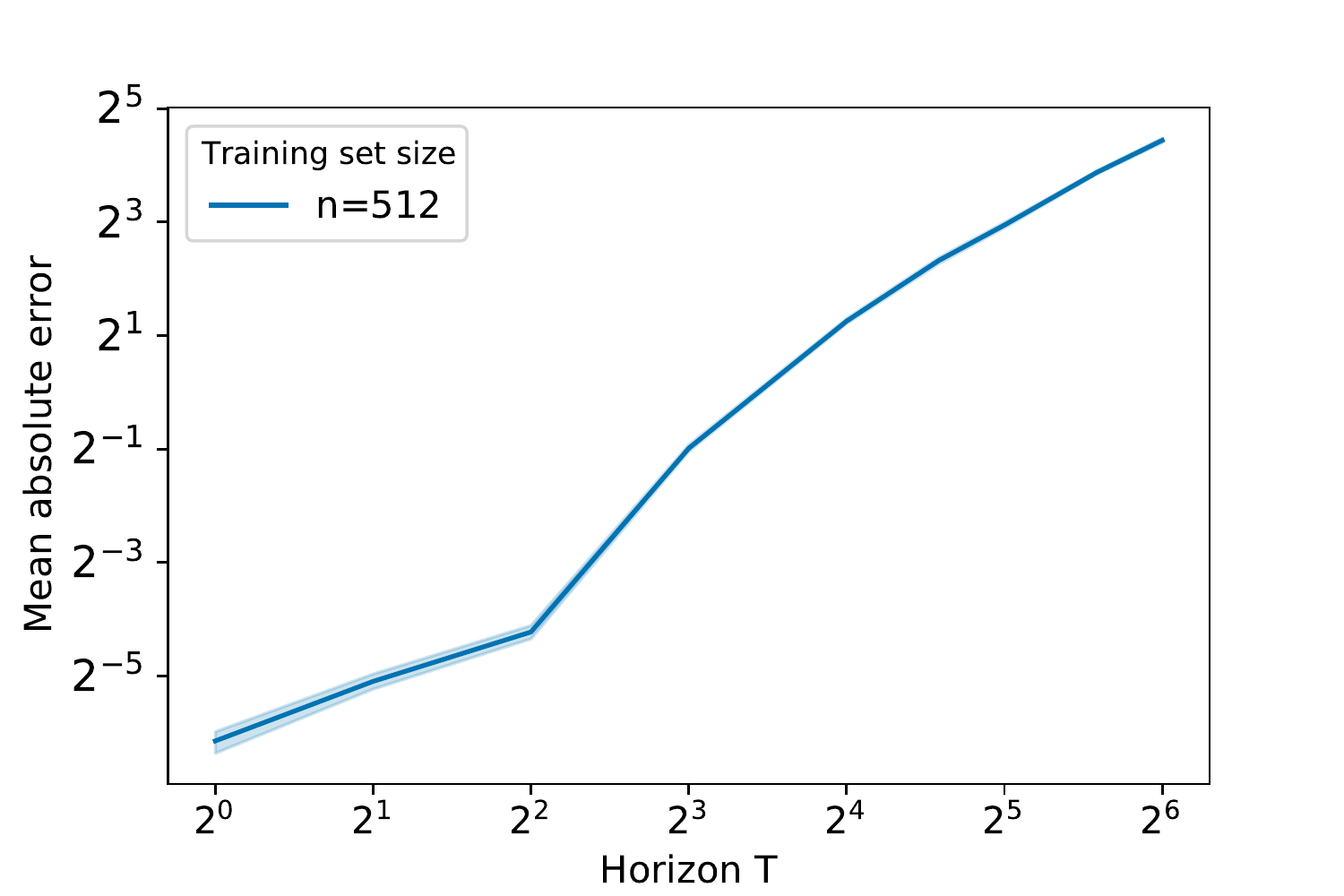}
    \includegraphics[width=0.4\textwidth]{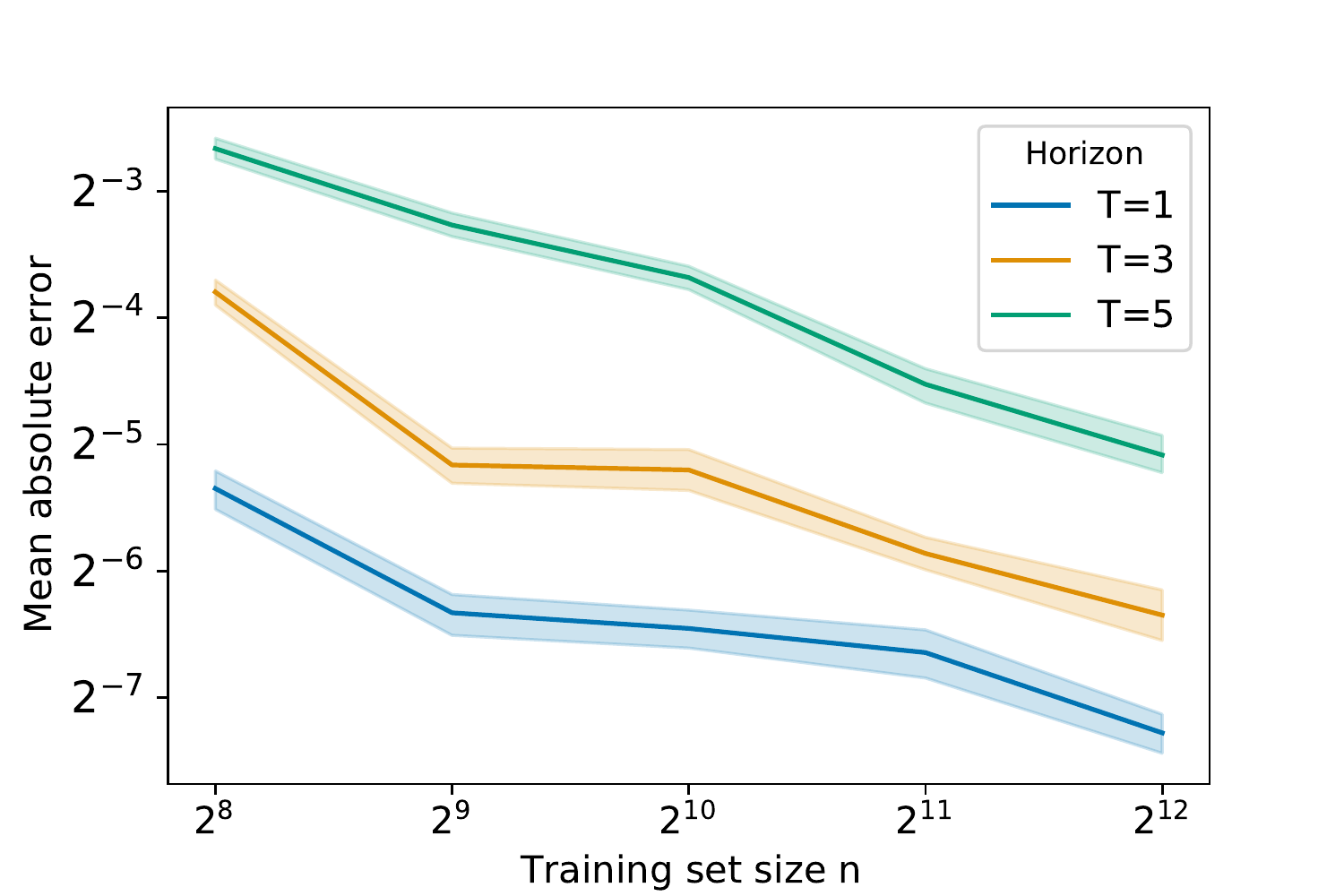}
    
     (a)  ~~~~~~~~~~~~~~~~~~~~~~~~~~~~~~~~~~~~~~~~~~~~~~~~~~~~~~~~~~~~~~~~~~~~~(b) 
    \caption{Simulation results for OPE errors $|\widehat{\mV}(\pi) - \mV(\pi)|$. Mean absolute errors (solid lines) and their standard error bands (shaded regions) are displayed for different combinations of $(n,T)$.} %
    \label{fig:simu_results}
\end{figure}

We choose $\mF^{(t)}$ and $\mH^{(t)}$ as RKHSs endowed with Gaussian kernels, with bandwidths selected according to the median heuristic trick by \cite{fukumizu2009kernel} for each $1\leq t\leq T$. The pool of scaling factors \texttt{SCALE} contains 30 positive numbers spaced evenly on a log scale between 0.001 to 0.05. The number of cross-validation partition $K=5$.
The true target policy value of $\pi$ is estimated by the mean cumulative rewards of $50,000$ Monte Carlo trajectories with policy $\pi$. 
We compare our OPE estimator $\widehat{\mathcal{V}}(\pi)$ with the target policy value by computing mean absolute error (MAE) for each setting of $(n,T)$, as reported in Figure \ref{fig:simu_results}. Figure \ref{fig:simu_results} validate the derived finite-sample error bound of our OPE estimator in Theorem \ref{thm:main}. Specifically,
Figure \ref{fig:simu_results} (a) shows that the OPE estimation error is polynomial in $T$, but with an order slightly smaller than $\bigO(T^{7/2})$ as stated in Theorem \ref{thm:main}. Figure \ref{fig:simu_results} (b) shows that the convergence rate in terms of the sample size $n$ for our OPE estimator is slower than $\bigO(n^{-1/2})$, which also justifies our theoretical results.

\section{Discussion}

In this paper, we propose a non-parametric identification and estimation method for OPE in episodic confounded POMDPs with continuous states, relying on time-dependent proxy variables. 
We develop a fitted-$Q$-evaluation-type algorithm for estimating the $V$-bridge functions sequentially and for OPE based on the estimated $V$-bridges. The first finite-sample error bound for estimating the policy value under confounded POMDPs is established, which achieves a polynomial order with respect to the sample size and the length of horizon. Our OPE results can serve as a foundation for developing new policy optimization algorithms in the confounded POMDP, which We will leave for future work.

\section*{Acknowledgement}
Zhang's research is partially supported by GW University Facilitating Fund.

\bibliographystyle{abbrvnat}
\bibliography{reference}

\newpage

\appendix

{\centering 
\large\bf Supplementary Material:\\
Off-Policy Evaluation for Episodic Partially Observable Markov Decision Processes under Non-Parametric Models

}

\tableofcontents
\newpage
\allowdisplaybreaks[0]
\section*{List of Notations}
\begin{table}[H]
    \centering
    \begin{tabular}{c|c}
    \hline
    $\mM$ & the episodic and confounded POMDP\\
    $S_t\in\mS$ & observed state at $t$ and observed state space\\
    $U_t\in\mU$ & unobserved state at $t$ and unobserved state space\\
    $A_t\in\mA$ & action at $t$ and discrete action space\\
    $T$   & length of horizon\\
    $r = \{r_t\}_{t=1}^T$ & reward functions over $\mS\times\mU\times\mA$\\
    $R_t$ & reward at $t$\\
    $W_t\in\mW$ & reward-proxy variable at $t$ and corresponding space\\
    $Z_t\in\mW$ & action-proxy variable at $t$ and corresponding space\\
    $X_{t,i}$ & variable $X$ at $t$ from sample trajectory $i$\\
    $\pi=\{\pi_t\}_{t=1}^T$ & target policy depending on $S_t$\\
    $\tilde{\pi}_t^b$ & behavior policy at $t$ depending on $S_t,U_t$\\
    $V_t^{\pi}(s,u)$ & state value function\\
    $\mV(\pi)$ ($\widehat{\mV}(\pi)$) & (estimated) policy value of a target policy $\pi$\\
    $\vpi_t$ ($\hvpi_t$) & (estimated) V-bridge function (or V-bridge for short) at $t$\\
    $\qpi_t$ ($\hqpi_t$) & (estimated) Q-bridge function (or Q-bridge for short) at $t$\\
    $\widetilde{\mP}_t$ & operator $[\widetilde{\mP}_t](Z_t,S_t,A_t) = \EE[g(R_t,W_{t+1},S_{t+1})\given Z_t,S_t,A_t]$ \\
    $\overline{\mP_t}$ & operator $[\overline{\mP}_t](Z_t,S_t,A_t) = \EE[h(W_t,S_t,A_t)\given Z_t,S_t,A_t]$ \\ 
    $\mP_t$ ($\widehat{\mP}_t$) & operator $\mP_t g = \overline{\mP_t}^{-1}\widetilde{\mP}_t g$ (estimator of $\mP_t$ defined in \eqref{eqn: minmax estimation})\\
    $\Ppi_t$ ($\hPpi_t$) & operator $\Ppi_t g = \langle \pi_t, \mP_t g\rangle$ (estimator of $\Ppi_t$: $\hPpi_tg = \left\langle \pi_t, \widehat{\mP}_t g\right\rangle$)\\
    $\mH^{(t)}$ & user-defined function space on $\mW\times\mS\times\mA$\\
    $\mF^{(t)}$ & user-defined function space on $\mZ\times\mS\times\mA$\\
    $\mG^{(t)}$ & user-defined function space on $\mZ\times\mS$\\
    $\mR_n(\mF,\delta)$ & local Rademacher complexity for function class $\mF$ and radius $\delta>0$\\
    $\widehat{\mR}_n(\mF,\delta)$ & local empirical Rademacher complexity for function class $\mF$ and radius $\delta>0$\\
    $N_n(\epsilon,\mG)$ & the smallest empirical $\epsilon$-covering number of $\mG$\\
    $\alpha \mF$ & $\alpha \mF = \{\alpha f: f\in\mF\}$ for some $\alpha\in\RR$\\
    $\mF_B$ & $\mF_B = \{f\in\mF: \nmF{f}^2 \leq B\}$ for any $B>0$\\
    $\norm{\proj_t f}_2$ & $\norm{\proj_t f}_2 = \sqrt{\EE\{f(X)\given Z_t,S_t,A_t\}^2}$\\
    $\bar{\tau}_{1}$ & ill-posedness $\bar{\tau}_{1}=
\sup_{g\in\mG^{(1)}}\norm{g(W_1,S_1)}_2/\norm{\EE[g(W_1,S_1)\given Z_1,S_1]}_2$\\
$\tau_t$ & ill-posedness $\tau_t=\sup_{h\in\mH^{(t)}}\norm{h(W_t,S_t,A_t)}_2/\norm{\proj_th(W_t,S_t,A_t)}_2$ \\
$C_{t',t'-1}^{(t)}$ & one-step transition ill-posedness defined after Corollary \ref{cor:RKHS radii}\\
$\VV(\mF)$ & VC dimension of $\mF$\\
$\zeta(\alpha)$ & Riemann Zeta function $\zeta(\alpha) = \sum_{n=1}^{\infty} (1/n)^{\alpha}$\\
$\text{Ker}(K)$ & $\text{Ker}(K) = \{g: Kg=0\}$ null space of linear operator $K$\\
$A^{\perp}$ & orthogonal complement of space $A$\\
$|\mZ|$ & cardinality of class $Z$\\
    
    \hline
    \end{tabular}
    \caption{List of Notations}
    \label{tab: notations}
\end{table}

\newpage

\section{Additional Identification Assumptions}
\label{sec:additional assumption}

In this section, we list Assumptions \ref{ass: Markovian}-\ref{ass:regularity} which are needed for Theorem \ref{thm:existence}.

\subsection{Basic assumptions on the confounded POMDP structure}
\label{sec: basic assumptions}

For the confounded POMDP with trajectory $\left(U_{t}, S_{t}, W_{t}, Z_{t}, A_{t}, R_{t}\right)_{t=1}^T$, 
we list three basic assumptions below. 
Let $\indep$ and $\dep$ denote statistical independence and dependence respectively.

\begin{assumption}[Markovian] \label{ass: Markovian}
For all $1\leq t\leq T$, the time-variant transition kernel $\PP_t$ satisfies that for any $(s,u)\in\mS\times\mU$, $a\in\mA$ and set $F\in\mB(\mS\times\mU)$,
\begin{align*}
&\Pr((S_{t+1},U_{t+1})\in F\given S_t=s, U_t=u,A_t=a, \{S_j,U_j,A_j\}_{1\leq j <t}) \\
 &= \PP_t((S_{t+1},U_{t+1})\in F\given S_t=s,U_t=u,A_t=a),
\end{align*}
where $\mB(\mS\times\mU)$ is the family of Borel subsets of $\mS\times\mU$ and 
$\{S_j,U_j,A_j\}_{1\leq j <t}\neq\emptyset$ if $t=1$. 
\end{assumption}

\begin{assumption}[Reward proxy]
\label{ass:RewardProxy} 
$W_t\indep (A_t,U_{t-1},S_{t-1}) \given U_t, S_t$ and $W_t\dep U_t\given S_t$ for $1\leq t\leq T$.
\end{assumption}
\begin{assumption}[Action proxy]
\label{ass:ActionProxy} %
  $Z_t\indep W_t \given (U_t, S_t,A_t)$,
  $Z_t\indep R_t \given (U_t, S_t,A_t)$ and
  $Z_t\indep (S_{t+1}, W_{t+1}) \given (U_t, S_t, A_t)$, $1\leq t \leq T$.
\end{assumption}

It can be easily verified that the DAG in Figure \ref{fig:POMDP}
satisfies Assumptions \ref{ass: Markovian}-\ref{ass:ActionProxy}.
Assumption \ref{ass: Markovian} requires that given the current full state and action $(U_t,S_t,A_t)$, the future are independent of the past.

Assumption \ref{ass:RewardProxy} requires that the reward proxy $W_t$ is associated
with the hidden state $U_t$ after adjusting observed state $S_t$ but $W_t$ is
not causally affected by action $A_t$ and past state $(U_{t-1},S_{t-1})$ after adjusting the full current state $(U_t,S_t)$. 
This assumption does not restrict the association between $W_t$ and $R_t$.
Assumption \ref{ass:ActionProxy} requires that upon conditioning on the current
full state and action tuple $(U_t,S_t,A_t)$, the action proxy $Z_t$ does not
affect the reward proxy $W_t$ and outcomes $R_t,S_{t+1},W_{t+1}$
after the action $A_t$. Again, this assumption does not restrict the association between $Z_t$ and $A_t$.

However, based on above three assumptions, we cannot directly identify the value of target policy $\pi$ by adjusting $(U_t,S_t)$ since $U_t$ is unobserved. In addition to Assumptions \ref{ass:RewardProxy} and \ref{ass:ActionProxy}, we also need Assumption 
\ref{ass:completeness} to be stated in Section \ref{sec: bridge assumption} below to get around the hidden state $U_t$. 

\subsection{Assumptions on the existence of bridge functions}
\label{sec: bridge assumption}

\begin{assumption}[Completeness] \label{ass:completeness} For any $(s,a)\in\mS\times\mA$, $t=1,\dots,T$,
\begin{enumerate}[]
\item[(a)] For any square-integrable function $g$,
  $\EE\{g(U_t)\given Z_t, S_t=s, A_t=a\}=0$ a.s. if and only if $g=0$ a.s;
\item[(b)] For any square-integrable function $g$,
  $\EE\{g(Z_t)\given W_t, S_t=s, A_t=a\}=0$ a.s. if and only if $g=0$ a.s.
\end{enumerate}
\end{assumption}
Completeness is a commonly made technical assumption in value identification problems,
e.g., instrumental variable identification
\citep{newey2003instrumental,d2011completeness,chen2014local}, and proximal
causal inference \citep{miao2018identifying,miao2018confounding,tchetgen2020introduction}.
Together with the regularity conditions in Assumption \ref{ass:regularity}, we can
ensure the existence of $Q$-bridges $\qpi_t$ and $V$-bridges $\vpi_t$, $1\leq
t\leq T$.

For a probability measure function $\mu$, let $\mL^2\{\mu(x)\}$
denote the space of all squared integrable functions of $x$ with respect to
measure $\mu(x)$, which is a Hilbert space endowed with the inner product
$\left\langle g_1,g_2 \right\rangle = \int g_1(x)g_2(x){\rm d}\mu(x)$. 
For all $s,a,t$, define the following operator 
\begin{align*}
  K_{s,a;t}:  \mL^2 \left\{ \mu_{W_t\given S_t,A_t}(w\given s,a)\right\} &\rightarrow \mL^2 \left\{ \mu_{Z_t\given S_t,A_t}(z\given s,a) \right\}\\
 h &\mapsto\EE \left\{ h(W_t)\given Z_t=z, S_t=s, A_t=a \right\},
\end{align*}
and its adjoint operator
\begin{align*}
  K_{s,a;t}^{*}: \mL^2 \left\{ \mu_{Z_t\given S_t, A_t}(z\given s,a)\right\}&\rightarrow \mL^2 \left\{ \mu_{W_t\given S_t,A_t}(w\given s,a) \right\}\\ 
 g &\mapsto \EE \left\{ g(Z_t)\given W_t=w, S_t=s, A_t=a \right\}.
\end{align*}

\begin{assumption}[Regularity Conditions] \label{ass:regularity}

For any $Z_t=z, S_t=s, W_t=w, A_t=a$ and $1\leq t\leq T$,
\begin{enumerate}[]
\item[(a)] $\iint_{\mW\times\mZ} f_{W_t\given Z_t,S_t,A_t}(w\given z,s,a)
  f_{Z_t\given W_t,S_t,A_t}(z\given w,s,a) {\rm d}w {\rm d}z < \infty$,
  where $f_{W_t\given Z_t,S_t,A_t}$ and $f_{Z_t\given W_t,S_t,A_t}$ are conditional density functions.

\item[(b)] For
  any $g\in\mG^{(t+1)}$,
$$
\int_{\mZ} \left[ \EE \left\{ R_t + g(W_{t+1},S_{t+1})\given
      Z_t=z,S_t=s,A_t=a \right\} \right]^2 f_{Z_t \given S_t, A_t}(z\given s,a){\rm d}z < \infty.
      $$
\item[(c)] There exists a  singular decomposition $\left(
    \lambda_{s,a;t;\nu},\phi_{s,a;t;\nu},\psi_{s,a;t;\nu} \right)_{\nu=1}^{\infty}$ of
  $K_{s,a;t}$ such that for all $g\in\mG^{(t+1)}$,
  \begin{equation*}
    \sum_{\nu=1}^{\infty} \lambda_{s,a;t;\nu}^{-2} \left| \left\langle \EE \left\{ R_t+g(W_{t+1},S_{t+1})\given Z_t=z,S_t=s,A_t=a \right\},\psi_{s,a;t;\nu} \right\rangle \right|^2 < \infty.
  \end{equation*}
\item[(d)] For all $1\leq t \leq T$, $\vpi_t\in\mG^{(t)}$ where $\mG^{(t)}$ satisfies the regularity conditions (b) and (c) above.
\end{enumerate}
Note that the existence of the singular decomposition of $K_{s,a,;t}$ in Assumption \ref{ass:regularity} (c) can be ensured by Assumption \ref{ass:regularity} (a), which is a sufficient condition for the compactness of
$K_{s,a;t}$ by Lemma \ref{lem:Picard}. 
\end{assumption}

\noindent For tabular $(\mU,\mW,\mZ)$, Corollary \ref{cor:discrete} provides a  sufficient condition for Assumptions
\ref{ass:completeness} and \ref{ass:regularity} \citep{shi2020multiply}. 
\begin{corollary}\citep{shi2020multiply}
\label{cor:discrete}
  Suppose that all $\mU$, $\mW$, and $\mZ$ are tabular. If both $Z_t$ and $W_t$ have at least as many categories as $U_t$ for
  $1\leq t\leq T$, i.e., $|\mZ|, |\mW| \geq |\mU|$ (where $|\mX|$ is the cardinality of set $\mX$), and transition probability matrices $P_t(\bW\given \bU,s)\triangleq \left[P_t(w_i\given u_j,s)\right]_{w_i\in\mW,u_j\in \mU}$
  and $P_t(\bU\given\bZ,a,s)\triangleq \left[P_t(u_i\given z_j,a,s)\right]_{u_i\in\mU,z_j\in\mZ}$ %
  are of full rank with rank $|\mU|$ for all $a,s,t$, 
  then Assumptions \ref{ass:completeness} and \ref{ass:regularity} hold.
  
\end{corollary}

\subsection{Assumptions on the uniqueness of bridge functions}

In general, we do not need to impose restrictions on the uniqueness of $V$-bridges 
$\{\vpi_t\}_{t=1}^T$ for policy value identification. To simplify our theoretical analysis on the estimation
error of $V$-bridges, we need the uniqueness of $V$-bridges
$\{\vpi_t\}_{t=1}^T$ and $Q$-bridges $\{\qpi_t\}_{t=1}^T$, which can be ensured
by the following Assumption \ref{ass:completeness2}.

\begin{assumption} \label{ass:completeness2}
  For any square-integrable function $g$ and for any $(s,a)\in\mS\times\mA$,
  $\EE\{g(W_t)\mid Z_t,S_t=s,A_t=a\}=0$ a.s. if and only $g=0$ a.s.
\end{assumption}

\begin{corollary} \label{cor:unique-q}
  Under Assumption \ref{ass:completeness2} and all conditions in Theorem
  \ref{thm:existence}, 
  the $V$-bridges $\left\{ \vpi_t \right\}_{t=1}^T$ 
  that satisfy \eqref{eq:V-bridge} and $Q$-bridges $\{\qpi_t\}_{t=1}^T$ that satisfy \eqref{eq:Q-bridge}
  are both unique. Moreover, they can be
  non-parametrically identified by \eqref{eq:nonpara-identification}. 
\end{corollary}
\begin{proof}
Apparently it suffices to prove the uniqueness of $Q$-bridges $\{\qpi_t\}_{t=1}^T$. If there is another set of $\{\tilde{q}_t^{\pi}\}_{t=1}^T$ that is also a solution to
  \eqref{eq:nonpara-identification}, then
\begin{equation*}
\EE \left\{ \tilde{q}_t^{\pi}(W_t,S_t,A_t) - \qpi_t(W_t,S_t,A_t) \given Z_t,S_t=s,A_t=a \right\} = 0, \quad \text{ a.s.} 
\end{equation*}
By Assumption \ref{ass:completeness2},  $\tilde{q}_t^{\pi}(W_t,s,a) = \qpi_t(W_t,s,a)$ a.s. for all
$(s,a)\in\mS\times\mA$.
\end{proof}

For the tabular case, we have the following corollary for the uniqueness of $V$-bridges and $Q$-bridges.
\begin{corollary}\citep{shi2020multiply}
\label{cor:discrete2}
  Under the conditions in Corollary \ref{cor:discrete}, if $|\mZ| = |\mW| = |\mU|$,
  then Assumptions \ref{ass:completeness}--\ref{ass:completeness2} are satisfied.
\end{corollary}

\section{Additional Results}
\label{sec:additional results}

In this section, we derive finite-sample error bounds for $V$-bridge estimation and OPE when hypothesis spaces $\mH^{(t)}$, $\mG^{(t)}$ and testing space $\mF^{(t)}$ are VC-subgraph classes or RKHSs with exponential eigen-decay. Then we discuss possible choices of proximal variables $W_t$ and $Z_t$.

\subsection{Additional Finite-sample error bounds for $V$-bridge estimation and OPE}
\subsubsection{VC-subgraph class}
\begin{theorem}
  \label{thm:VC}
  Under Assumptions \ref{ass:technical} and \ref{ass: one step transition
  ill-posedness}, and the assumptions in Theorem \ref{thm:one-step} and Corollary
  \ref{cor:VC radii}, with probability at least $1-\zeta$, we have
  \begin{align*}
    \norm{\vpi_1-\hvpi_1}_2  & \lesssim \texttt{ill}_{\max} \times \texttt{trans-ill}\\
    & \qquad\times T^{7/2}\left\{\sqrt{\frac{\max_{1\leq t\leq T}\left\{ \VV(\mF^{(t)}),\VV(\mH^{(t)}),\VV(\mG^{(t+1)}) \right\}}{n}}+\sqrt{\frac{\log(T/\zeta)}{n}}\right\}, \text{and}\\
    |\mV(\pi) - \hat\mV(\pi)| & \lesssim \texttt{ill}_{\max} \times \texttt{trans-ill}\\
    & \qquad\times T^{7/2}\left\{\sqrt{\frac{\max_{1\leq t\leq T}\left\{ \VV(\mF^{(t)}),\VV(\mH^{(t)}),\VV(\mG^{(t+1)}) \right\}}{n}}+\sqrt{\frac{\log(T/\zeta)}{n}}\right\},
  \end{align*}
  where $\texttt{trans-ill}=\max_{1\leq t\leq T}\exp\{a_t\zeta(\alpha_t)\}$ with
  $\zeta(\alpha) = \sum_{t=1}^{\infty} t^{-\alpha}$, and
  $\texttt{ill}_{\max}=\tau_{\pi_1}\max_{1\leq t\leq T}\tau_t\norm{\pi_t/\pi_t^b}_{\infty}^2$.
\end{theorem}

The proof of Theorem  \ref{thm:VC} is given in Appendix \ref{sec: proofs of main results}.

\subsubsection{RKHS with exponential eigen-decay}
\begin{theorem}
  \label{thm:RKHS exp}
  Under Assumptions \ref{ass:technical} and \ref{ass: one step transition
  ill-posedness}, and the assumptions in Theorem \ref{thm:one-step} and Corollary
  \ref{cor:RKHS radii} (2), with probability at least $1-\zeta$, we have
  \begin{align*}
    \norm{\vpi_1-\hvpi_1}_2  & \lesssim \texttt{ill}_{\max} \times \texttt{trans-ill}\times T^{7/2}\left\{\sqrt{\frac{(\log n)^{1/\min\{\beta_{\mH},\beta_{\mG},\beta_{\mF}\}}}{n}} +  \sqrt{\frac{\log(T/\zeta)}{n}}\right\}, \text{and}\\
    |\mV(\pi) - \hat\mV(\pi)| & \lesssim \texttt{ill}_{\max} \times \texttt{trans-ill}\times T^{7/2}\left\{\sqrt{\frac{(\log n)^{1/\min\{\beta_{\mH},\beta_{\mG},\beta_{\mF}\}}}{n}} +  \sqrt{\frac{\log(T/\zeta)}{n}}\right\}.
  \end{align*}
  where $\texttt{trans-ill}=\max_{1\leq t\leq T}\exp\{a_t\zeta(\alpha_t)\}$ with
  $\zeta(\alpha) = \sum_{t=1}^{\infty} t^{-\alpha}$, and
  $\texttt{ill}_{\max}=\tau_{\pi_1}\max_{1\leq t\leq T}\tau_t\norm{\pi_t/\pi_t^b}_{\infty}^2$.
\end{theorem}
The proof of Theorem  \ref{thm:RKHS exp} is given in Appendix \ref{sec: proofs of main results}.

\subsection{Different choices of proxy variables}
Here we first provide several options on how to choose proxy variables $W_t$ and $Z_t$
satisfying basic assumptions \ref{ass: Markovian} --\ref{ass:ActionProxy}. %
Then we discuss their effect on the ill-posedness and one step estimation errors. Finally, we comment on some practical issues.

\paragraph{Choice of $W_{t}$.}
In our confounded POMDP setting, typically we need a reward-inducing
proxy $W_t$ to be separated from the current observations at time $t$ and satisfy
the basic assumptions listed in Appendix \ref{sec: basic assumptions}.
In practice,
$W_t$ can be some environmental variables that are correlated with the outcome $R_t$ but
$A_t$ cannot affect $W_t$ (see Figure \ref{fig:choice of W}). It is worth mentioning that \citet{bennett2021proximal} and
\citet{shi2021minimax} use (part of) the current observed state, i.e., $S_t$ in our
paper, as the reward-inducing
proxy. In their settings, given the current action $A_t$, only the hidden state $U_t$ can affect the next hidden state $U_{t+1}$ (Their $U_t$ is the full state variables in our setting). This requires that the proximal variables $Z_t$ and $W_t$ are able to capture the whole information of their hidden state $U_t$. In this case,  Assumption \ref{ass:completeness} becomes harder to hold.
In our setting, however, we allow part of their $U_t$ to be observable. We denote this part by $S_t$ in our paper. This can alleviate the burden on proximal variables $Z_t$ and $W_t$ to capture the whole information of their hidden $U_t$. Therefore, our completeness assumption \ref{ass:completeness} is relatively weaker.
Moreover, \citet{bennett2021proximal} only consider the evaluation for deterministic
target policies, while in our setting, a separate $W_t$ (other than $S_t$) allows us to
evaluate random target policies.

We list some possible causal relationship among $W_t$, $(U_t,S_t)$ and $R_t$ in Figure \ref{fig:choice of W}. We require the causal relationship between $U_t$ and $W_t$. But the effect of $W_t$ on $R_t$ is optional. In practice, one can use the observed variables that have no direct effect on the action, for example, measurement of action independent disturbance which may not may not affect the current reward.

\begin{figure}[H]
    \centering
    \begin{tikzpicture}
	\begin{pgfonlayer}{nodelayer}
		\node [style=Ellipse Upper Gray] (US) at (-3, 0) {$U_t$ $S_t$};
		\node [style=Circle] (W) at (-3, 6) {$W_t$};
		\node [style=Circle] (A) at (3, 0) {$A_t$};
		\node [style=Circle] (R) at (3, 6) {$R_t$};
	\end{pgfonlayer}
	\begin{pgfonlayer}{edgelayer}
		\draw [style=ArrowTo] (US) to (W);
		\draw [style=ArrowToDash] (W) to (R);
		\draw [style=ArrowTo] (US) to (R);
		\draw [style=ArrowTo] (A) to (R);
		\draw [style=ArrowTo] (US) to (A);
		\draw [style=ArrowTo] (A) to (6,0);
		\draw [style=ArrowTo, bend right=15, looseness=1] (-6,-3) to (US);
		\draw [style=ArrowTo, bend right=40, looseness=0.9] (US) to (5, -2.5);
	\end{pgfonlayer}
\end{tikzpicture}\\
    \caption{Causal relationship about $W_t$. Dashed arrows: optional causal effect. $W_t$ may or may not affect $R_t$.}
    \label{fig:choice of W}
\end{figure}
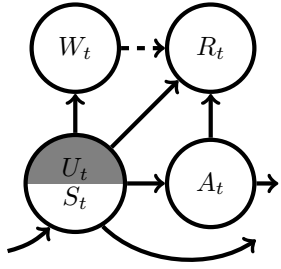

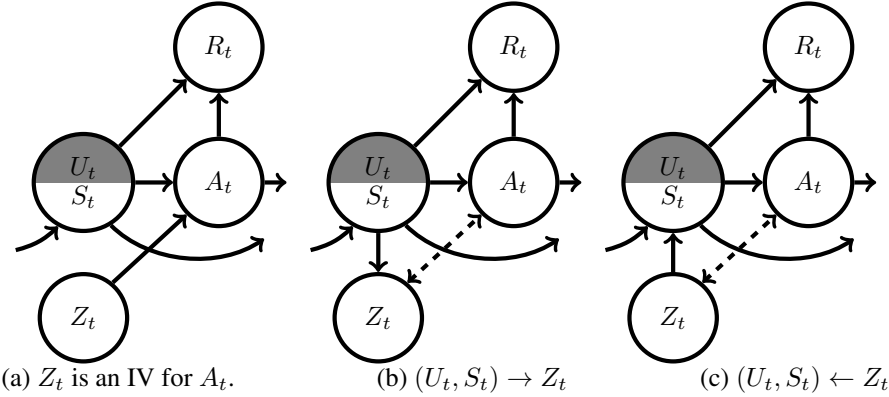
\begin{figure}[H]
    \centering
    \begin{tikzpicture}
	\begin{pgfonlayer}{nodelayer}
		\node [style=Ellipse Upper Gray] (US) at (-3, 0) {$U_t$ $S_t$};
		\node [style=Circle] (Z) at (-3, -6) {$Z_t$};
		\node [style=Circle] (A) at (3, 0) {$A_t$};
		\node [style=Circle] (R) at (3, 6) {$R_t$};

	\end{pgfonlayer}
	\begin{pgfonlayer}{edgelayer}
		\draw [style=ArrowTo] (Z) to (A);
		\draw [style=ArrowTo] (US) to (R);
		\draw [style=ArrowTo] (A) to (R);
		\draw [style=ArrowTo] (US) to (A);
		\draw [style=ArrowTo] (A) to (6,0);
		\draw [style=ArrowTo, bend right=15, looseness=1] (-6,-3) to (US);
		\draw [style=ArrowTo, bend right=40, looseness=0.9] (US) to (5, -2.5);
	\end{pgfonlayer}
\end{tikzpicture}
    \begin{tikzpicture}
	\begin{pgfonlayer}{nodelayer}
		\node [style=Ellipse Upper Gray] (US) at (-3, 0) {$U_t$ $S_t$};
		\node [style=Circle] (Z) at (-3, -6) {$Z_t$};
		\node [style=Circle] (A) at (3, 0) {$A_t$};
		\node [style=Circle] (R) at (3, 6) {$R_t$};

	\end{pgfonlayer}
	\begin{pgfonlayer}{edgelayer}
		\draw [style=ArrowTo] (US) to (Z);
		\draw [style=ArrowTwoDash] (Z) to (A);
		\draw [style=ArrowTo] (US) to (R);
		\draw [style=ArrowTo] (A) to (R);
		\draw [style=ArrowTo] (US) to (A);
		\draw [style=ArrowTo] (A) to (6,0);
		\draw [style=ArrowTo, bend right=15, looseness=1] (-6,-3) to (US);
		\draw [style=ArrowTo, bend right=40, looseness=0.9] (US) to (5, -2.5);
	\end{pgfonlayer}
\end{tikzpicture}
    \begin{tikzpicture}
	\begin{pgfonlayer}{nodelayer}
		\node [style=Ellipse Upper Gray] (US) at (-3, 0) {$U_t$ $S_t$};
		\node [style=Circle] (Z) at (-3, -6) {$Z_t$};
		\node [style=Circle] (A) at (3, 0) {$A_t$};
		\node [style=Circle] (R) at (3, 6) {$R_t$};

	\end{pgfonlayer}
	\begin{pgfonlayer}{edgelayer}
		\draw [style=ArrowTo] (Z) to (US);
		\draw [style=ArrowTwoDash] (Z) to (A);
		\draw [style=ArrowTo] (US) to (R);
		\draw [style=ArrowTo] (A) to (R);
		\draw [style=ArrowTo] (US) to (A);
		\draw [style=ArrowTo] (A) to (6,0);
		\draw [style=ArrowTo, bend right=15, looseness=1] (-6,-3) to (US);
		\draw [style=ArrowTo, bend right=40, looseness=0.9] (US) to (5, -2.5);
	\end{pgfonlayer}
\end{tikzpicture}\\
    (a) $Z_t$ is an IV for $A_t$. ~~~~~~~~~~~~~~~~~~~ (b) $(U_t,S_t)\rightarrow Z_t$ ~~~~~~~~~~~~~~~~~~~(c) $ (U_t,S_t)\leftarrow Z_t$
    \caption{Causal relationship about $Z_t$. Dashed arrows: optional causal effect $Z_t\rightarrow A_t$ or $Z_t\leftarrow A_t$ or no causal effect. (c) is incompatible with Figure \ref{fig:choice of W} (b).}
    \label{fig:choice of Z}
\end{figure}

\begin{figure}[H]
    \centering
    \begin{tikzpicture}
	\begin{pgfonlayer}{nodelayer}
		\node [style=Ellipse Upper Gray] (USp) at (-15, 0) {$U_{t-1}$ $S_{t-1}$};
		\node [style=Circle] (Wp) at (-15, 6) {$W_{t-1}$};
		\node [style=Circle] (Ap) at (-9, 0) {$A_{t-1}$};
		\node [style=Circle] (Rp) at (-9, 6) {$R_{t-1}$};
		\node [style=Ellipse Upper Gray] (US) at (-3, 0) {$U_t$ $S_t$};
		\node [style=Circle] (W) at (-3, 6) {$W_t$};
		\node [style=Circle] (A) at (3, 0) {$A_t$};
		\node [style=Circle] (R) at (3, 6) {$R_t$};
		\node [style=Ellipse Upper Gray] (USn) at (9, 0) {$U_{t+1}$ $S_{t+1}$};
		\node [style=Circle] (Wn) at (9, 6) {$W_{t+1}$};
		\node [style=Circle] (An) at (15, 0) {$A_{t+1}$};
		\node [style=Circle] (Rn) at (15, 6) {$R_{t+1}$};

	\end{pgfonlayer}
	\begin{pgfonlayer}{edgelayer}
		\draw [style=ArrowTo] (USp) to (Wp);
		\draw [style=ArrowToDash] (Wp) to (Rp);
		\draw [style=ArrowTo] (USp) to (Rp);
		\draw [style=ArrowTo] (Ap) to (Rp);
		\draw [style=ArrowTo] (USp) to (Ap);
		\draw [style=ArrowTo] (US) to (W);
		\draw [style=ArrowToDash] (W) to (R);
		\draw [style=ArrowTo] (US) to (R);
		\draw [style=ArrowTo] (A) to (R);
		\draw [style=ArrowTo] (US) to (A);
		\draw [style=ArrowTo] (USn) to (Wn);
		\draw [style=ArrowToDash] (Wn) to (Rn);
		\draw [style=ArrowTo] (USn) to (Rn);
		\draw [style=ArrowTo] (An) to (Rn);
		\draw [style=ArrowTo] (USn) to (An);
		\draw [style=ArrowTo] (Ap) to (US);
		\draw [style=ArrowTo] (A) to (USn);
		\draw [style=ArrowTo, bend right=15, looseness=1] (-18,-3) to (USp);
		\draw [style=ArrowTo, bend right=60, looseness=0.75] (USp) to (US);
		\draw [style=ArrowTo, bend right=60, looseness=0.75] (US) to (USn);
		\draw [style=ArrowTo, bend right=40, looseness=0.9] (USn) to (15, -2.5);

	\end{pgfonlayer}
\end{tikzpicture}
    \caption{An example of $Z_t$ as the observed history.}
    \label{fig:Z history}
\end{figure}
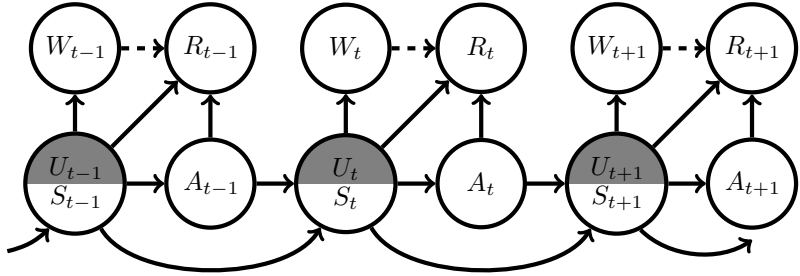

\paragraph{Choice of $Z_{t}$.}
Once we determine $W_t$, there are several proper choices of
$Z_t$ that are compatible with $W_t$ (see Figure \ref{fig:choice of Z}). 
One choice of $Z_t$ is the observed history up to step $t-1$, e.g., $Z_t= (Z_{t-1};S_{t-1},W_{t-1},A_{t-1},R_{t-1})$ with some 
pre-observed history before $(U_1,S_1)$ as $Z_1$.
See Figure \ref{fig:Z history} for a valid example. In this case 
$Z_{t+1}$ contains information of $Z_t$ so that we expect that
$C_{t',t'+1}^{(t)}$ tends to be smaller. However, this can enlarge the one-step
errors $M_{\mH}(T-t+1)^2 \left( \bar\delta_n^{(t)}+c_0
  \sqrt{\frac{\log(c_1T/\zeta)}{n}} \right)$, where the upper bound of critical
radii $\bar\delta_n^{(t)}$ becomes larger because the dimension of testing space
$\mF^{(t)}(\mZ\times\mA\times\mS)$ is now $\bigO(t)$. Fortunately, these one-step errors only contribute to the final error bound for $\calV(\pi)$ linearly.

In practice, to reduce the dimension of $Z_t$, one may use the most recent $k$-step observed
history, or try to learn a low dimensional representation $\phi(Z_t)$ of
$Z_t$ and then replace $\mF^{(t)}(\mZ\times\mA\times\mS)$ by
$\tilde\mF^{(t)}(\phi(\mZ)\times\mA\times\mS)$ in \eqref{eqn: minmax estimation}. Similar ideas have been used in kernel IV regression \citep{singh2020kernel}.

\section{Technical Proofs}
\label{sec:proofs}

In this section, we provide the proofs of identification result in Section \ref{sec:prelim} and the finite sample bounds for $V$-bridges and OPE in Section \ref{sec:theory}.%

\subsection{Proof of Theorem \ref{thm:existence}}
\begin{proof}[\unskip\nopunct]
\textbf{Part I.}~~~  
  We suppose there exists $\qpi_t$ satisfying \eqref{eq:nonpara-identification}, $1\leq
  t\leq T$. Define $\vpi_{T+1}=0$. Then
\begin{align*}
  &\EE \left\{ R_t + \vpi_{t+1}(W_{t+1},S_{t+1}) \given Z_t,S_t,A_t \right\}\\
 = & \EE \left[ \EE \left\{ R_t + \vpi_{t+1}(W_{t+1},S_{t+1})\given U_t,Z_t,S_t,A_t \right\}\given Z_t,S_t,A_t \right]\\
 = & \EE \left[ \EE \left\{ R_t + \vpi_{t+1}(W_{t+1},S_{t+1})\given U_t,S_t,A_t \right\}\given Z_t,S_t,A_t \right] \text{ by Assumption \ref{ass:ActionProxy},}
\end{align*}
and
\begin{align*}
   &\EE \left\{ \qpi_t(W_t,S_t,A_t)\given Z_t,S_t,A_t \right\}\\
 = & \EE \left[ \EE \left\{ \qpi_t(W_t,S_t,A_t)\given U_t, Z_t,S_t,A_t \right\}\given Z_t,S_t,A_t \right]\\
 = & \EE \left[ \EE \left\{ \qpi_t(W_t,S_t,A_t)\given U_t, S_t,A_t \right\}\given Z_t,S_t,A_t \right] \text{ by Assumption \ref{ass:ActionProxy}.}
\end{align*}

Therefore, by  Assumption \ref{ass:completeness} (a), we have
\begin{equation} \label{eq:bellman-like}
\EE \left\{ R_t + \vpi_{t+1}(W_{t+1},S_{t+1})\given U_t,S_t,A_t \right\} = \EE \left\{ \qpi_t(W_t,S_t,A_t)\given U_t,S_t,A_t \right\} \quad \text{ a.s.} 
\end{equation}
We will use this Bellman-like equation \eqref{eq:bellman-like} to verify
\eqref{eq:V-bridge} and \eqref{eq:Q-bridge}.

Next, we prove that such these $\{\qpi_t,\vpi_t\}_{t=1}^T$ obtained by 
Algorithm \ref{alg:identification} can be used
as $Q$-bridges \eqref{eq:Q-bridge} and $V$-bridges \eqref{eq:V-bridge}. 

First, at time $T$,
\begin{align*}
  \EE^{\pi} \left( R_T\given U_T,S_T \right) &= \sum_{a_T\in\mA}\EE \left( R_T\given U_T,S_T,A_T=a_T \right)\pi_T(a_T\given S_T)\\
  =&\sum_{a_T\in\mA}\EE \left\{ \qpi_T(W_T,S_T,a_T) \given U_T,S_T,A_T=a_T\right\}\pi_T(a_T\given S_T) \text{ by \eqref{eq:bellman-like}}\\
  =&\sum_{a_T\in\mA}\EE \left\{ \qpi_T(W_T,S_T,a_T) \given U_T,S_T\right\}\pi_T(a_T\given S_T) \text{ by Assumption \ref{ass:RewardProxy}}\\
  =&\EE \left\{\sum_{a_T\in\mA} \pi(a_T\given S_T) \qpi_T(W_T,S_T,a_T) \Given U_T,S_T\right\}\\
  =&\EE \left\{ \vpi_T(W_T,S_T)\given U_T,S_T\right\} \text{ by definition of $\vpi_T$}.
\end{align*}
By induction, suppose that at time $t+1$, $\EE^{\pi}\left[\sum_{t'=t+1}^T R_{t'}\given
S_{t+1},U_{t+1}\right] = \EE \left\{ \vpi_{t+1}(W_{t+1},S_{t+1})\given S_{t+1},U_{t+1}
\right\}$. Then at time $t$,
\begin{align*}
  &\EE^{\pi} \left( \sum_{t'=t}^T R_{t'} \Given U_t,S_t \right)\\
  =&\EE^{\pi} \left\{ R_t + \EE^{\pi}\left( \sum_{t'=t+1}^T R_{t'}\Given U_{t+1},S_{t+1},U_t,S_t\right) \Given U_t,S_t \right\}\\
  =&\EE^{\pi} \left\{ R_t + \EE^{\pi}\left( \sum_{t'=t+1}^T R_{t'}\Given U_{t+1},S_{t+1}\right) \Given U_t,S_t \right\} \text{ by  Assumption \ref{ass: Markovian}}\\
  =&\EE^{\pi} \left\{ R_t + \EE\left(\vpi_{t+1}(W_{t+1},S_{t+1})\Given U_{t+1},S_{t+1}\right) \Given U_t,S_t \right\}\\
  =&\EE^{\pi} \left\{ R_t + \EE\left(\vpi_{t+1}(W_{t+1},S_{t+1})\Given U_{t+1},S_{t+1},U_t,S_t\right) \Given U_t,S_t \right\}\text{ by  Assumption \ref{ass:RewardProxy}}\\
  =&\EE^{\pi} \left\{ R_t+\vpi_{t+1}(W_{t+1},S_{t+1}) \given U_t,S_t\right\}\text{ by the law of total expectation and Assumption \ref{ass:RewardProxy}}\\
  =&\sum_{a_t\in\mA}\EE \left\{ R_t + \vpi_{t+1}(W_{t+1},S_{t+1}) \given U_t,S_t,A_t=a_t\right\}\pi_t(a_t\given S_t)\\
  =&\sum_{a_t\in\mA}\EE \left\{ \qpi_t(W_t,S_t,a_t) \given U_t,S_t,A_t=a_t\right\}\pi_t(a_t\given S_t) \text{ by \eqref{eq:bellman-like}}\\
  =&\sum_{a_t\in\mA}\EE \left\{ \qpi_t(W_t,S_t,a_t) \given U_t,S_t\right\}\pi_t(a_t\given S_t) \text{ by Assumption \ref{ass:RewardProxy}}\\
  =&\EE \left\{\sum_{a_t\in\mA} \pi(a_t\given S_t) \qpi_t(W_t,S_t,a_t) \Given U_t,S_t\right\}\\
  =&\EE \left\{ \vpi_t(W_t,S_t)\given U_t,S_t\right\} \text{ by definition of $\vpi_t$}.
\end{align*}
Therefore \eqref{eq:V-bridge} hold for all $1\leq
t\leq T$. The validity of $Q$-bridge \eqref{eq:Q-bridge} can be similarly verified by restricting on $A_t=a$, for each $a\in\mA$.

\textbf{Part II.}~~~ Now we prove the existence of the solution to  \eqref{eq:nonpara-identification}.

For $t=T,\dots,1$,
by Assumption \ref{ass:regularity} (a), $K_{s,a;t}$ is a compact operator for
each $(s,a)\in\mS\times\mA$
\citep[Example 2.3]{carrasco2007linear}, so there exists a singular value
system stated in Assumption \ref{ass:regularity} (c) by Lemma \ref{lem:Picard}.
Then by Assumption \ref{ass:completeness} (b), we have ${\rm
Ker}(K_{s,a;t}^{*})={0}$, since for any $g\in{\rm Ker}(K_{s,a;t}^{*})$, we have, 
by the definition of ${\rm Ker}$, $K_{s,a;t}^{*}g = \EE \left[ g(Z_t)\given
  W_t,S_t=s,A_t=a\right]=0$, which implies that $g=0$ a.s. Therefore ${\rm
Ker}(K_{s,a;t}^{*})={0}$ and ${\rm Ker}(K_{s,a;t}^{*})^{\perp} =
\mL^2(\mu_{Z_t\given S_t,A_t}(z\given s,a))$. By Assumption \ref{ass:regularity}
(b), $\EE \left\{ R_t+g(W_{t+1},S_{t+1})\given Z_t=\cdot,S_t=s,A_t=a \right\}\in
{\rm Ker}(K_{s,a,;t}^{*})$ for given $(s,a)\in\mS_t\times\mA$ and any
$g\in\mG^{(t+1)}$. Now we have verified the condition (a) in Lemma \ref{lem:Picard}.
The condition (b) is satisfied given Assumption \ref{ass:regularity} (c).
Recursively applying the above argument from $t=T$ to $t=1$ yields the existence
of the solution to \eqref{eq:nonpara-identification}.

\end{proof}

\subsection{Proof of Theorem \ref{thm:decomposition}}
By definition and Assumptions \ref{ass:completeness}--\ref{ass:completeness2}, $\Ppi_t$, $t=1,\dots,T$, are linear operators, i.e.,
$\Ppi_t(\alpha_1g_1+\alpha_2g_2)=\alpha_1\Ppi_tg_1 + \alpha_2\Ppi_tg_2$, for any $\alpha_1,\alpha_2\in\RR$ and $g_1, g_2 \in \mL^2(\mR\times\mW\times\mS)$.

We first decompose $\hvpi_t-\vpi_t$ into a summation of projections of one-step
error. Then we bound each one-step error by the projected errors times a
product of transition ill-posedness.

\subsubsection{Decomposition of $\mL^2$-error of $\vpi_1$}
Following the identification procedure in Algorithm \ref{alg:DetailedFQE}, we can decompose $\vpi_1$ by
\begin{equation*}
\vpi_1 = \Ppi_1(R_1+\vpi_2) = \Ppi_1(R_1+\Ppi_2(R_2+\vpi_3)) = \dots =\Ppi_1(R_1+\Ppi_2(R_2 +\Ppi_3(\dots+\Ppi_TR_T))).
\end{equation*}
Similarly, according to Section \ref{sec:estimation}, we have the empirical version
\begin{equation*}
  \hvpi_1 = \hPpi_1(R_1+\hvpi_2) = \hPpi_1(R_1+\hPpi_2(R_2+\hvpi_3)) = \dots =\hPpi_1(R_1+\hPpi_2(R_2 +\hPpi_3(\dots+\hPpi_TR_T))).
\end{equation*}
Then for each $t=1,\dots,T$, we can decompose $\hvpi_t-\vpi_t$ as
\begin{align*}
  \hvpi_t-\vpi_t &= \hPpi_t(R_t+\hvpi_{t+1}) - \Ppi_t(R_t+ \vpi_{t+1})\\
                 &= [\hPpi_t(R_t+\hvpi_{t+1}) - \Ppi_t(R_t+ \hvpi_{t+1})] + [\Ppi_t(R_t+\hvpi_{t+1})-\Ppi_t(R_t+ \vpi_{t+1})]\\
                 &\triangleq g_t + [\Ppi_t(R_t+\hvpi_{t+1})-\Ppi_t(R_t+ \vpi_{t+1})]\\
                 &= g_t + \Ppi_t[\hvpi_{t+1} -\vpi_{t+1}],
\numit\label{eq:decompose1}
\end{align*}
where the last equality is due the the linearity of $\Ppi_t$, and
$\vpi_{T+1}=\hvpi_{T+1}\triangleq 0$. Recursively we have
\begin{equation}
  \label{eq:decompose2}
  \hvpi_1 - \vpi_1 = g_1 + \Ppi_1g_2 + \Ppi_{1:2}g_3+\dots +\Ppi_{1:T-1}g_T,
\end{equation}
where $\Ppi_{t':t}\triangleq\Ppi_{t'}\dots\Ppi_t$. If $t<t'$, $\Ppi_{t':t}\triangleq\mI$, the identity operator.

By the definition of the ill-posedness and combining the above decomposition, we
can obtain the discrepancy between $\hvpi_1$ and $\vpi_1$: 
\begin{align*}
  \norm{\vpi_1 - \hvpi_1}_2 &\leq \bar{\tau}_{1} \norm{\EE(\vpi_1-\hvpi_1\given Z_1,S_1)}_2\\
&\leq \bar{\tau}_{1}\sum_{t=1}^T\norm{\EE(\Ppi_{1:t-1}g_t\given Z_1,S_1)}_2 \quad \text{by the triangular inequality},
\end{align*}
where $\bar{\tau}_{_1}=\sup_{g_1\in\mG^{(1)}}
\frac{\norm{g_1}_2}{\norm{\EE[g_1(W_1,S_1)\given Z_1,S_1]}_2}$.
This indicates that we only need to separately bound the $\mL^2$ norm of the projected one-step error defined as
\begin{equation}
  \label{eq: by-step error}
  \norm{\EE[\Ppi_{1:t-1}g_t\given Z_1,S_1]}_2=\norm{\EE[\Ppi_{1:t-1}(\hPpi_t-\Ppi_t)(R_t+\hvpi_{t+1})\given Z_1,S_1]}_2,
\end{equation}
for each $t=1,\dots,T$.

\subsubsection{Error bounds for projected one-step error}
To study the one-step $\mL^2$ projected error of \eqref{eq: by-step error}, for
each $t=1,\dots,T$,
motivated by \eqref{eq: by-step error}, we sequentially define the following functions:
\begin{align*}
  g_t &\triangleq (\hPpi_t-\Ppi_t)(\hvpi_{t+1}+R_t),\\
 g_{t,t-1} &\triangleq \Ppi_{t-1}g_t,\\
 g_{t,t-2} &\triangleq \Ppi_{t-2}g_{t,t-1}= \Ppi_{t-2:t-1}g_t,\\
  &\vdots\\
 g_{t,1} &\triangleq \Ppi_1g_{t,2}=\Ppi_{1:t-1}g_t.
\end{align*}
For each $1\leq t'<t$,
\begin{align*}
  & \quad \norm{\EE[g_{t,t'}(W_{t'},S_{t'})\given Z_{t'},S_{t'}]}_2  = \norm{\EE\{[\Ppi_{t'}g_{t,t'+1}](W_{t'},S_{t'})\given Z_{t'},S_{t'}\}}_2\\
  &= \norm{\EE^{\pi_{t'}}[g_{t,t'+1}(W_{t'+1},S_{t'+1})\given Z_{t'},S_{t'}]}_2\\
  &\leq C_{t'+1,t'}^{(t)}\norm{\EE[g_{t,t'+1}(W_{t'+1},S_{t'+1})\given Z_{t'+1},S_{t'+1}]}_2,
\end{align*}
where the local transition ill-posedness $C_{t'+1,t'}^{(t)}$ will be defined
later in \eqref{eq:local ill}, %
and the second equality is due to
\begin{align*}
    & \quad\EE\{[\Ppi_{t'}g_{t,t'+1}](W_{t'},S_{t'})\given Z_{t'},S_{t'}\}\\
    &= \EE\left\{\sum_{a\in\mA}\pi_{t'}(a\given S_{t'})[\mP_{t'}g_{t,t'+1}](W_{t'},S_{t'},A_{t'}=a)\given Z_{t'},S_{t'}\right\}\\
    &= \sum_{a\in\mA}\pi_{t'}(a\given S_{t'})\EE\left\{[\mP_{t'}g_{t,t'+1}](W_{t'},S_{t'},A_{t'}=a)\given Z_{t'},S_{t'}\right\}\\
    &= \sum_{a\in\mA}\pi_{t'}(a\given S_{t'})\EE\left\{[\mP_{t'}g_{t,t'+1}](W_{t'},S_{t'},A_{t'})\given Z_{t'},S_{t'},A_{t'}=a\right\}\\
    &= \sum_{a\in\mA}\pi_{t'}(a\given S_{t'})\EE\left\{g_{t,t'+1}(W_{t'+1},S_{t'+1})\given Z_{t'},S_{t'},A_{t'}=a\right\} \text{  by Q-bridge}\\
    &= \EE^{\pi_{t'}}\left\{g_{t,t'+1}(W_{t'+1},S_{t'+1})\given Z_{t'},S_{t'}\right\}. 
\end{align*}

Then by induction, we can show that
\begin{align*}
\norm{\EE[\Ppi_{1:t-1}g_t\given Z_1,S_1]}_2 \leq C_{2,1}^{(t)}\dots C_{t,t-1}^{(t)} \norm{\EE[g_t\given Z_t,S_t]}_2
\end{align*}
Therefore,
\begin{align*}
  \norm{\vpi_1 - \hvpi_1}_2 &\leq \bar{\tau}_{1}\sum_{t=1}^T\norm{\EE[\Ppi_{1:t-1}g_t\given Z_1,S_1]}_2\\
                            &\leq \bar{\tau}_{1}\sum_{t=1}^T C_{2,1}^{(t)}\dots C_{t,t-1}^{(t)}\norm{\EE[g_t\given Z_t,S_t]}_2
\end{align*}
Then for each $t=1,\dots,T$, we need to bound 
\begin{align*}
\norm{\EE[(\hPpi_t-\Ppi_t)(R_t+\hvpi_{t+1})\given Z_t,S_t]}_{2}  & \leq \norm{(\hPpi_t - \Ppi_t)(R_t+\hvpi_{t+1})}_{2}\\
& \leq \norm{(\hmP_t - \mP_t)(R_t+\hvpi_{t+1})}_{2}\norm{\pi_t/\pi_t^b}_\infty\\
& \leq \tau_t\norm{\proj_t(\hmP_t - \mP_t)(R_t+\hvpi_{t+1})}_{2}\norm{\pi_t/\pi_t^b}_\infty,
\numit\label{eq:transition bound}
\end{align*}
where $\tau_t$ is the local ill-posedness constant at step $t$, defined in \eqref{eq:taut}.

Finally, we have
\begin{equation*}
  \norm{\vpi_1-\hvpi_1}_2 \leq \bar{\tau}_{1}\sum_{t=1}^T \left\{ \prod_{t'=1}^t C_{t',t'-1}^{(t)} \right\}  \tau_t\norm{\pi_t/\pi_t^b}_{\infty}\norm{\proj_t(\hmP_t-\mP_t)(R_t+\hvpi_{t+1})}_2.
\end{equation*}

\subsection{Proof of Theorem \ref{thm:one-step}}
For $t=T,\dots,1$, we iteratively bound $\norm{\proj_t(\hmP_t - \mP_t)(R_{t}+\hvpi_{t+1})}_2$ by
applying Lemma \ref{lem:CMM}, which depends on the critical radius of the space
that contains $\hvpi_{t+1}$ from the last step $t+1$. Then we give the bound of
$\norm{\hvpi_t}_{\mG^{(t)}}^2$, which will be used to calculate critical radii in next step.

\subsubsection{One-step error bound}
\label{sec: one-step error bound}
\textbf{Start from $t=T$,} $\hvpi_{T+1}=\vpi_{T+1}\triangleq 0$.
By Lemma \ref{lem:dikkala}, we have with probability at least $1-3\zeta$,

\begin{align*}
  \norm{\proj_T(\hmP_T-\mP_T)R_T}_2 & \lesssim\delta_n^{(T)}[1+\norm{\qpi_T}_{\mH^{(T)}}^2]\\
                                    &\leq \delta_n^{(T)}[1+M_{\mH}],
\end{align*}
and
\begin{align*}
\norm{\hqpi_T}_{\mH^{(T)}}^2 = \norm{\hmP_TR_T}_{\mH^{(T)}}^2 \leq \norm{\mP_TR_T}_{\mH^{(T)}}^2 + C = \norm{\qpi_T}_{\mH^{(T)}}^2 + C \leq 2M_{\mH},
\end{align*}
by Assumption \ref{ass:technical} (4) and we let $M_{\mH}\geq C$.

\textbf{Iteratively, at time $1\leq t < T$,} by Lemma \ref{lem:CMM},
we have with probability at least $1-4\zeta$,
\begin{align*}
  \norm{\proj_t(\hmP_t-\mP_t)[R_t+\hvpi_{t+1}]}_2 &\lesssim (T-t+1)\delta_n^{(t)}[1+\norm{\mP_t \left(  \frac{R_t+\hvpi_{t+1}}{T-t+1}\right)}_{\mH^{(t)}}^2]\\
&\leq (T-t+1)\delta_n^{(t)}[1+(T-t+1)M_{\mH}],\\
  &\lesssim M_{\mH}(T-t+1)^2\delta_n^{(t)},
\end{align*}
where the second inequality is due to Assumption \ref{ass:technical} (2), $\norm{\mP_t \left(  \frac{R_t+\hvpi_{t+1}}{T-t+1}\right)}_{\mH^{(t)}}^2 \leq \norm{\frac{\hqpi_{t+1}}{T-t}}_{\mH^{(t+1)}}^2\leq (T-t+1)M_{\mH}$.

Also, 
\begin{align*}
  \norm{\frac{\hqpi_t}{T-t+1}}_{\mH^{(t)}}^2 = \norm{\hmP_t \left( \frac{R_t+\hvpi_{t+1}}{T-t+1} \right)}_{\mH^{(t)}}^2&\leq \norm{\mP_t \left( \frac{R_t+\hvpi_{t+1}}{T-t+1} \right)}_{\mH^{(t)}}^2 + M_{\mH}\\
&\leq (T-t+2)M_{\mH},
\end{align*}
where $\delta_n^{(t)}=\bar\delta_n^{(t)} + c_0\sqrt{\log(c_1/\zeta)/n}$, $c_0, c_1>0$, $\delta_n^{(t)}$ upper bounds the critical radii of
$\mF_{3M}^{(t)}(\mZ_t\times\mS_t\times\mA_t)$, $\bOmega^{(t)}$ and $\bXi^{(t)}$. %

Since $\norm{\frac{\hvpi_{t}}{T-t+1}}_{\mG^{(t)}}^2\leq C_{\mG} \norm{\frac{\hqpi_{t}}{T-t+1}}_{\mH^{(t)}}^2$ by Assumption \ref{ass:technical} (3), we have that $\norm{\frac{\hvpi_{t}}{T-t+1}}_{\mG^{(t)}}^2\leq C_{\mG}\norm{\frac{\hqpi_{t}}{T-t+1}}_{\mH^{(t)}}^2\leq C_{\mG}\norm{\frac{\hqpi_{t+1}}{T-t}}_{\mH^{(t+1)}}^2\leq C_{\mG}(T-t+2)M_{\mH}$. Therefore $\frac{\hvpi_{t}}{T-t+1}\in\mG_{C_{\mG}(T-t+2)M_{\mH}}^{(t)}$.

\subsubsection{Combined Result}
Finally, we replace $\zeta$ by $\zeta/(4T)$ and redefine $\delta_n^{(t)}=\bar\delta_n^{(t)} + c_0\sqrt{\log(c_1 T/\zeta)/n}$  for $t=1,\dots,T$, and consider
the intersection of above events, we have with
probability at least $1-\zeta$,
\begin{equation*}
\norm{\proj_t(\hmP_t-\mP_t)(\hvpi_{t+1}+R_t)}_2 \lesssim M_{\mH}(T-t+1)^2\delta_n^{(t)},
\end{equation*}
uniformly for all $1\leq t\leq T$.

\subsection{Localized ill-posedness $\tau_t$ and one-step transition ill-posedness $C_{t',t'-1}^{(t)}$}
\paragraph{Localized ill-posedness.}
By Theorem \ref{thm:one-step} and \eqref{eq:transition bound}, we have that with
probability at least $1-\zeta$,
\begin{equation*}
\norm{\EE [(\hPpi_t-\Ppi_t)(R_t+\hvpi_{t+1})\given Z_t,S_t]}_2 \lesssim \tau_t(T-t+1)^2 M_{\mH}\delta_n^{(t)}\norm{\pi_t/\pi_t^b}_{\infty},
\end{equation*}
uniformly for all $1\leq t\leq T$, where we define the local ill-posedness \citep{chen2011rate}
\begin{align*}
\tau_t  \triangleq \sup_{h\in\mH^{(t)}}\frac{\norm{h}_2}{\norm{\proj_th}_2}\ 
  \text{subject to }&\  \norm{\proj_th}_2 \lesssim (T-t+1)^2 M_{\mH}\delta_n^{(t)},\\
  & \norm{h}_{\mH^{(t)}}^2 \lesssim (T-t+1)^3 M_{\mH},
\numit\label{eq:taut}
\end{align*}
where the bounds for $\norm{\proj_th}_2$ and $\norm{h}_{\mH^{(t)}}^2$ are adapted from above results in Appendix \ref{sec: one-step error bound}.

We show that under further assumption on the joint distribution of
$(S_t,A_t,W_t,Z_t)$, for RKHS $\mH^{(t)}$ with kernel $K_{\mH^{(t)}}$, the local ill-posedness can be properly controlled.
By Mercer's theorem with some regularity conditions, for any $h\in\mH^{(t)}$, we have
\begin{equation*}
h = \sum_{j=1}^{\infty} a_j e_j,
\end{equation*}
where \{$e_j: \mW\times\mS\times\mA\rightarrow \RR\}$ are the eigenfunctions of kernel $K_{\mH^{(t)}}$
corresponding to nonincreasing eigenvalues $\{\lambda_j\triangleq\lambda_j^{\downarrow}(K_{\mH^{(t)}})\}$. Then we have
$\norm{h}_2^2 = \sum_ja_j^2$ and $\norm{h}_{\mH}^2 = \sum_j a_j^2/\lambda_j$.
\begin{equation*}
\norm{\proj_t h}_2^2 = \sum_{i,j} a_ia_j \EE \left\{ \EE[e_i(W_t,S_t,A_t)\given Z_t,S_t,A_t] \EE[e_j(W_t,S_t,A_t)\given Z_t,S_t,A_t] \right\}.
\end{equation*}
For $m\in\NN_+$, let $I = \left\{ 1,\dots,m \right\}$, $e_I = (e_1,\dots,e_m)$ and
$a_I = (a_1,\dots,a_m)$ and define
\begin{equation*}
\Gamma_m \triangleq \EE \left\{ \EE[e_I(W_t,S_t,A_t)\given Z_t,S_t,A_t] \EE[e_I(W_t,S_t,A_t)\given Z_t,S_t,A_t]^{\top} \right\}.
\end{equation*}
With same argument as \citet{dikkala2020minimax}, we impose the assumption that $\lambda_{\min}(\Gamma_m)\geq \nu_m$ for all $m$ almost surely, which means that the
projected eigenfunctions are not strongly dependent. And we further
assume that for all $i\leq m <j$,
\begin{equation}\label{eq: projection condition}
\left| \EE \left\{ \EE[e_i(W_t,S_t,A_t)\given Z_t,S_t,A_t] \EE[e_j(W_t,S_t,A_t)\given Z_t,S_t,A_t] \right\} \right| \leq c \nu_m,
\end{equation}
for some constant $c>0$. This implies that the projection does not destroy the
orthogonality for the first $m$ eigenfunctions and eigenfunctions with indices
larger than $m$ too much. Then we can bound the local measure of ill-posedness as follow.

\begin{lemma} [\citet{dikkala2020minimax}, Lemma 11]
  \label{lem: ill-posedness bound}
  Suppose that $\lambda_{\min}(\Gamma_m)\geq \nu_m$ and \eqref{eq: projection condition} holds for
  all $i\leq m <j$ and some constant $c>0$. Then
\begin{equation*}
[\tau_t^{*}(\delta,B)]^2 \triangleq \max_{h\in\mH_B^{(t)}:\norm{\proj_t h}_2\leq \delta}\norm{h}_2^2\leq  \min_{m\in\NN_+}\left\{ \delta^2/\nu_m + B\left(2c\sqrt{\sum_{i=1}^{\infty}\lambda_i} \times \sqrt{\sum^{\infty}_{j=m+1}\lambda_j} + \lambda_{m+1}\right) \right\}.
\end{equation*}
The optimal $m_{*}$ is such that $\delta^2/\nu_m \asymp B\left(2c\sqrt{\sum_{i=1}^{\infty}\lambda_i} \sqrt{ \sum_{j=m+1}^{\infty} \lambda_j} + \lambda_{m+1}\right)$.
\end{lemma}

\textbullet{} For a mild ill-posed case, if $\lambda_m \leq m^{-2\alpha_{\mH}}$ for $\alpha_{\mH}>1/2$ and $\nu_m>m^{-2b}$ for
$b>0$, then
$m_{*} \sim  \left[ \delta^2/B \right]^{-\frac{1}{2(\alpha_{\mH}-1/2+b)}}$
and thus
\begin{align*}
  & \quad\norm{(\hmP_t-\mP_t)(\hvpi_{t+1}+R_t)}_2\\
  & \lesssim \tau_t^{*}\left[ (T-t+1)^2 M_{\mH}\delta_n^{(t)}, (T-t+1)^3M_{\mH} \right]\\
  & \lesssim (T-t+1)^{\frac{2(\alpha_{\mH}-1/2)+3b}{(\alpha_{\mH}-1/2)+2b}}[\delta_n^{(t)}]^{\frac{\alpha_{\mH}-1/2}{\alpha_{\mH}-1/2+2b}}.
\end{align*}

\textbullet{} For a severe ill-posed case, if $\lambda_m\leq m^{-2\alpha_{\mH}}$ for $\alpha_{\mH}>1/2$ and
$\nu_m \sim e^{-m^{b}}$ for $b>0$, then
$m_{*}\sim \left[ \log \left( \frac{B}{\delta^2} \right) \right]^{\frac{1}{b}}$, by
the same argument above,
\begin{align*}
   \quad\norm{(\hmP_t-\mP_t)(\hvpi_{t+1}+R_t)}_2
 \lesssim \left[ \log \left( \frac{1}{(T-t+1)[\delta_n^{(t)}]^2} \right) \right]^{\frac{\alpha_{\mH}-1/2}{2b}} (T-t+1)^{3/2}.
\end{align*}

\paragraph{One-step transition ill-posedness.} For each $t$, from $t'=t-1$ to $t'=1$, we can recursively define a sequence of local transition
ill-posedness as the following:
\begin{align*}
  C_{t'+1,t'}^{(t)}&\triangleq\sup_{g\in\mG(W_{t'+1}\times S_{t'+1})} \frac{\norm{\EE^{\pi_{t'}}[g(W_{t'+1},S_{t'+1})\given Z_{t'},S_{t'}]}_2}{\norm{\EE[g(W_{t'+1},S_{t'+1})\given Z_{t'+1},S_{t'+1}]}_2} \\
  \text{subject to } &\norm{\EE[g(W_{t'+1},S_{t'+1})\given Z_{t'+1},S_{t'+1}]}_2\\
                   &\lesssim  \tau_t (T-t+1)^2M_{\mH}\delta_n^{(t)} \norm{\pi_{t'}/\pi_{t'}^b}_{\infty}\prod_{s=t'+1}^{t-1}C_{s+1,s}^{(t)}.
                       \numit\label{eq:local ill}
\end{align*}
Then we have with probability at least $1-\zeta$,
\begin{align*}
  &\norm{\EE [\Ppi_{1:t-1}(\hPpi_t-\Ppi_t)(R_t+\hvpi_{t+1})\given Z_t,S_t]}_2\\ 
  &\leq \left\{ \prod_{t'=1}^tC_{t',t'-1}^{(t)}\right\}\tau_t(T-t+1)^2M_{\mH} \delta_n^{(t)}\tau_t\norm{\pi_t/\pi_t^b}_{\infty},
\end{align*}
uniformly for all $1\leq t \leq T$.

\subsection{Proofs of Theorems \ref{thm:main},
\ref{thm:VC} and \ref{thm:RKHS exp}}
\label{sec: proofs of main results}
\subsubsection{Decomposition of Off-Policy Value Estimation Error}
Our objective is to give an upper bound of
\begin{align*}
  \left| \EE\vpi_1(W_1,S_1) - \EE_n\hvpi_1(W_1,S_1) \right| & \leq \left| \EE\vpi_1 - \EE_n\vpi_1 \right|+ \left| \EE (\vpi_1-\hvpi_1)\right|\\
                                                            &\quad  + \left| \EE_n(\vpi_1-\hvpi_1) - \EE(\vpi_1-\hvpi_1) \right|\\
                                                            &\quad = (I) + (II) + (III),
\end{align*}
For (I), by applying Hoeffding's inequality, we have with probability at least $1-\zeta/T$,
\begin{align*}
 (I) = |\EE \vpi_1 - \EE_n\vpi_1| \lesssim \norm{\vpi_1}_{\infty}\sqrt{\frac{\log(c_1T/\zeta)}{n}} \lesssim T\sqrt{\frac{\log(c_1T/\zeta)}{n}}.
\end{align*}

For (II), obviously  $(II)=|\EE(\vpi_1-\hvpi_1)|\leq\EE|\vpi_1-\hvpi_1|\leq \norm{\vpi_1-\hvpi_1}_2$.

For (III), by applying Theorem 14.20 of \citet{wainwright2019high}, we have with probability at least $1-\zeta$,
\begin{align*} (III) =\left| \EE_n(\vpi_1-\hvpi_1) - \EE(\vpi_1-\hvpi_1) \right|\lesssim \delta_n^{(0)}(\norm{\vpi_1-\hvpi_1}_2+T\delta_n^{(0)}), \end{align*}
where $\delta_n^{(0)} = \bar\delta_n^{(0)} + c_0\sqrt{\frac{\log(c_1T/\zeta)}{n}}$, and $\bar\delta_n^{(0)}$ is the critical radius of $\mG_{C_{\mG}(T+1)M_{\mH}}$.

The $\mL^2$-error $\norm{\vpi_1-\hvpi}_2$ in the upper bounds of (II) and (III)
can be bound by combining Theorems \ref{thm:decomposition} and \ref{thm:one-step}.

\subsubsection{Applying decomposition of OPE error}
\label{sec: applying decompisition of OPE error}
By Assumption \ref{ass: one step transition ill-posedness}, we can define
$\texttt{trans-ill}=\max_{1\leq t\leq T}\exp \left\{ a_t\zeta(\alpha_t)
\right\}$ since $\prod_{t'=1}^tC_{t',t'-1}^{(t)}\leq \exp \left\{
  a_t\zeta(\alpha_t) \right\}$, $1\leq t\leq T$ are bounded by Corollary
\ref{cor: transition ill-posedness bound}. Define
$\texttt{ill}_{\max}=\bar{\tau}_{1}\max_{1\leq t\leq
T}\tau_t\norm{\pi_t/\pi_t^b}_{\infty}$.

By applying Theorems \ref{thm:decomposition} and \ref{thm:one-step}, and crtical radii results in Example \ref{ex:VC} -- \ref{ex:RKHSexponential} in Appendix \ref{sec:critical radii and local Rademacher complexity}, we have the following results:

\paragraph{For Theorem \ref{thm:main}.} With probability at least $1-\zeta$,
\begin{align*}
  \norm{\vpi_1-\hvpi_1}_2  & \lesssim \texttt{ill}_{\max} \times \texttt{trans-ill}\times T^{7/2}\sqrt{\log(c_1T/\zeta)} n^{-\frac{1}{2+\max\{1/\alpha_{\mH},1/\alpha_{\mG},1/\alpha_{\mF}\}}}\log(n),
\end{align*}
by Corollary \ref{cor:RKHS radii} (1). Then by above decomposition, with probability at least $1-\zeta$,
\begin{align*}
  |\mV(\pi) - \hat\mV(\pi)| & \lesssim \texttt{ill}_{\max} \times \texttt{trans-ill}\times T^{7/2}\sqrt{\log(c_1T/\zeta)}n^{-\frac{1}{2+\max\{1/\alpha_{\mH},1/\alpha_{\mG},1/\alpha_{\mF}\}}}\log(n).
\end{align*}

\paragraph{For Theorem \ref{thm:VC}.} With probability at least $1-\zeta$, with probability at least $1-\zeta$,
  \begin{align*}
    \norm{\vpi_1-\hvpi_1}_2  & \lesssim \texttt{ill}_{\max} \times \texttt{trans-ill}\\
    & \qquad\times T^{7/2}\left\{\sqrt{\frac{\max_{1\leq t\leq T}\left\{ \VV(\mF^{(t)}),\VV(\mH^{(t)}),\VV(\mG^{(t+1)}) \right\}}{n}}+\sqrt{\frac{\log(T/\zeta)}{n}}\right\},
  \end{align*}
by Corollary \ref{cor:VC radii}. Then by above decomposition, with probability at least $1-\zeta$,
  \begin{align*}
    |\mV(\pi) - \hat\mV(\pi)| & \lesssim \texttt{ill}_{\max} \times \texttt{trans-ill}\\
    & \qquad\times T^{7/2}\left\{\sqrt{\frac{\max_{1\leq t\leq T}\left\{ \VV(\mF^{(t)}),\VV(\mH^{(t)}),\VV(\mG^{(t+1)}) \right\}}{n}}+\sqrt{\frac{\log(T/\zeta)}{n}}\right\}.
  \end{align*}

\paragraph{For Theorem \ref{thm:RKHS exp}.} With probability at least $1-\zeta$,
\begin{align*}
  \norm{\vpi_1-\hvpi_1}_2  & \lesssim \texttt{ill}_{\max} \times \texttt{trans-ill}\times T^{7/2}\left\{\sqrt{\frac{(\log n)^{1/\min\{\beta_{\mH},\beta_{\mG},\beta_{\mF}\}}}{n}} +  \sqrt{\frac{\log(T/\zeta)}{n}}\right\},
\end{align*}
by Corollary \ref{cor:RKHS radii} (1). Then by above decomposition,
\begin{align*}
  |\mV(\pi) - \hat\mV(\pi)| & \lesssim \texttt{ill}_{\max} \times \texttt{trans-ill}\times T^{7/2}\left\{\sqrt{\frac{(\log n)^{1/\min\{\beta_{\mH},\beta_{\mG},\beta_{\mF}\}}}{n}} +  \sqrt{\frac{\log(T/\zeta)}{n}}\right\}.
\end{align*}

\paragraph{For Corollary \label{cor: ill-posedness main} under mild and severe
  ill-posed cases.}
Under assumptions in main Theorem \ref{thm:main}, by directly applying Lemma
\ref{lem: ill-posedness bound}, we have that
\begin{align*}
  \norm{\vpi_1-\hvpi_1}_2  & \lesssim \bar{\tau}_{1} \max_{1\leq t\leq T}\norm{\pi_t/\pi_t^b}_{\infty}\times \texttt{trans-ill}\times \eta(n,T,\zeta,\alpha_{\mH},\alpha_{\mF},\alpha_{\mG},b),\\
  |\mV(\pi) - \hat\mV(\pi)| & \lesssim \bar{\tau}_{1} \max_{1\leq t\leq T}\norm{\pi_t/\pi_t^b}_{\infty}\times \texttt{trans-ill}\times \eta(n,T,\zeta,\alpha_{\mH},\alpha_{\mF},\alpha_{\mG},b),
\end{align*}
where, for mild ill-posed case that $\nu_m \sim m^{-2b}$ for $b>0$:
\begin{align*}
\eta(n, T, \zeta, \alpha_\mathcal{H}, \alpha_\mathcal{F}, \alpha_\mathcal{G}, b) =  T^{\frac{7(\alpha_{\mH}-1/2)+10b}{2(\alpha_{\mH}-1/2)+4b}} \left(\sqrt{\log(c_1T/\zeta)}  n^{-\frac{1}{2+\max\{1/\alpha_{\mH},1/\alpha_{\mG},1/\alpha_{\mF}\}}}\log(n)\right)^{\frac{(\alpha_{\mH}-1/2)}{(\alpha_{\mH}-1/2)+2b}}.
\end{align*}
for severe ill-posed case  that $\nu_m \sim e^{-m^b}$ for $b>0$:
\begin{align*}
  \eta(n,T,\zeta,\alpha_{\mH},\alpha_{\mF},\alpha_{\mG},b) &= \sum_{t=1}^T(T-t+1)^{3/2}\left\{ \log \frac{n^{\frac{2}{2+\max\{1/\alpha_{\mH},1/\alpha_{\mG},1/\alpha_{\mF}\}}}}{(\log n)^2(T-t+1)^2\log(T/\zeta)^2} \right\}^{-\frac{\alpha_{\mH}-1/2}{2b}}.\\
\end{align*}

\section{Auxiliary Lemmas}
\label{sec:lemmas}

In this section, we provide some auxiliary lemmas which are needed to prove Theorem \ref{thm:existence} -- \ref{thm:main} and their proofs. %

\subsection{Lemmas For Identification}
\begin{lemma}[Picard's Theorem, Theorem 15.16 of \cite{kress1989linear}]
  \label{lem:Picard}
  Given Hilbert spaces $\mH_1$ and $\mH_2$, a compact operator
  $K:\mH_1\rightarrow\mH_2$ and its adjoint operator
  $K^{*}:\mH_2\rightarrow\mH_1$, there exists a singular system $(\lambda_{\nu},
  \phi_{\nu},\psi_{\nu})_{\nu=1}^{\infty}$ of $K$, with singular values
  $\{\lambda_{\nu}\}$ and orthogonal sequences $\{\phi_{\nu}\}\subset\mH_1$ and
  $\{\psi_{\nu}\}\subset\mH_2$ such that $K\phi_{\nu}=\lambda_{\nu}\psi_{\nu}$ and
  $K^{*}\psi_{\nu}=\lambda_{\nu}\phi_{\nu}$.

  Given $g\in\mH_2$, the Fredholm integral equation of the first kind $Kh=g$ is
  solvable if and only if
\begin{enumerate}[]
\item [(a)] $g\in {\rm Ker}(K^{*})^{\perp}$ and
\item [(b)] $\sum_{\nu=1}^{\infty} \lambda_{\nu}^{-2}|\left\langle g,\psi_{\nu}
  \right\rangle|^2 <\infty$,
\end{enumerate}
where ${\rm Ker}(K^{*}) = \left\{ h:K^{*}h=0 \right\}$ is the null space of
$K^{*}$, and $^{\perp}$ denotes the orthogonal complement to a set.
\end{lemma}

\subsection{One-step estimation error}
\label{sec: one-step estimation error}

Consider the problem of estimating a function $h$ that satisfying the
conditional moment restriction
\begin{equation} \label{eq:CMM}
\EE \left\{ g(W) - h(X) \given Z \right\} = 0,
\end{equation}
where $Z\in \mZ$, $X \in \mX$, $W\in\mW$, $h\in \mH \subset \{h\in\RR^{\mX}:\norm{h}_{\infty}\leq 1\}$, $g\in\mG\subset\{g\in\RR^{\mW}:\norm{g}_{\infty}\leq 1\}$.
Suppose that $h_g^{*}\in\mH$ is the true $h$ that satisfies the conditional moment restriction \eqref{eq:CMM}. 

Suppose that we observe an i.i.d. sample $\{(W_i,X_i,Z_i)\}_{i=1}^n$ of sample size $n$ drawn
from an unknown distribution.
Consider the minimax estimator
\begin{equation}
\label{eq:minimax}
\hhat_g = \argmin_{h\in\mH} \sup_{f\in\mF} \Psi_n (h,f,g) - \lambda \left( \nmF{f}^2 + \frac{U}{\delta^2}\nmEmp{f}^2 \right) + \lambda\mu \nmH{h}^2,
\end{equation}
where $\Psi_n (h,f,g) = n^{-1}\sum_{i=1}^n \{g(W_i)-h(X_i)\} f(Z_i)$ with
the population version $\Psi (h,f,g) = \EE \{g(W)-h(X)\} f(Z)$ and $\lambda,
\delta, \mu, U>0$ are tuning parameters. 

\begin{lemma}[$\mL^2$-error rate for minimax estimator]
  \label{lem:CMM}
  Let $\mF\subset \{f\in\RR^{\mZ}:\norm{f}_{\infty}\leq 1\}$ be a symmetric and star-convex set of test functions.
  Define $\delta=\delta_n + c_0\sqrt{\frac{\log (c_1/\zeta)}{n}}$ for some
  univeral constants $c_0,c_1>0$ and $\delta_n$ the upper bound of critical radii of $\mF_{3U}$,
\begin{equation*}
\bOmega = \left\{ (x,w,z) \mapsto r(h_g^{*}(x)-g(w))f(z): g\in\mG, f\in\mF_{3U}, r\in[0,1] \right\},
\end{equation*}
  and
  \begin{equation*}
    \bXi = \left\{ (x,z)\mapsto r[h-h_g^{*}](x)f_{\Delta}^{L^2 B}(z); h\in\mH, (h-h_g^{*})\in\mH_B, g\in\mG, r\in[0,1]\right\},
  \end{equation*}
  where $f_{\Delta}^{L^2 B} = \argmin_{f\in\mF_{L^2 B}}\norm{f-\proj_Z(h-h_g^{*})}_2$.
  Moreover, suppose that $\forall h\in \mH$, $g\in\mG$, $\norm{f_{\Delta} - \proj_Z(h-h_g^{*})}_2\leq \eta_n\lesssim\delta_n$,
  where $f_{\Delta}\in \arginf_{f\in
  \mF_{L^2\nmH{h-h_g^{*}}^2}} \norm{f-\proj_Z(h-h_g^{*})}_2$.
  If the tuning parameters satisfy $324 C_{\lambda}\delta^2/U \leq \lambda \leq 324
  C_{\lambda}^{\prime}\delta^2/U$ and $\mu\geq \frac{4}{3}L^2 + \frac{18(C_f+1)}{B}\frac{\delta^2}{\lambda}$, then 
  with probability $1-4\zeta$,
\begin{equation*}
\sup_{g\in\mG}\norm{\proj_Z(\hhat_g-h_g^{*})}_2 \lesssim (1+\sup_{g\in\mG}\nmH{h_g^{*}}^2)\delta,
\end{equation*}
and for all $g\in\mG$ uniformly,
\begin{equation*}
  \nmH{\hhat_g}^2 \leq C + \nmH{h_g^{*}}^2.
\end{equation*}
\end{lemma}

The proof of Lemma  \ref{lem:CMM} is given in Appendix \ref{proof: lemma CMM}.

\begin{lemma}[\citet{dikkala2020minimax}, Theorem 1]
  \label{lem:dikkala}
  Consider the problem of estimating a function $h$ that satisfies
  \begin{equation*}
    \EE \left\{ Y - h(X) \given Z \right\} = 0,
  \end{equation*}
  where $Z\in \mZ$, $X \in \mX$, $W\in\mW$, $h\in \mH \subset \{h\in\RR^{\mX}:\norm{h}_{\infty}\leq 1\}$, $|Y|\leq 1$.
  Suppose that there exists $h^{*}\in\mH$ that satisfies the conditional moment equation.
  Suppose that we observed an i.i.d. sample $\{(Y_i,X_i,Z_i)\}_{i=1}^n$ of sample size $n$ drawn
  from an unknown distribution.
  Consider the minimax estimator
  \begin{equation}
    \label{eq:dikkala}
    \hhat = \argmin_{h\in\mH} \sup_{f\in\mF} \Phi_n (h,f) - \lambda \left( \nmF{f}^2 + \frac{U}{\delta^2}\nmEmp{f}^2 \right) + \lambda\mu \nmH{h}^2,
  \end{equation}
  where $\Phi_n (h,f) = n^{-1}\sum_{i=1}^n \{Y_i-h(X_i)\} f(Z_i)$ with
  the population version $\Phi (h,f) = \EE \{Y-h(X)\} f(Z)$ and $\lambda,
  \delta, \mu, U>0$ are tuning parameters. 
  
  Let $\mF\subset \{f\in\RR^{\mZ}:\norm{f}_{\infty}\leq 1\}$ be a symmetric and star-convex set of test functions.
  Define $\delta=\delta_n + c_0\sqrt{\frac{\log (c_1/\zeta)}{n}}$ for some
  univeral constants $c_0,c_1>0$ and $\delta_n$ the upper bound of critical radii of $\mF_{3U}$ and
  \begin{equation*}
    \bar\bXi = \left\{ (x,z)\mapsto r[h-h^{*}](x)f_{\Delta}^{L^2 B}(z); h\in\mH, (h-h^{*})\in\mH_B, r\in[0,1]\right\},
  \end{equation*}
  where $f_{\Delta}^{L^2 B} = \argmin_{f\in\mF_{L^2 B}}\norm{f-\proj_Z(h-h^{*})}_2$.
  Moreover, suppose that $\forall h\in \mH$, $\norm{f_{\Delta} - \proj_Z(h-h^{*})}_2\leq \eta_n\lesssim\delta_n$,
  where $f_{\Delta}\in \arginf_{f\in
  \mF_{L^2\nmH{h-h^{*}}^2}} \norm{f-\proj_Z(h-h^{*})}_2$.
  Suppose tuning parameters satistying $324 C_{\lambda}\delta^2/U \leq \lambda \leq 324
  C_{\lambda}^{\prime}\delta^2/U$ and $\mu\geq \frac{4}{3}L^2 + \frac{18(C_f+1)}{B}\frac{\delta^2}{\lambda}$.
  Then with probability $1-3\zeta$,
\begin{equation*}
\norm{\proj_Z(\hhat-h^{*})}_2 \lesssim (1+\nmH{h^{*}}^2)\delta,
\end{equation*}
and 
\begin{equation*}
  \nmH{\hhat}^2 \leq C + \nmH{h^{*}}^2.
\end{equation*}
\end{lemma} 

\subsection{Critical radii and local Rademacher complexity}
\label{sec:critical radii and local Rademacher complexity}
In this section we list several ways to bound the critical radii of $\mF_{3U}$, $\bOmega$ and $\bXi$ for Lemmas in Appendix \ref{sec: one-step estimation error}. We restrict $\mG = \mG_D = \{g\in\mG: \nmG{g}^2\leq D\}$ for some $D>0$ in this section.

\subsubsection{Local Rademacher complexity bound by entropy integral}
In this subsection, we introduce an entropy integral based approach to bound the local Rademacher complexity and critical radii. Similar to local Rademacher complexity, for a star-shaped and $b$-uniformly bounded function class $\mF$, the \emph{local empirical Rademacher complexity}, a data-dependent quantity, is defined by
\begin{equation*}
   \widehat{R}_n(\delta;\mF) \triangleq \EE \left[\sup_{f\in\mF, \norm{f}_n\leq \delta} |\frac{1}{n}\epsilon_i f(X_i)|\Given \{X_i\}_{i=1}^n\right]
\end{equation*}
where $\{\epsilon_i\}_{i=1}^n$ are i.i.d. Rademacher variables. The \emph{empirical critical radius} $\hat\delta_n$ is the smallest positive solution to 
\begin{equation}
\label{eq:empirical critical inequality}
    \widehat{R}_n(\delta) \leq \frac{\delta^2}{b}.
\end{equation}

\citet[][Proposition 14.25]{wainwright2019high} gives the relationship that with probability at least $1-\zeta$,
\begin{equation*}
    \delta_n \leq \bigO(\hat\delta_n + \sqrt{\frac{\log(1/\zeta)}{n}}).
\end{equation*}
Therefore, we can study the critical radius $\delta_n$ by empirical critical radius $\hat\delta_n$.

Given a space $\mG$, an \emph{empirical $\epsilon$-covering} of $\mG$ is defined as any
function class $\mG^{\epsilon}$  such that for all $g\in\mG$,
$\inf_{g_{\epsilon}\in\mG^{\epsilon}} \norm{g_{\epsilon}-g}_n\leq \epsilon$.
Denote the smallest empirical $\epsilon$-covering of
$\mG$ by $N_n(\epsilon,\mG)$. Let $\BB_n(\delta;\mG)\triangleq \left\{ g\in\mG: \norm{g}_n\leq \delta
\right\}$. Then we have the following Lemma to bound the empirical critical radius by Dudley's entropy integral.

\begin{lemma}\citep[Corollary 14.3]{wainwright2019high}
  \label{lem:Dudley integral}
  The empirical critical inequality \eqref{eq:empirical critical inequality} is satisfied for any $\delta>0$ such that
  
\begin{equation*}
  \frac{64}{\sqrt{n}}\int_{\frac{\delta^2}{2b}}^{\delta} \sqrt{\log N_n(t, \BB_n(\delta,\mG))} {\mathrm d}t \leq \frac{\delta^2}{b}.
\end{equation*}

\end{lemma}

\begin{lemma} 
  \label{lem:radius by entropy}
  Suppose that $\nmH{h_g^{*}}^2 \leq A \nmG{g}^2$ for all $g\in\mG$, so that $\nmH{h_g^*}^2\leq AD$. Let
  $\hdelta_n>0$ satisfy the inequality
\begin{align*}
  \frac{64}{\sqrt{n}} \int_{\frac{\delta^2}{2}}^{4\delta} \sqrt{\log N_n(t,\star{\mF_{3U\vee L^2 B}}) + \log N_n(t,\star{\mH_{AD \vee B}})+ \log N_n(t,\star{\mG_D})}{\rm d}t \leq \delta^2.
\end{align*}
Then with probability $1-\zeta$, we have $\delta_n\leq \bigO(\hdelta_n + \sqrt{\frac{\log
(1/\zeta)}{n}})$, where $\delta_n$ is the maximum critical radii of $\mF_{3U}$,
$\bOmega$ and $\bXi$, with
\begin{equation*}
  \bOmega = \left\{ (x,w,z) \mapsto r(h_g^{*}(x)-g(w))f(z): g\in\mG_D, f\in\mF_{3U}, r\in[0,1] \right\}.
\end{equation*}
\end{lemma}

The proof of Lemma \ref{lem:radius by entropy} is given in Appendix \ref{proof: lemma radius by entropy}.

\begin{ex}[Critical radii for VC subspaces]
\label{ex:VC}
  If star shaped $\mF,\mH$ and $\mG$ are VC subspaces with VC dimensions $\VV(\mF)$, $\VV(\mH)$ and $\VV(\mG)$, respectively, 
  then $\log N_n(t,\mF) + \log N_n(t,\mH) + \log N_n(t,\mG)\lesssim [\VV(\mF) + \VV(\mH) + \VV(\mG)] \log(1/t) \lesssim \max\{\VV(\mF),\VV(\mH),\VV(\mG)\}\log(1/t)$. 
  By Lemma \ref{lem:Dudley integral} and Lemma \ref{lem:radius by entropy}, 
  we have with probability at least $1-\zeta$, $\delta_n \lesssim \sqrt{\frac{\max\{\VV(\mF),\VV(\mH),\VV(\mG)\}}{n}} + \sqrt{\frac{\log(1/\zeta)}{n}}$, where the $\delta_n$ is defined in Lemma \ref{lem:radius by entropy}.
\end{ex}

\subsubsection{Local Rademacher complexity bound for RKHSs}
\begin{lemma}[Critical radii for RKHSs, Corollary 14.5 of  \citet{wainwright2019high}]
  \label{lem:criticalradRKHS}
  Let $\mF_B = \left\{ f\in\mF \given \norm{f}_{\mF}^2\leq B \right\}$ be the
  $B$-ball of a RKHS $\mF$. Suppose that $K_{\mF}$ is the reproducing kernel of $\mF$ with eigenvalues $\{\lambda_j^{\downarrow}(K_{\mF})\}_{j=1}^{\infty}$
  sorted in a  decreasing order. Then the localized population Rademacher complexity is
  upper bounded by
\begin{equation*}
\mR_n(\mF_B,\delta) \leq \sqrt{\frac{2B}{n}}\sqrt{\sum_{j=1}^{\infty}\min \left\{ \lambda_j^{\downarrow}(K_{\mF}),\delta^2 \right\}}.
\end{equation*}
\end{lemma}

\begin{lemma}[Critical radii for $\bOmega$ and $\bXi$ when $\mH$,
  $\mF$, $\mG$ are RKHSs]
  \label{lem:radius by spectra}
  Suppose that $\mF$,$\mH$, and $\mG$ are RKHSs endowed with reproducing kernels
  $K_{\mF}$, $K_{\mH}$, and $K_{\mG}$ with decreasingly sorted eigenvalues
  $\left\{\lambda_j^{\downarrow}(K_{\mF}) \right\}_{j=1}^{\infty}$, 
  $\left\{\lambda_j^{\downarrow}(K_{\mH}) \right\}_{j=1}^{\infty}$, and 
  $\left\{\lambda_j^{\downarrow}(K_{\mG}) \right\}_{j=1}^{\infty}$,
  respectively. Then 
  \begin{align*}
    \mR_n(\bXi,\delta) \leq LB\sqrt{\frac{2}{n}}\sqrt{\sum_{i,j=1}^{\infty}\min\left\{\lambda_i^{\downarrow}(K_{\mH})\lambda_j^{\downarrow}(K_{\mF}),\delta^2\right\}},\quad \text{and}
  \end{align*}
  \begin{equation*}
    \mR_n(\bOmega,\delta)\leq \sqrt{D}(1+\sqrt{A})\sqrt{\frac{12U}{n}}\sqrt{\sum_{i,j=1}^{\infty}\min \left\{ [(\lambda_{i}^{\downarrow}(K_{\mH})+\lambda_{i}^{\downarrow}(K_{\mG})]\lambda_j^{\downarrow}(K_{\mF}),\delta^2 \right\}}.
  \end{equation*}
\end{lemma}
The proof of Lemma \ref{lem:radius by spectra} is given in Appendix \ref{proof: lemma radius by spectra}.

We give the following two examples as directly applications of Lemma \ref{lem:criticalradRKHS} and \ref{lem:radius by spectra}.
\begin{ex}[Critical radii for RKHSs endowed with kernels with polynomial decay]
\label{ex:RKHSpolynomial}
  With the same conditions in Lemma \ref{lem:radius by spectra}, when
  $\lambda_j^{\downarrow}(K_{\mF})\leq cj^{-2\alpha_{\mF}}$,
  $\lambda_j^{\downarrow}(K_{\mG})\leq cj^{-2\alpha_{\mG}}$,
  $\lambda_j^{\downarrow}(K_{\mH})\leq cj^{-2\alpha_{\mH}}$, where constant
  $\alpha_{\mH},\alpha_{\mG},\alpha_{\mF}>1/2, c>0$, then by \citet{krieg2018tensor} we have the upper bound
  of critical radii of $\mF_{3U}$, $\bOmega$ and $\bXi$ satisfies
  \begin{equation*}
      \delta_n \lesssim \max\{\sqrt{B},LB,\sqrt{6DU}(1+\sqrt{A})\} n^{-\frac{1}{2+\max\{1/\alpha_{\mF},1/\alpha_{\mG},1/\alpha_{\mH}\}}}\log(n).
  \end{equation*}

\end{ex}

\begin{ex}[Critical radii for RKHSs endowed with kernels with exponential decay]
\label{ex:RKHSexponential}
  With the same conditions in Lemma \ref{lem:radius by spectra}, when
    $\lambda_j^{\downarrow}(K_{\mH})\leq a_1 e^{-a_2 j^{\beta_{\mH}}}$,
    $\lambda_j^{\downarrow}(K_{\mG})\leq a_1 e^{-a_2 j^{\beta_{\mG}}}$ and
    $\lambda_j^{\downarrow}(K_{\mF})\leq a_1 e^{-a_2 j^{\beta_{\mF}}}$, for
    constants $a_1,a_2,\beta_{\mH},\beta_{\mG},\beta_{\mF}>0$, then we have the upper bound 
  of critical radii of $\mF_{3U}$, $\bOmega$ and $\bXi$ satisfies
  \begin{equation*}
      \delta_n \lesssim \max\{\sqrt{B},LB,\sqrt{6DU}(1+\sqrt{A})\} 
      \sqrt{\frac{(\log n)^{1/\min\{\beta_{\mF},\beta_{\mG},\beta_{\mH}\}}}{n}}.
  \end{equation*}
\end{ex}

\subsection{Proof of Lemmas}
\subsubsection{Proof of Lemma \ref{lem: ill-posedness bound}}
\label{proof: lem: ill-posedness bound}
\begin{proof}
For any $m\in\NN_+$,
\begin{align*}
  \norm{\proj_t h}_2^2 & = a_I^{\top}\Gamma_m a_I + 2\sum_{i\leq m < j} a_ia_j\EE \left\{ \EE[e_i(W_t,S_t,A_t)\given Z_t,S_t,A_t] \EE[e_j(W_t,S_t,A_t)\given Z_t,S_t,A_t] \right\}\\
                       &\quad + \EE \left( \sum_{j>m}a_j\EE[e_j(W_t,S_t,A_t)\given Z_t,S_t,A_t] \right)\\
                       & \geq a_I^{\top}\Gamma_m a_I - 2\sum_{i\leq m < j} |a_ia_j|\EE \left\{ \EE[e_i(W_t,S_t,A_t)\given Z_t,S_t,A_t] \EE[e_j(W_t,S_t,A_t)\given Z_t,S_t,A_t] \right\}\\
                       & \geq a_I^{\top}\Gamma_m a_I - 2\sum_{i\leq m < j} |a_ia_j|c\nu_m\\
                       & \geq \nu_m \norm{a_I}_2^2 - 2c\nu_m\sum_{i\leq m}|a_i|\sum_{j>m} |a_j|\\
                       & \geq \nu_m \norm{a_I}_2^2 - 2c\nu_m \sqrt{\sum_{i\leq m}\lambda_i} \sqrt{\sum_{i\leq m}\frac{|a_i|^2}{\lambda_i}} \sqrt{ \sum_{j>m} \lambda_j}\sqrt{\sum_{j>m}\frac{|a_j|^2}{\lambda_j}} \\
                       & \geq \nu_m \norm{a_I}_2^2 - 2c\nu_mB \sqrt{\sum_{i=1}^{\infty}\lambda_i} \sqrt{ \sum_{j>m} \lambda_j}, \text{ ~~~
                         since }\sum_{j=1}^{\infty} \frac{|a_j|^2}{\lambda_j} \leq B.
\end{align*}
Therefore,
$\norm{h}_2^2 \leq \norm{a_I}^2 + B \lambda_{m+1} \leq \norm{\proj_t h}_2^2/\nu_m + 2cB\sqrt{\sum_{i=1}^{\infty}\lambda_i} \sqrt{ \sum_{j>m} \lambda_j} + B \lambda_{m+1}$.
Because $\norm{\proj_t h}_2\leq \delta$, by taking minimum over $m\in\NN_+$, we have that
\begin{equation*}
[\tau^{*}(\delta,B)]^2 \leq \min_{m\in\NN_+}\left\{ \delta^2/\nu_m + B\left(2c\sqrt{\sum_{i=1}^{\infty}\lambda_i} \sqrt{ \sum_{j>m} \lambda_j} + \lambda_{m+1}\right) \right\}.
\end{equation*}
\end{proof}

\subsubsection{Proof of Lemma \ref{lem:CMM}}
\label{proof: lemma CMM}
\begin{proof}
  Let
  $\mH_B = \left\{ h\in\mH: \nmH{h}^2 \leq B \right\}$ and $\mF_U = \left\{ f\in\mF: \nmF{f}^2 \leq U \right\}$.
  Moreover, let
\begin{align*}
  \Psi^{\lambda} (f,g,h) &= \Psi(h,f,g) - \lambda \left( \frac{2}{3} \nmF{f}^2 + \frac{U}{2\delta^2} \norm{f}_2^2\right), \quad \text{and}\\
\Psi_n^{\lambda} (f,g,h) &= \Psi_n(h,f,g) - \lambda \left( \nmF{f}^2 + \frac{U}{\delta^2} \nmEmp{f}^2\right).
\end{align*}
We first study the relationship between the empirical penalty $\lambda \left( \nmF{f}^2 +
  \frac{U}{\delta^2} \nmEmp{f}^2\right)$ and population penalty $\lambda \left(
  \frac{2}{3} \nmF{f}^2 + \frac{U}{2\delta^2} \norm{f}_2^2\right)$. 
  Let $\delta = \delta_n + c_0\sqrt{\frac{\log (c_1/\zeta)}{n}}$, where $\delta_n$ upper bounds the critical radius of
$\mF_{3U}$ and $c_0,c_1$ are universal constants, by Theorem 14.1 of \citet{wainwright2019high}, with probablity $1-\zeta$,
uniformly for any $f\in\mF$, we have 
\begin{equation}
\label{eq:Thm14.1}
\left| \nmEmp{f}^2 - \norm{f}_2^2 \right| \leq \frac{1}{2} \norm{f}_2^2 + \delta^2\max\left\{1,\frac{\nmF{f}^2}{3U}\right\},\text{and  thus}
\end{equation}
\begin{align*}
  \nmF{f}^2 + \frac{U}{\delta^2}\nmEmp{f}^2 & \geq \nmF{f}^2 + \frac{U}{\delta^2}\left[ \frac{1}{2}\norm{f}_2^2 -\delta^2 \max \left\{ 1, \frac{\nmF{f}^2}{3U} \right\} \right]\\
  &\geq \nmF{f}^2 + \frac{U}{2\delta^2}\norm{f}_2^2 -\max \left\{ U,\frac{1}{3}\nmF{f}^2 \right\}\\
  &\geq \frac{2}{3}\nmF{f}^2 + \frac{U}{2\delta^2}\norm{f}_2^2 - U.
\numit \label{eq:conineq1}
\end{align*}
In the following proof, we obtain the error rate of the uniform projected RMSE
$\sup_{g\in\mG}\norm{\proj_Z(\hhat_g-h_g^{*})}_2$ by combinding upper and lower bounds of the sup-loss
\begin{equation}
  \label{eq:BoundObj}
  \sup_{f\in\mF} \Psi_n(\hhat_g,f,g) - \Psi_n(h_g^{*},f,g) - 2\lambda \left( \nmF{f}^2 + \frac{U}{\delta^2}\nmEmp{f}^2 \right).
\end{equation}

\paragraph{Upper bound of sup-loss \eqref{eq:BoundObj}.} 
By a simple decomposition of $\Psi_n^{\lambda}(h,f,g)$, we have
\begin{align*}
  \Psi_n^{\lambda}(h,f,g) &= \Psi_n(h,f,g) - \Psi_n(h_g^{*},f,g) + \Psi_n(h_g^{*},f,g) - \lambda \left( \nmF{f}^2 + \frac{U}{\delta^2}\nmEmp{f}^2 \right)\\
  &\geq \Psi_n (h,f,g) - \Psi_n(h_g^{*},f,g) - 2\lambda \left( \nmF{f}^2 + \frac{U}{\delta^2}\nmEmp{f}^2 \right)\\
  &\quad + \inf_{f\in\mF} \left\{ \Psi_n(h_g^{*},f,g) + \lambda \left( \nmF{f}^2 + \frac{U}{\delta^2}\nmEmp{f}^2 \right) \right\}\\
  &= \Psi_n (h,f,g) - \Psi_n(h_g^{*},f,g) - 2\lambda \left( \nmF{f}^2 + \frac{U}{\delta^2}\nmEmp{f}^2 \right)\\
  &\quad -\sup_{f\in\mF} \Psi_n^{\lambda}(h_g^{*},f,g), \quad\text{since $\mF$ is symmetric about $0$.}
\end{align*}
Taking $\sup_{f\in\mF}$ on both sides and picking $h\leftarrow \hhat_g$ yields
the basic inequality: 
\begin{align*}
  & \sup_{f\in\mF} \Psi_n(\hhat_g,f,g) - \Psi_n(h_g^{*},f,g) - 2\lambda \left( \nmF{f}^2 + \frac{U}{\delta^2}\nmEmp{f}^2 \right)\\
  \leq & \sup_{f\in\mF} \Psi_n^{\lambda}(h_g^{*},f,g) + \sup_{f\in\mF} \Psi_n^{\lambda}(\hhat_g,f,g)\\
  \leq & 2\sup_{f\in\mF} \Psi_n^{\lambda}(h_g^{*},f,g) + \lambda\mu (\nmH{h_g^{*}}^2 - \nmH{\hhat_g}^2),
\numit\label{eq:UpBound0}
\end{align*}
where the last inequality is given by the definition of $\hhat_g$ in \eqref{eq:minimax}. Now it
suffices to obtain the upper bound of $\sup_{f\in\mF}
\Psi_n^{\lambda}(h_g^{*},f,g)$ uniformly over $g\in\mG$.

\subparagraph{For upper bound of $\sup_{f\in\mF}\Psi_n^{\lambda}(h_g^{*},f,g)$.}
By the assumption that $\norm{g}_{\infty}\leq 1$, $\norm{h}_{\infty} \leq 1$ and $\norm{f}_{\infty}\leq 1$, 
we have $\norm{\frac{1}{2}\left\{g(W) - h(X)\right\}f(Z)}_{\infty}\leq 1$.
Then we apply Lemma 11 of \citet{foster2019orthogonal}, with $\mL_{\frac{1}{2}(g-h_g^{*})f}=\frac{1}{2}(g-h_g^{*})f$. 
Let $\delta_n$ be the upper
bound of critical radii of $\bOmega$. %
By choosing $\delta = \delta_n + c_0 \sqrt{\frac{\log
(c_1/\zeta)}{n}}$, we have with probability $1-\zeta$, %
uniformly for any $f\in \mF_{3U}$ and $g\in\mG$: 
\begin{align*}
     & \frac{1}{2}\left| \left\{ \Psi_n(h_g^{*},f,g) - \Psi_n(h_g^{*},0,g) \right\} - \left\{ \Psi(h_g^{*},f,g) - \Psi(h_g^{*},0,g) \right\}\right|\\ 
  &\leq 18\delta \left( \norm{\frac{1}{2}(g-h_g^{*})f}_2 + \delta \right)\\
  &\leq 18\delta \left( \norm{f}_2 + \delta \right),
\end{align*}
where, by definition, $\Psi_n(h_g^{*},0,g) = \Psi(h_g^{*},0,g)=0$. If $\nmF{f}^2 \geq
3U$, applying the above inequality with $f\leftarrow f\sqrt{3U}/\nmF{f}$, we have with probability
$1-\zeta$, for all $f\in\mF$ and $g\in\mG$:
\begin{align*}
  \left| \Psi_n(h_g^{*},f,g) - \Psi(h_g^{*},f,g) \right| &\leq 36\delta \left\{ \norm{f}_2 + \max\left\{1,\frac{\nmF{f}}{\sqrt{3U}}\right\} \delta  \right\}\\
                                                     &\leq 36\delta \left\{ \norm{f}_2 + \left( 1+\frac{\nmF{f}}{\sqrt{3U}} \right) \delta \right\}. \numit\label{eq:conineq2}
\end{align*}
By using \eqref{eq:conineq2} and \eqref{eq:conineq1} sequentially, we
have with probability $1-2\zeta$, for all $f\in\mF$ and $g\in\mG$:
\begin{align*}
 & \Psi_n^{\lambda}(h_g^{*},f,g) = \Psi_n(h_g^{*},f,g) - \lambda \left( \nmF{f}^2 + \frac{U}{\delta^2}\nmEmp{f}^2 \right)\\
  \leq &\Psi (h_g^{*},f,g) + 36\delta \left\{ \norm{f}_2 + \left( 1+ \frac{\nmF{f}}{\sqrt{3U}} \right) \delta \right\} - \lambda \left( \nmF{f}^2 + \frac{U}{\delta^2}\nmEmp{f}^2 \right)\\
  \leq &\Psi (h_g^{*},f,g) + 36\delta \left\{ \norm{f}_2 + \left( 1+ \frac{\nmF{f}}{\sqrt{3U}} \right) \delta \right\} - \lambda \left( \frac{2}{3}\nmF{f}^2 + \frac{U}{2\delta^2}\norm{f}_2^2 \right) + \lambda U\\
  =& \Psi^{\lambda/2}(h_g^{*},f,g)+36\delta^2 + \lambda U + \left( 36\delta \norm{f}_2 - \frac{\lambda U}{4\delta^2}\norm{f}_2^2 \right) +\left( \frac{36\delta}{\sqrt{3U}}\delta\nmF{f} - \frac{\lambda}{3}\nmF{f}^2 \right).
\end{align*}
With the assumption that $\lambda\geq 324C_{\lambda}\delta^2/U$, by completing squares, we have
\begin{equation*}
  36\delta\norm{f}_2 - \frac{\lambda U}{4\delta^2}\norm{f}_2^2 \leq \frac{(36\delta)^2}{4 \frac{\lambda U}{4\delta^2}}\leq \frac{4\delta^2}{C_{\lambda}},\quad \text{ and}
\end{equation*}
\begin{equation*}
  \frac{36\delta^2}{\sqrt{3U}}\nmF{f} - \frac{\lambda}{3}\nmF{f}^2 \leq \frac{324\delta^4}{\lambda U} \leq \frac{\delta^2}{C_{\lambda}}.
\end{equation*}
Therefore, with probability $1-2\zeta$, for all $f\in\mF$ and $g\in\mG$:
\begin{equation}
  \label{eq:UpperBound1}
  \Psi_n^{\lambda} (h_g^{*},f,g) \leq \Psi^{\lambda/2}(h_g^{*},f,g) + \lambda U + \left( 36+\frac{5}{C_{\lambda}} \right)\delta^2.
\end{equation}

Now we go back to \eqref{eq:UpBound0}. By applying two upper bounds 
above, we have with probability $1-2\zeta$, uniformly for all $g\in\mG$:
\begin{align*}
  & \sup_{f\in\mF} \Psi_n(\hhat_g,f,g) - \Psi_n(h_g^{*},f,g) - 2\lambda \left( \nmF{f}^2 + \frac{U}{\delta^2}\nmEmp{f}^2 \right)\\
  \leq & 2\sup_{f\in\mF} \Psi_n^{\lambda}(h_g^{*},f,g)  + \lambda\mu (\nmH{h_g^{*}}^2 - \nmH{\hhat_g}^2)\\
  \leq & 2\sup_{f\in\mF} \Psi^{\lambda/2}(h_g^{*},f,g) 
         + 2\lambda U +(72+10/C_{\lambda})\delta^2 + \lambda\mu (\nmH{h_g^{*}}^2 - \nmH{\hhat_g}^2)\\
  = & 2\lambda U +(72+10/C_{\lambda})\delta^2 + \lambda\mu (\nmH{h_g^{*}}^2 - \nmH{\hhat_g}^2),
    \numit\label{eq:UpBound}
\end{align*}
where $\sup_{f\in\mF}\Psi^{\lambda/2}(h_g^{*},f,g^{*})=0$ since $\EE \left\{
  g(W) - h_g^{*}(X) \right\} f(Z)=0$.

We can also obtain the upper bound of $\nmH{\hhat_g}$ by \eqref{eq:UpBound}. By
choosing $f=0$, the LHS of \eqref{eq:UpBound} is 0, so the supremum of LHS is nonnegative. Then with probability $1-2\zeta$,
\begin{align*}
     \nmH{\hhat_g}^2 &\leq \frac{1}{\lambda\mu} \Big\{2\lambda U +(72+10/C_{\lambda})\delta^2\Big\} + \nmH{h_g^{*}}^2\\
                     &\leq \frac{36 C_{\lambda} + 3 + \frac{5}{9C_{\lambda}}}{\frac{24C_{\lambda}L^2}{U}+ \frac{C_f+1}{B}} + \nmH{h_g^{*}}^2.
     \numit\label{eq:hhatNorm}
\end{align*}

\paragraph{Lower bound of sup-loss \eqref{eq:BoundObj}.}
For any $h$ and $g$, by our assumption that $\norm{f_{\Delta} - \proj_Z(h-h_g^{*})}_2\leq \eta_n$, where  $f_{\Delta}= \argmin_{f\in \mF_{L^2\nmH{h-h_g^{*}}^2}}
\norm{f-\proj_Z(h-h_g^{*})}_2$.
Let $\hDelta_g = \hhat_g - h_g^{*}$, and $f_{\hDelta_g}=\argmin_{f\in\mF_{L^2\nmH{\hhat_g-h^{*}_g}}}\norm{f-\proj_Z(h-h_g^{*})}_2$.

If $\norm{f_{\hDelta_g}}_2<C_f\delta$, then by the triangle inequality, we have
\begin{equation*}
\norm{\proj_Z(\hhat_g - h_g^{*})}_2 \leq \norm{f_{\hDelta_g}}_2 + \norm{f_{\hDelta_g}-\proj_Z(\hhat_g-h_g^{*})}_2 \leq C_f\delta + \eta_n.
\end{equation*}
If $\norm{f_{\hDelta_g}}_2\geq C_f\delta$, let $r=\frac{C_f\delta}{2\norm{f_{\hDelta_g}}_2}\in[0,1/2]$. By star-convexity,
$rf_{\hDelta_g}\in\mF_{L^2\nmH{\hhat_g-h_g^{*}}^2}$. Therefore, for any $g\in\mG$,
\begin{align*}
  &\sup_{f\in\mF}\Psi_n(\hhat_g,f,g) - \Psi_n(h_g^{*},f,g) - 2\lambda \left( \nmF{f}^2 + \frac{U}{\delta^2}\nmEmp{f}^2 \right) \\
  &\qquad\qquad\geq \underbrace{r\left\{ \Psi_n(\hhat_g,f_{\hDelta_g},g) - \Psi_n(h_g^{*},f_{\hDelta_g},g) \right\}}_{(I)} - 2\lambda \underbrace{r^2 \left( \nmF{f_{\hDelta_g}}^2 + \frac{U}{\delta^2}\nmEmp{f_{\hDelta_g}}^2 \right)}_{(II)}.
\end{align*}
\subparagraph{For (II):} We have
\begin{align*}
  (II)&=r^2 \left( \nmF{f_{\hDelta_g}}^2 + \frac{U}{\delta^2} \nmEmp{f_{\hDelta_g}}^2 \right)  \leq \frac{1}{4}\nmF{f_{\hDelta_g}}^2 + \frac{U}{\delta^2}r^2\nmEmp{f_{\hDelta_g}}^2\\
&\leq \frac{1}{4} \nmF{f_{\hDelta_g}}^2 + \frac{U}{\delta^2}r^2 \left( \frac{3}{2}\norm{f_{\hDelta_g}}_2^2 + \delta^2+\delta^2 \frac{\nmF{f_{\hDelta_g}}^2}{3U} \right) \text{ with probability $1-\zeta$ by \eqref{eq:Thm14.1}}\\
  &\leq \frac{1}{3}\nmF{f_{\hDelta_g}}^2 + \frac{1}{4}U + \frac{3}{8}C_f^2U \text{ by definition of $r$}\\
&\leq \frac{1}{3} L^2 \nmH{\hhat_g-h_g^{*}}^2 + (\frac{1}{4} + \frac{3}{8}C_f^2)U \text{ since } f_{\hDelta_g}\in \mF_{L^2\nmH{\hhat_g-h_g^{*}}^2}.
\end{align*}
\subparagraph{For (I):}

Note that $\Psi_n(h,f,g) -\Psi_n(h_g^{*},f,g)=
\frac{1}{n}\sum_{i=1}^n[h-h_g^{*}](X_i)f(Z_i)$. We apply Lemma 11 of  \citet{foster2019orthogonal}, with $\mL_{(h-h_g^{*})f}=(h-h_g^{*})f$. 
Recall that 
\begin{equation*}
\bXi = \left\{ (x,z)\mapsto r[h-h_g^{*}](x)f_{\Delta}^{L^2 B}(z): h\in\mH, (h-h_g^{*})\in\mH_B, g\in\mG, r\in[0,1]\right\},
\end{equation*}
where $f_{\Delta}^{L^2 B} = \argmin_{f\in\mF_{L^2 B}}\norm{f-\proj_Z(h-h_g^{*})}_2$.
Since $\delta_n$ upper bounds critical radius of
$\bXi$, we have with probability $1-\zeta$, uniformly for all $g\in\mG$, and
$h\in\mH$ such that $\Delta = h-h_g^{*}\in\mH_B$,
\begin{align*}
  &\left| \left\{ \Psi_n(h,f_{\Delta},g) - \Psi_n(h_g^{*},f_{\Delta},g) \right\} - \left\{ \Psi(h,f_{\Delta},g) - \Psi(h_g^{*},f_{\Delta},g) \right\} \right| \\
  &\leq 18\delta \left(\norm{(h-h_g^{*})f_{\Delta}}_2 + \delta \right)\\
  &\leq 18\delta(\norm{f_{\Delta}}_2 + \delta),
\end{align*}
where in the second inequality, we use the fact that $h-h_g^{*}\in\mH_B$, so
that $\norm{h-h_g^{*}}_{\infty}\leq 1$.
When $\nmH{\Delta}^2=\nmH{h-h_g^{*}}^2 > B$, by  replacing $h-h_g^{*}$ by
$(h-h_g^{*})\sqrt{B}/\nmH{h-h_g^{*}}$ and multiplying both sides by
$\nmH{h-h_g^{*}}^2/B$, we have with probability $1-\zeta$, uniformly for all $h\in\mH$, $g\in\mG$,
\begin{align*}
  &\left| \left\{ \Psi_n(h,f_{\Delta},g) - \Psi_n(h_g^{*},f_{\Delta},g) \right\} - \left\{ \Psi(h,f_{\Delta},g) - \Psi(h_g^{*},f_{\Delta},g) \right\} \right| \\
  &\leq 18\delta(\norm{f_{\Delta}}_2 + \delta) \max \left\{ 1,\frac{\nmH{h-h_g^{*}}^2}{B} \right\}.
\end{align*}

When $\norm{f_{\hDelta_g}}_2\geq C_f\delta$, with probability $1-\zeta$, uniformly for all $g\in\mG$,
\begin{align*}
  (I) \geq & r \left\{ \Psi(\hhat_g,f_{\hDelta_g},g) - \Psi(h_g^{*},f_{\hDelta_g},g) \right\} - 18\delta r \left[ \norm{f_{\hDelta_g}}_2 + \delta \right]\max\left\{ 1,\frac{\nmH{\hhat_g-h_g^{*}}^2}{B} \right\}\\
  \geq & \underbrace{r \left\{ \Psi(\hhat_g,f_{\hDelta_g},g) - \Psi(h_g^{*},f_{\hDelta_g},g) \right\}}_{(I.1)} - 9\delta \left[ C_f\delta + \delta\right]\max\left\{ 1,\frac{\nmH{\hhat_g-h_g^{*}}^2}{B} \right\},
\end{align*}
where the second inequality is due to the definition of $r=\frac{C_f\delta}{2\norm{f_{\hDelta_g}}_2}\leq \frac{1}{2}$, and 
\begin{align*}
  (I.1) &= \frac{C_f\delta}{2\norm{f_{\hDelta_g}}_2} \left\{ \Psi(\hhat_g,f_{\hDelta_g},g) - \Psi(h_g^{*},f_{\hDelta_g},g) \right\}\\
        &= \frac{C_f\delta}{2\norm{f_{\hDelta_g}}_2} \EE \left\{ \hhat_g(X) - h_g^{*}(X)\right\}f_{\hDelta_g}(Z)\\
        &= \frac{C_f\delta}{2\norm{f_{\hDelta_g}}_2} \EE \left( f_{\hDelta_g}(Z) \EE \left[\hhat_g(X) - h_g^{*}(X) \mid Z \right] \right)\\
        &= \frac{C_f\delta}{2\norm{f_{\hDelta_g}}_2} \EE \left( f_{\hDelta_g}(Z) \left\{ \proj_Z(\hhat_g-h_g^{*})(Z)\right\} \right)\\
          &=\frac{C_f\delta}{2\norm{f_{\hDelta_g}}_2} \EE \left[f_{\hDelta_g}(Z)^2 - \left\{ f_{\hDelta_g}(Z)-\proj_Z(\hhat_g-h_g^{*})(Z) \right\}f_{\hDelta_g}(Z) \right]\\
  &\geq \frac{C_f\delta}{2} \left( \norm{f_{\hDelta_g}}_2 - \norm{f_{\hDelta_g}-\proj_Z(\hhat_g-h_g^{*})}_2 \right)\text{ by Cauchy-Schwartz inequality}\\
  &\geq \frac{C_f\delta}{2} \left( \norm{f_{\hDelta_g}}_2 - \eta_n \right) \text{ since }\norm{f_{\hDelta_g} - \proj_Z(h_g^{*}-\hhat_g)}_2\leq \eta_n\\
  &\geq \frac{C_f\delta}{2} \left( \norm{\proj_Z(\hhat_g-h_g^{*})}_2 - 2\eta_n \right) \text{ by triangle inequality.}
\end{align*}
Finally, we have either $\norm{f_{\hDelta_g}}_2 < C_f\delta$ or with probability $1-2\zeta$,
uniformly for all $g\in\mG$:
\begin{align*}
  &\sup_{f\in\mF} \Psi_n(\hhat_g,f,g) - \Psi_n(h_g^{*},f,g) - 2\lambda \left( \nmF{f}^2 + \frac{U}{\delta^2}\nmEmp{f}^2 \right) \geq (I) - 2\lambda(II) \\
  \geq &\frac{C_f\delta}{2} \left( \norm{\proj_Z(\hhat_g-h_g^{*})}_2 -2\eta_n \right) - 9(C_f+1)\delta^2\max\left\{ 1,\frac{\nmH{\hhat_g-h_g^{*}}^2}{B} \right\}\\
  &\qquad - \frac{2\lambda}{3}L^2 \nmH{\hhat_g-h_g^{*}}^2 - 2\lambda (\frac{1}{4} + \frac{3}{8}C_f^2)U.
\numit\label{eq:LoBound}
\end{align*}

\paragraph{Combine upper and lower bounds of \eqref{eq:BoundObj}.}
Combining the upper bound \eqref{eq:UpBound} and lower bound \eqref{eq:LoBound}, we
have either $\norm{f_{\hhat_g}}_2 < C_f\delta$ or with probability $1-4\zeta$, uniformly for
all $g\in\mG$:
\begin{align*}
  \frac{C_f\delta}{2} \norm{\proj_Z(\hhat_g-h_g^{*})}_2 \leq & 2\lambda U +\left(72+\frac{10}{C_{\lambda}}\right)\delta^2 + \lambda\mu(\nmH{h_g^{*}}^2-\nmH{\hhat_g}^2)\\
&\quad+C_f\delta\eta_n +9(C_f+1)\delta^2\max\left\{ 1,\frac{\nmH{\hhat_g-h_g^{*}}^2}{B} \right\}\\
&\quad+\frac{2\lambda}{3}L^2 \nmH{\hhat_g - h_g^{*}}^2 + \left(\frac{1}{2}+\frac{3}{4}C_f^2\right)\lambda U\\
& =  \lambda\mu(\nmH{h_g^{*}}^2-\nmH{\hhat_g}^2)+ \left( \frac{2\lambda}{3} L^2 +\frac{9(C_f+1)\delta^2}{B} \right)\nmH{\hhat_g - h_g^{*}}^2\\
&\quad+\left(\frac{5}{2}+\frac{3}{4}C_f^2\right)\lambda U +C_f\delta\eta_n + \left( 72+\frac{10}{C_{\lambda}}+9(C_f+1) \right)\delta^2.
\end{align*}

Then, with the assumption that $\mu\geq \frac{4}{3}L^2 + \frac{18(C_f+1)}{B}\frac{\delta^2}{\lambda}$, we have
\begin{align*}
  &\lambda\mu(\nmH{h_g^{*}}^2-\nmH{\hhat_g}^2)+ \left( \frac{2\lambda}{3} L^2 +\frac{9(C_f+1)\delta^2}{B} \right)\nmH{\hhat_g - h_g^{*}}^2\\ 
  \leq & \lambda\mu(\nmH{h_g^{*}}^2-\nmH{\hhat_g}^2)+2\left( \frac{2\lambda}{3} L^2 +\frac{9(C_f+1)\delta^2}{B} \right)\left(  \nmH{\hhat_g}^2 + \nmH{h_g^{*}}^2\right)\\
  \leq & 2\lambda\mu\nmH{h_g^{*}}^2\leq 2\lambda\mu\sup_{g\in\mG}\nmH{h_g^{*}}^2.
\end{align*}

Finally, with probability $1-4\zeta$, uniformly for all $g\in\mG$:
\begin{align*}
&\quad\sup_{g\in\mG}\norm{\proj_Z(\hhat_g - h_g^{*})}_2\\ 
&\leq \left( \frac{4\mu\sup_{g\in\mG}\nmH{h_g^{*}}^2+5U}{C_f}+ \frac{3U}{2}C_f \right) \frac{\lambda}{\delta} +2\eta_n + \left( \frac{162+20/C_{\lambda}}{C_f}+18 \right)\delta\\
&\lesssim \left[ 324 C_{\lambda}^{\prime} \left( \frac{4\mu\sup_{g\in\mG}\nmH{h_g^{*}}^2/U+5}{C_f}+ \frac{3}{2}C_f \right) + \frac{162+20/C_{\lambda}}{C_f}+18\right]\delta + 2\eta_n\\
&\lesssim (1+\sup_{g\in\mG}\nmH{h_g^{*}}^2)\delta,
\end{align*}
where the second inequality is due to the assumption that
$324C_{\lambda}\delta^2/U\leq \lambda \leq 324C_{\lambda}^{\prime}\delta^2/U$, and
the last inequality is due to the assumption that $\eta_n\lesssim\delta_n$.
\end{proof}

\subsubsection{Proof of Lemma \ref{lem:radius by entropy}}
\label{proof: lemma radius by entropy}
\begin{proof}
\textbf{Step 1. Critical radius of $\mF_{3U}$.}~~~
Directly applying Lemma \ref{lem:Dudley integral}, we only require that $\hdelta_n$
satisfies the inequality
\begin{equation*}
\frac{64}{\sqrt{n}} \int_{\frac{\delta^2}{2}}^{\delta} \sqrt{\log N_n(t,\star{\mF_{3U}})}{\rm d}t \leq \delta^2.
\end{equation*}
Then with probability $1-\zeta$, we have $\delta_n\leq \bigO(\hdelta_n + \sqrt{\frac{\log (1/\zeta)}{n}})$, where $\delta_n$ is the maximum critical radii of $\bOmega$.

\textbf{Step 2. Critical radius of $\bXi$.}~~~
  
 Since $\bXi\subset \left\{ (x,z)\mapsto rh(x)f(z): h\in\mH_B, f\in\mF_{L^2
    B}, r\in [0,1] \right\} \triangleq\tilde\bXi$, we only need to
  consider a conservative critical radius for $\tilde\bXi$.

  Suppose that $\mH_B^{\epsilon}$ is an empirical $\epsilon$-covering of $\star{\mH_B}$ and 
  $\mF_{L^2B}^{\epsilon}$ is an empirical $\epsilon$-covering of $\star{\mF_{L^2 B}}$.
  Then for any $rhf\in\tilde\bXi$, $r\in[0,1]$,
  
\begin{align*}
  \inf_{h_{\epsilon}\in\mH_B^{\epsilon},f_{\epsilon}\in\mF_{L^2B}^{\epsilon}} \norm{h_{\epsilon}f_{\epsilon}-rhf}_n &\leq \inf_{h_{\epsilon}\in\mH_B^{\epsilon}}\norm{(h_{\epsilon}-h)f_{\epsilon}}_n + \inf_{f_{\epsilon}\in\mF_{L^2B}^{\epsilon}}\norm{h(rf-f_{\epsilon})}_n\\
&\leq \inf_{h_{\epsilon}\in\mH_B^{\epsilon}}\norm{h_{\epsilon}-h}_n + \inf_{f_{\epsilon}\in\mF_{L^2B}^{\epsilon}}\norm{rf-f_{\epsilon}}_n\\
  &\leq 2\epsilon.
\end{align*}
Therefore, $\mH_B^{\epsilon/2}\times\mF_{L^2B}^{\epsilon/2}$ is an empirical
$\epsilon$-covering of $\bXi$. Since
\begin{align*}
  \log N_n(t,\BB_n(\delta, G_{\Delta})) & \leq \log N_n(t,\BB_n(\delta,\tmG_{\Psi})) \leq \log N_n(t,\tmG_{\Psi})\\
  &\leq \log N_n(t/2, \star{\mH_B}) + \log N_n(t/2,\star{\mF_{L^2 B}}),
\end{align*}
by Lemma \ref{lem:Dudley integral}, we only require that $\hat\delta_n$ satisfies
the inequality
\begin{equation*}
  \frac{64}{\sqrt{n}}\int_{\frac{\delta^2}{2}}^{\delta} \sqrt{\log N_n(t/2,\star{\mH_B}) + \log N_n(t/2,\star{\mF_{L^2 B}})} {\mathrm d}t \leq \delta^2.
\end{equation*}
Then with probability $1-\zeta$, we have $\delta_n\leq \bigO(\hdelta_n + \sqrt{\frac{\log (1/\zeta)}{n}})$, where $\delta_n$ is the maximum critical radii of $\bOmega$.

\textbf{Step 3. Critical radius of $\bOmega$.}~~~

  \begin{align*}
    \bOmega &\triangleq \left\{ (x,w,z) \mapsto r(h_g^{*}(x)-g(w))f(z): g\in\mG_D, f\in\mF_{3U}, r\in[0,1] \right\}\\
  &\subset \left\{ (x,w,z)\mapsto r(h(x)-g(w))f(z): g\in\mG_D,h\in\mH_{AD}, f\in\mF_{3U},r\in[0,1] \right\}\\
  &\triangleq \tmG_{\Psi},
  \end{align*}
 where the second line is due to $\nmH{h_g^{*}}^2 \leq A \nmG{g}^2$ for all $g\in\mG$.
 Suppose that  $\mH_{AD}^{\epsilon}$ is an empirical $\epsilon$-covering of
 $\star{\mH_{AD}}$ and $\mG_D^{\epsilon}$ is that of $\star{\mG_D}$, $\mF_{3U}^{\epsilon}$ is
 that of $\star{\mF_{3U}}$. Then for any $r(h-g)f\in\tmG_{\Psi}$, $r\in [0,1]$,
\begin{align*}
  &\inf_{h_{\epsilon}\in\mH_{AD}^{\epsilon}, f_{\epsilon}\in\mF_{3U}^{\epsilon},g_{\epsilon}\in\mG_D^{\epsilon}} \norm{r(h-g)f - (h_{\epsilon}-g_{\epsilon})f_{\epsilon}} \\ 
  &\leq \inf_{f_{\epsilon}\in\mF_{3U}^{\epsilon}}\norm{(h-g)(f_{\epsilon}-rf)}_n + \inf_{h_{\epsilon}\in\mH_{AD}^{\epsilon}}\norm{(h_{\epsilon}-h)f_{\epsilon}}_n + \inf_{g_{\epsilon}\in\mG_D^{\epsilon}}\norm{(g_{\epsilon}-g)f_{\epsilon}}_n\\
  &\leq \inf_{f_{\epsilon}\in\mF_{3U}^{\epsilon}}2\norm{f_{\epsilon}-rf}_n + \inf_{h_{\epsilon}\in\mH_{AD}^{\epsilon}}\norm{h_{\epsilon}-h}_n + \inf_{g_{\epsilon}\in\mG_D^{\epsilon}}\norm{g_{\epsilon}-g}_n\\
  & \leq 4\epsilon,
\end{align*}
where the second inequality is from triangular inequality and the thrid
inequality is due to the fact that $\norm{h-g}_{\infty}\leq 2$ and
$\norm{f_{\epsilon}}_{\infty}\leq 1$.

Therefore, $\mH_{AD}^{\epsilon/4}\times\mG_D^{\epsilon/4}\times\mF_{3U}^{\epsilon/4}$ is an
empirical $\epsilon$-covering of $\bOmega$.

By Lemma \ref{lem:Dudley integral}, we only require that $\hdelta_n$ satisfies
the Dudley's integral inequality. Actually, since
\begin{align*}
  \log N_n(t, \BB_n(\delta,\bOmega)) & \leq \log N_n(t,\BB_n(\delta,\tmG_{\Psi}))\leq \log N_n(t,\tmG_{\Psi})\\ 
  &\leq \log N_n(t/4,\star{\mH_{AD}}) + \log N_n (t/4,\star{\mG_D})\\ 
  &\qquad\qquad + \log N_n(t/4,\star{\mF_{3U}}),
\end{align*}
when $\hdelta_n$ satisfies the inequality
\begin{equation*}
  \frac{64}{\sqrt{n}}\int_{\frac{\delta^2}{2}}^{\delta} \sqrt{\log N_n(t/4,\star{\mH_{AD}}) + \log N_n (t/4,\star{\mG_D}) + \log N_n(t/4,\star{\mF_{3U}})} {\rm d}t \leq \delta^2,
\end{equation*}
then with probability $1-\zeta$, we have $\delta_n\leq \bigO(\hdelta_n + \sqrt{\frac{\log (1/\zeta)}{n}})$, where $\delta_n$ is the maximum critical radii of $\bOmega$.
Finally, after combining Steps 1-3, we have that if $\hdelta_n$ satisfies the inequality
\begin{align*}
  \frac{64}{\sqrt{n}} \int_{\frac{\delta^2}{2}}^{4\delta} \sqrt{\log N_n(t,\star{\mF_{3U\vee L^2 B}}) + \log N_n(t,\star{\mH_{AD \vee B}})+ \log N_n(t,\star{\mG_D})}{\rm d}t \leq \delta^2,
\end{align*}

then with probability $1-\zeta$, we have $\delta_n\leq \bigO(\hdelta_n + \sqrt{\frac{\log
(1/\zeta)}{n}})$, where $\delta_n$ is the maximum critical radii of $\mF_{3U}$,
$\bXi$ and $\bOmega$.
\end{proof}

\subsubsection{Proof of Lemma \ref{lem:radius by spectra}}
\label{proof: lemma radius by spectra}
\begin{proof}
\textbf{Critical radius of $\bXi$.}~~~
We consider a conservative critical radius for $\tmG_{\Delta}$, which is a
tensor product of two RKHSs $\mH_B$ and $\mF_{L^2 B}$. Suppose that $\mH$ and $\mF$
are endowed with reproducing kernels $K_{\mH}$ and $K_{\mF}$, with ordered eigenvalues
$\left\{ \lambda_j^{\downarrow}(K_{\mH}) \right\}_{j=1}^{\infty}$ and $\left\{ \lambda_j^{\downarrow}(K_{\mF}) \right\}_{j=1}^{\infty}$, respectively.
Then the RKHS $\tmG_{\Delta}$ has reproducing kernel $K_{\bXi} = K_{\mH}\otimes K_{\mF}$, with eigenvalues
 $\left\{ \lambda_j^{\downarrow}(K_{\mH}) \right\}_{j=1}^{\infty}\times\left\{ \lambda_j^{\downarrow}(K_{\mF}) \right\}_{j=1}^{\infty}$.
Therefore, by Lemma \ref{lem:criticalradRKHS},
\begin{align*}
  \mR_n(\tmG_{\Delta},\delta) \leq \sqrt{\frac{2L^2B^2}{n}}\sqrt{\sum_{i,j=1}^{\infty}\min\left\{\lambda_i^{\downarrow}(K_{\mH})\lambda_j^{\downarrow}(K_{\mF}),\delta^2\right\}}.
\end{align*}

\textbf{Critical radius of $\bOmega$.}~~~
  We consider a conservative critical radius for
\begin{equation*}
  \tmG_{\Psi}=\left\{ (x,w,z)\mapsto r(h(x)-g(w))f(z): g\in\mG_D,h\in\mH_{AD}, f\in\mF_{3U},r\in[0,1] \right\}.
\end{equation*}
Let $\tilde{h}(x,w) = h(x)$ and $\tilde{g}(x,w) = g(w)$, $x\in\mX,w\in\mW$.
In addition, $\tilde{h}\in\tmH_{AD}$ on $\mX\times\mW$ with kernel $K_{\tmH} =
K_{\mH}\otimes 1$ and $\tilde{g}\in\tmG_D$ on $\mX\times\mW$ with kernel
$K_{\tmG}=1\otimes K_{\mG}$. Notice that $ h-g \in \tmH_{AD} + \tmG_D$, which is
a RKHS endowed with RKHS norm $\norm{f}_{\tmH+\tmG} =
\min_{f=\tilde{h}+\tilde{g},\tilde{h}\in\tmH,\tilde{g}\in\tmG}
\norm{\tilde{h}}_{\tmH}+\norm{\tilde{g}}_{\tmG}$, and reproducing kernel $K_{\tmH}+K_{\tmG}$.
As a result, $\norm{h-g}_{\tmH+\tmG}\leq \sqrt{AD}+\sqrt{D}$ for all $h-g\in\tmH_{AD} + \tmG_D$.

According to Weyl's inequality for compact self-adjoint operators in Hilbert spaces (see the $s$-number sequence theory in \citet{hinrichs2006optimal} and \citet[][2.11.9]{pietsch1987eigenvalues}),
$\lambda_{i+j-1}^{\downarrow}(K_{\tmH}+K_{\tmG})\leq \lambda_i^{\downarrow}(K_{\tmH}) +
\lambda_j^{\downarrow}(K_{\tmG}) = \lambda_i^{\downarrow}(K_{\mH}) +
\lambda_j^{\downarrow}(K_{\mG})$ whenever $i,j\geq 1$, so we have 
$\lambda_{j}^{\downarrow}(K_{\tmH}+K_{\tmG})\leq \lambda_{[(j+1)/2]}^{\downarrow}(K_{\mH}) +
\lambda_{[(j+1)/2]}^{\downarrow}(K_{\mG})$ whenever $j\geq 1$.

Since $(\tmH+\tmG)\otimes\mF$ is a RKHS with reproducing kernel
$(K_{\tmH}+K_{\tmG})\otimes K_{\mF}$, by the same argument for $\bXi$, we
have 
\begin{align*}
  \mR_n(\tmG_{\Psi},\delta) &\leq \sqrt{D}(1+\sqrt{A})\sqrt{\frac{6U}{n}}\sqrt{\sum_{i,j=1}^{\infty}\min \left\{ [\lambda_{[(i+1)/2]}^{\downarrow}(K_{\mH})+\lambda_{[(i+1)/2]}^{\downarrow}(K_{\mG})]\lambda_j^{\downarrow}(K_{\mF}),\delta^2 \right\}}\\
                            &\leq \sqrt{D}(1+\sqrt{A})\sqrt{\frac{12U}{n}}\sqrt{\sum_{i,j=1}^{\infty}\min \left\{ [\lambda_{i}^{\downarrow}(K_{\mH})+\lambda_{i}^{\downarrow}(K_{\mG})]\lambda_j^{\downarrow}(K_{\mF}),\delta^2 \right\}}.
\end{align*}
\end{proof}

\section{Additional estimation details} \label{sec: simulation appendix}

In this section we demonstrate the performance of the proposed FQE-type algorithm introduced in Section \ref{sec:estimation} for the case where $\mH^{(t)}$ and $\mF^{(t)}$ are Reproducing kernel Hilbert spaces (RKHSs) endowed with reproducing kernels $K_{\mH^{(t)}}$ and $K_{\mF^{(t)}}$ respectively and canonical RKHS norms $\norm{\bullet}_{\mH^{(t)}} = \norm{\bullet}_{K_{\mH^{(t)}}}$, $\norm{\bullet}_{\mF^{(t)}} = \norm{\bullet}_{K_{\mF^{(t)}}}$ respectively, for $1\leq t\leq T$.

For each $1\leq t\leq T$, based on observed batch data $ \{S_{t,i},W_{t,i},Z_{t,i},A_{t,i},R_{t,i}\}_{i=1}^n$, we can obtain the Gram matrices $\bK_{\mH^{(t)}} = \left[K_{\mH^{(t)}}([W_{t,i},S_{t,i},A_{t,i}], [W_{t,j},S_{t,j},A_{t,j}])\right]_{i,j=1}^n$ and $\bK_{\mF^{(t)}} = \left[K_{\mF^{(t)}}([Z_{t,i},S_{t,i},A_{t,i}], [Z_{t,j},S_{t,j},A_{t,j}])\right]_{i,j=1}^n$. Then we compute $\hqpi_t = \widehat{\mP}_t (\hvpi_{t+1} + R_t)$ via \eqref{eqn: minmax estimation} with $g=\hvpi_{t+1} + R_t$. Specifically,  $\hqpi_t$ has the following form:
\begin{equation}
    \label{eq: RKHS solution}
    \hqpi_t(w,s,a)=[\widehat{\mP}_t (\hvpi_{t+1} + R_t)](w,s,a) = \sum_{i=1}^n \alpha_i K_{\mH^{(t)}}([W_{t,i},S_{t,i},A_{t,i}],[w,s,a]),
\end{equation}
where $\boldsymbol{\alpha} =[\alpha_1,\dots,\alpha_n]^{\top} = \left(\bK_{\mH^{(t)}}\bM^{(t)}\bK_{\mH^{(t)}} + 4\lambda^2\mu\bK_{\mH^{(t)}}\right)^{\dagger}\bK_{\mH^{(t)}}\bM^{(t)}\bY_t$ with $\bM^{(t)} = \bK_{\mF^{(t)}}^{1/2} (\frac{M}{n\delta^2}\bK_{\mF^{(t)}}+\bI_n)^{-1}\bK_{\mF^{(t)}}^{1/2}$, and $\bY_t=\bR_t + \hat{\mathbf{v}}_{t+1}^{\pi}$ with $\bR_t = [R_{t,1},\dots,R_{t,k}]^\top$ and $\hat{\mathbf{v}}_{t+1}^{\pi} = [\hvpi_{t+1}(W_{t+1,1},S_{t+1,1}),\dots, \hvpi_{t+1}(W_{t+1,n},S_{t+1,n})]^{\top}$. Here $\bA^{\dagger}$ denotes the Moore-Penrose pseudo-inverse of $\bA$.

\textbf{Selection of hyper-parameters.} There are several hyper-parameters in \eqref{eq: RKHS solution} for each $1\leq t\leq T$. In each step, we treat $\bY_t=R_t+\hat{\mathbf{v}}_{t+1}^{\pi}$ as the response vector and use cross-validation to tune $M/\delta^2$ and $\lambda^2\mu$ in \eqref{eq: RKHS solution}. We adopt the tricks of \cite{dikkala2020minimax} and use the recommended defaults in their Python package \texttt{mliv}, where two scaling functions are defined by $\varsigma(n) = 5/n^{0.4}$ and $\zeta(\texttt{scale},n) = \texttt{scale}\times\varsigma^4(n)/2$.

For cross-validation, let $I^{(1)},\dots,I^{(K)}$ denote the index sets of the randomly partitioned $K$ folds of the indices $\{1,\dots,n\}$ and $I^{(-k)} = \{1,\dots,n\}\backslash I^{(k)}$, $k=1,\dots,K$. We summarize the one-step NPIV estimation with cross-validation in Algorithm \ref{alg: one-step NPIV}.

\begin{minipage}{\linewidth}
\begin{algorithm}[H] \label{alg: one-step NPIV}
  \SetAlgoLined
  \textbf{Input:} $\{S_{t,i}, W_{t,i}, Z_{t,i}, A_{t,i}, Y_{t,i}=R_{t,i}+\hvpi_{t+1}(W_{t+1,i},S_{t+1,i})\}_{i=1}^n$, target policy $\pi_t$, kernels $K_{\mH^{(t)}}$, $K_{\mF^{(t)}}$, $\texttt{SCALE}$ as some positive scaling factors, the number of cross-validation partition $K$.\\
  Repeat for $\texttt{scale}\in\texttt{SCALE}$:\\
  \Indp
  Repeat for $k=1,\dots,K$:\\
  \Indp
  $[M/\delta^2]^{(-k)} = 1/\varsigma^2(|I^{(-k)}|)$, $[\lambda^2\mu]^{(-k)} = \zeta(\texttt{scale},|I^{(-k)}|)$.\\ 
  Obtain $\hat{q}_t^{\pi~(-k)}$ by \eqref{eq: RKHS solution} with data whose indices are in $I^{(-k)}$.\\
  $[M/\delta^2]^{(k)} = 1/\varsigma^2(|I^{(k)}|)$.\\
  Calculate $\epsilon_i = Y_{t,i} - \hat{q}_t^{\pi~(-k)}(W_{t,i},S_{t,i},A_{t,i})$ for $i\in I^{(k)}$.\\
  $\texttt{Loss}^{(k)}(\texttt{scale}) = \epsilon^{\top} \bM_{I^{(k)}}\epsilon$, where $\epsilon = [\epsilon_i]_{i\in I^{(k)}}^{\top}$ and $\bM_{I^{(k)}}$ is obtained by data in $I^{(k)}$.\\
  \Indm
  $\texttt{Loss}(\texttt{scale}) = K^{-1}\sum_{k=1}^K \texttt{Loss}^{(k)}(\texttt{scale})$.\\
  \Indm
  $\texttt{scale}^{*} = \argmin_{\texttt{scale}\in\texttt{SCALE}} \texttt{Loss}(\texttt{scale})$.\\
  Obtain $\hqpi_t$ by \eqref{eq: RKHS solution} with all data and $M/\delta^2 = 1/\varsigma^2(n)$, $\lambda^2\mu=\zeta(\texttt{scale}^{*},n)$.\\
  
\textbf{Output:}  $\{\hvpi_{t}(W_{t,i}, S_{t,i}) = \sum_{a\in\mA}\hqpi_{t}(W_{t,i}, S_{t,i}, a)\pi(a\given S_{t,i})\}_{i=1}^n$.

\caption{Min-max NPIV estimation with RKHSs}
\end{algorithm}
\end{minipage}

Below we summarize our proposed FQE-type algorithm using a sequential NPIV estimation with tuning procedure described in Algorithm \ref{alg:DetailedFQE}.

\begin{minipage}{\linewidth}
\begin{algorithm}[H] \label{alg:DetailedFQE}
  \SetAlgoLined
  \textbf{Input:} Batch Data $\mD_n = \{\{S_{t,i}, W_{t,i}, Z_{t,i}, A_{t,i}, R_{t,i}\}_{t=1}^T\}_{i=1}^n$, a target policy $\pi = \{\pi_t\}_{t=1}^T$, kernels $\{K_{\mH^{(t)}}, K_{\mF^{(t)}}\}_{t=1}^{T}$, set \texttt{SCALE} as some positive scaling factors, number of cross-validation partition $K$.\\
  Let $\hvpi_{T+1}=0$.\\
  Repeat for $t=T,\dots,1$:\\
  \Indp
  Obtain $\{\hvpi_{t}(W_{t,i}, S_{t,i})\}_{i=1}^n$ by Algorithm \ref{alg: one-step NPIV}.\\
  \Indm
  \textbf{Output:} $\widehat\mV(\pi) = n^{-1}\sum_{k=1}^n\hvpi_1(W_{1,k},S_{1,k})$.
  \caption{A FQE-type algorithm by sequential min-max NPIV estimation}
\end{algorithm}
\end{minipage}

\section{Simulation details}
\label{sec: simulation details}
In this section, we perform a simulation study to evaluate the performance of our proposed OPE estimation and to verify the finite-sample error bound of our OPE estimator in the main result Theorem \ref{thm:main}.

\subsection{Simulation setup}

Let $\mS=\RR^2$, $\mU=\RR,\mW=\RR,\mZ=\RR$, and $\mA=\left\{ 1,-1 \right\}$.
\paragraph{MDP setting.}
At time $t$, given $(S_t,U_t,A_t)$, we generate
\begin{equation*}
S_{t+1} = S_t + A_tU_t\ones_2 + e_{S_{t+1}},
\end{equation*}
where $\ones_2 = [1,1]^{\top}$ and the random error $e_{S_{t+1}}\sim \mN([0,0]^{\top}, \bI_2)$ with $\bI_2$ denoting the $2$-by-$2$ identity matrix. 

The behavior policy is 
\begin{equation*}
\tilde{\pi}_t^b (A_t\given U_t,S_t) = \expit \left\{ -A_t \left( t_0 + t_uU_t+t_s^{\top}S_t \right) \right\},
\end{equation*}
where $t_0=0$, $t_u=1$, and $t_s^{\top}=[-0.5, -0.5]$.

By this behavior policy 
$$\pi_t^{b}(A_t\given S_t) = \EE[\tilde{\pi}_t^b(A_t\given U_t,S_t)\given A_t,S_t] = \expit\{-A_t\left(t_0+t_u\kappa_0 + (t_s+t_u\kappa_s)^{\top}S_t\right)\},$$ 
provided that the following conditional distribution is used.

We generate the hidden state $U_t$, and two proximal variables $Z_t$ and $W_t$ by the following conditional multivariate normal distribution given $(S_t,A_t)$:
\begin{equation*}
(Z_t,W_t,U_t)\given (S_t,A_t) \sim N \left(
  \begin{bmatrix}
    \alpha_0 +\alpha_aA_t +\alpha_sS_t\\
    \mu_0    +\mu_aA_t    +\mu_sS_t\\
    \kappa_0 +\kappa_aA_t +\kappa_sS_t\\
  \end{bmatrix},
  \Sigma = \begin{bmatrix}
    \sigma_z^2 & \sigma_{zw} & \sigma_{zu}\\
    \sigma_{zw} & \sigma_w^2 & \sigma_{wu}\\
    \sigma_{zu} & \sigma_{wu} & \sigma_u^2
  \end{bmatrix}
\right), 
\end{equation*}
where 
\begin{itemize}
    \item  $\alpha_0=0$, $\alpha_a=0.5$, $\alpha_s^{\top} = [0.5,0.5]$,
    \item  $\mu_0=0$, $\mu_a=-0.25$, $\mu_s^{\top} = [0.5,0.5]$,
    \item  $\kappa_0=0$, $\kappa_a=-0.5$, $\kappa_s^{\top} = [0.5,0.5]$
    \item the covariance matrix 
    \[\Sigma = \begin{bmatrix}
      1 & 0.25 & 0.5\\
      0.25 & 1 & 0.5\\
      0.5 & 0.5 & 1
    \end{bmatrix}\]
\end{itemize}
The initial $S_1$ is uniformly sampled\footnote{Sample by \texttt{gym} package build in function \texttt{spaces.sample()} from \texttt{spaces.Box(low=-np.inf, high=np.inf, shape=(2,), dtype=np.float32)}.
} from $\RR^2$. 

\paragraph{Reward setting.}
The reward is given by
\[R_t = \expit\left\{\frac{1}{2}A_t(U_t+[1,-2]S_t)\right\} + e_t,\]
where $e_t\sim \text{Uniform}[-0.1,0.1]$. One can verify that our simulation setting satisfies the conditions in Section \ref{sec: basic assumptions} so that our method can be applied.

\paragraph{Target policy.} We evaluate a $\epsilon$-greedy policy $\pi(a\given S_t)$ 
maximizing the immediate reward:
\begin{equation*}
    A_t\given S_t \sim \left\{\begin{array}{lr} \text{sign}\left\{ \EE[U_t+[1,-2]S_t\given S_t]\right\} & \text{with probability } 1-\epsilon, \\
    \text{Uniform}\{-1,1\} & \text{with probability } \epsilon.
    \end{array}\right.
\end{equation*}
We set $\epsilon=0.2$. 

\subsection{Implementation}

We present the results of policy evaluation for the simulation setup above. Specifically, to evaluate the finite-sample error bound of the proposed estimator in terms of the sample size $n$, we consider $T=1,3,5$ and let $n=256, 512, 1024, 2048, 4096$; to evaluate the estimation error of our OPE estimator in terms of the length of horizon $T$, we fix $n=512$ and let $T=1,2,4,8,16,24,32,48,64$. For each setting of $(n,T)$, we repeat 100 times. All simulation are computed on a desktop with one AMD Ryzen 3800X CPU, 32GB of DDR4 RAM and one Nvidia RTX 3080 GPU.

We choose $\mF^{(t)}$ and $\mH^{(t)}$ as RKHSs endowed with Gaussian kernels, with bandwidths selected according to the median heuristic trick by \cite{fukumizu2009kernel} for each $1\leq t\leq T$. The pool of scaling factors \texttt{SCALE} contains 30 positive numbers spaced evenly on a log scale between 0.001 to 0.05. The number of cross-validation partition $K=5$.
The true target policy value of $\pi$ is estimated by the mean cumulative rewards of $50,000$ Monte Carlo trajectories with policy $\pi$.

\end{document}